\documentclass{article}

\usepackage{amsmath, mathrsfs, amssymb, amsthm, mathtools, bm, bbm} 

\usepackage{caption} 

\usepackage{subfig}

\usepackage{makecell} 
\usepackage{float} 

\usepackage[english]{babel} 
\usepackage[utf8]{inputenc} 
\usepackage[T1]{fontenc} 
\usepackage[margin=1.2in,footskip=0.35in]{geometry} 

\usepackage[normalem]{ulem}

\usepackage{tikz} 
\usetikzlibrary[topaths] 
\usepackage[symbol]{footmisc}

\usepackage{enumitem}

\usepackage{algorithm}
\usepackage{algcompatible} 
\usepackage{graphics}

\newtheorem{theorem}{Theorem}
 
\newtheorem{definition}{Definition}
\newtheorem{lemma}{Lemma}
\newtheorem{remark}{Remark}
\newtheorem{corollary}{Corollary}
\newtheorem{proposition}{Proposition}

\newtheorem*{theorem*}{Theorem}
\newtheorem*{example*}{Example} 
\newtheorem*{definition*}{Definition}
\newtheorem*{lemma*}{Lemma}
\newtheorem*{proposition*}{Proposition}
\newtheorem*{assumption*}{Assumption}
\newtheorem*{claim*}{Claim}

\newtheoremstyle{TheoremNum}
        {\topsep}{\topsep}              
        {\itshape}                      
        {}                              
        {\bfseries}                     
        {.}                             
        { }                             
        {\thmname{#1}\thmnote{ \bfseries #3}}
\theoremstyle{TheoremNum}

\newtheoremstyle{LemmaNum}
        {\topsep}{\topsep}              
        {\itshape}                      
        {}                              
        {\bfseries}                     
        {.}                             
        { }                             
        {\thmname{#1}\thmnote{ \bfseries #3}}
\theoremstyle{LemmaNum}

\newtheorem*{corollary*}{Corollary}

\usepackage{hyperref} 
\usepackage{multirow}

\newcommand{\hochkomma}{${\,\,}^{,}$}

\newcommand{\X}{ \mathcal{X} }

\renewcommand{\Pr}{ \mathbb{P} }

\newcommand{\x}{ \mathbf{x} }
\newcommand{\y}{ \mathbf{y} }
\newcommand{\z}{ \mathbf{z} }

\renewcommand{\u}{ \mathbf{u} }

\newcommand{\Dim}{ \mathrm{Dim} }

\newcommand{\Ind}{ \mathbb{I} }

\newcommand{\A}{ \mathcal{A} } 

\newcommand{\M}{\mathcal{M}}
\newcommand{\D}{\mathcal{D}}
\newcommand{\Exp}{{\mathrm{Exp}}}

\newcommand{\R}{\mathbb{R}}
\newcommand{\B}{\mathbb{B}}
\renewcommand{\S}{\mathbb{S}}
\newcommand{\E}{\mathbb{E}}

\renewcommand{\[}{\left[ }
\renewcommand{\]}{\right] }

\newcommand{\<}{\left< }
\renewcommand{\>}{\right> }

\renewcommand{\(}{\left( }
\renewcommand{\)}{\right) }

\newcommand{\wt}{\widetilde }
\newcommand{\wh}{\widehat }

\newcommand{\citep}{\cite}
\newcommand{\citet}{\cite}

\title{Batched Stochastic Bandit for Nondegenerate Functions} 

\author{Yu Liu\footnote{Equal contribution}\hochkomma\thanks{22110840006@m.fudan.edu.cn} \qquad Yunlu Shu${}^{*,}$\footnote{22110840008@m.fudan.edu.cn} \qquad Tianyu Wang\footnote{Correspondence to: wangtianyu@fudan.edu.cn}} 

\date{}

\begin{document}

\maketitle

\begin{abstract}
    This paper studies batched bandit learning problems for nondegenerate functions. Over a compact doubling metric space $(\mathcal{X}, \mathcal{D})$, a function $f : \mathcal{X} \to \mathbb{R}$ is called nondegenerate if there exists $ L \ge \lambda >0 $ and $q \ge 1$, such that 
    $$\lambda \left( \mathcal{D} (\mathbf{x}, \mathbf{x}^*) \right)^q \le f (\mathbf{x}) - f (\mathbf{x}^*) \le L \left( \mathcal{D} (\mathbf{x}, \mathbf{x}^*) \right)^q , \; \mathbf{x} \in \mathcal{X}, $$
where $\mathbf{x}^* = \arg\min_{\mathbf{z} \in \mathcal{X}} f (\mathbf{z}) $ is the unique minimizer of $f$ over $\mathcal{X}$. In this paper, we introduce an algorithm that solves the batched bandit problem for nondegenerate functions near-optimally. More specifically, we introduce an algorithm, called Geometric Narrowing (GN), whose regret bound is of order $ \widetilde{{\mathcal{O}}} \left( A_+^d \sqrt{T} \right) $, where $d$ is the doubling dimension of $(\mathcal{X},\mathcal{D})$, and $A_+$ is a constant independent of $ d $ and the time horizon $T$. In addition, GN only needs $ {\mathcal{O}} (\log \log T) $ batches to achieve this regret. We also provide lower bound analysis for this problem. More specifically, we prove that over some (compact) doubling metric space of doubling dimension $d$: 1. For any policy $\pi$, there exists a problem instance on which $\pi$ admits a regret of order ${\Omega} \left( A_-^d \sqrt{T} \right)$, where $A_-$ is a constant independent of $ d $ and $T$; 2. No policy can achieve a regret of order $ A_-^d \sqrt{T} $ over all problem instances, using less than $ \Omega \left( \log \log T \right) $ rounds of communications. Our lower bound analysis shows that the GN algorithm achieves near optimal regret with minimal number of batches. 
\end{abstract}

\section{Introduction}

In batched stochastic bandit, an agent collects noisy rewards/losses in batches, and aims to find the best option while exploring the space \citep{thompson1933likelihood,robbins1952some,gittins1979bandit,lai1985asymptotically,auer2002finite,auer2002nonstochastic,perchet2016batched,gao2019batched}. This setting reflects the key attributes of crucial real-world applications. For example, in experimental design \citep{robbins1952some,berry1985bandit}, the observations are often noisy and collected in batches \citep{perchet2016batched,gao2019batched}. In this paper, we consider batched stochastic bandits for an important class of functions, called ``nondegenerate functions''. 

\subsection{Nondegenerate Functions}

Over a compact doubling metric space $(\X, \D)$, a function $f : \X \to \R$ is called a nondegenerate function if there exists $ L \ge \lambda >0 $ and $q \ge 1$, such that 
\begin{align} 
    \lambda \( \D (\x, \x^*) \)^q \le f (\x) - f (\x^*) \le L \( \D (\x, \x^*) \)^q ,  \label{eq:def-nondegen} 
\end{align} 
$\forall \x \in \X,$ where $\x^* = \arg\min_{\z \in \X} f (\z) $ is the unique minimizer of $f$ over $\X$. 
Nondegenerate functions \citep{valko2013stochastic,10.5555/3294771.3294841,gemp2024approximating} hold significance as they encompass various important problems. 
Below we list two important {classes of motivating applications for} nondegenerate functions. 
\begin{itemize} 
    \item \textbf{(P0, real-world motivations) Revenue curve as a function of price:} Consider the space $\(\X, \D\)$ with $ \X = [0,1] $ and $\D (\x, \y) = | \x - \y |$. If $\x \in [0,1]$ models price, then functions satisfying (\ref{eq:def-nondegen}) provide a natural model for revenue curve as a function of price, up to a flip of sign. 
    Nondegenerate functions, which naturally extend the concept of strictly concave/convex functions, provide a general framework for modeling revenue as a function of price. 
    With the revenue curve modelled by a nondegenerate function of price, our results find applications in several real-world scenarios, including dynamic pricing (e.g., \citet{NEURIPS2019_0a3df703,chen2023robust,Perakis2023dynamic}). 

    As is widely accepted (e.g., \citet{mankiw1998principles}), the overall revenue curve typically exhibits a pattern where it first rises and then falls, as the price increases. These overall revenue curves can be adequately represented by nondegenerate functions. In contrast to standard models, such modeling permits fluctuations in regions that are relatively distant from the optimum. Recently, with the growing demand for online pricing strategies in e-commerce and repeated auctions, bandit algorithms have been employed to address dynamic pricing challenges, where the agent repeatedly selects prices, observes the corresponding revenue, and aims to maximize the gain on-the-fly. Therefore, our research on batched bandits for nondegenerate functions offers a novel approach to solving dynamic pricing problems. 
    \item \textbf{(P) Nonsmooth nonconvex objective over Riemannian manifolds:} Our study introduces a \emph{global} bandit optimization method for a class of nonconvex functions on compact Riemannian manifolds, where nontrivial convex functions do not exist \citep{Yau1974}. Let $ \( \X , \D \)$ be a compact finite-dimensional Riemannian manifold with the metric defined by the geodesic distance (e.g., \citet{petersen2006riemannian}). 
    Then a smooth function with nondegenerate Taylor approximation satisfies (\ref{eq:def-nondegen}) near its global minimum $\x^*$. More specifically, we can Taylor approximate the function $f$ near $\x^*$ and get, for $\x = \mathrm{Exp}_{\x^*} (\mathbf{v})$ with some $\mathbf{v} \in T_{\x^*} \mathcal{M}$, $ f (\x) \approx \sum_{i=0}^K \frac{1}{i!}\varphi_{\mathbf{v}}^{(i)} \( \| \mathbf{v} \| \) $ where $ \varphi_{\mathbf{v}}^{(i)} $ is the $i$-th derivative of $f \circ \Exp_{\x^*}$ along the direction of $\mathbf{v}$, and $K \ge 2$ is some integer. Since $ \x^* $ is a local minimum of $f$, we have $\varphi_{\mathbf{v}}^{(1)} = 0$ (for all $\mathbf{v}$), and thus $ f (\x) - f (\x^*) \approx \frac{1}{q!} \varphi_{\mathbf{v}}^{(q)} \( \| \mathbf{v} \| \) $ where $q \ge 2$ is the smallest integer such that $ \varphi_{\mathbf{v}}^{(q)} \neq 0 $ (for some $\mathbf{v}$). If $ \varphi_{\mathbf{v}}^{(q)} $ is nontrivial for all $\mathbf{v} \in T_{\x^*} \M$, that is, the leading nontrivial total derivative of $f$ is nondegenerate, then the function $ f $ satisfies (\ref{eq:def-nondegen}) in a neighborhood of $\x^*$. This justifies the name ``nondegenerate''. 
    In Figure \ref{fig:example-2}, we provide a specific example of a nondegenerate function over a Riemannian manifold. Over the entire manifold, the objective is nonsmooth and nonconvex. 
    
\end{itemize} 

A concrete illustrating example of nondegenerate function is the linear function constrained to the unit sphere: 
\begin{align*} 
    f (\x) = \< \mathbf{u} , \x \> , 
    \quad s.t. \quad 
    \| \x \|_2 = 1 , 
\end{align*} 
where $\mathbf{u} \in \S^{n-1}$ is an unknown fixed vector. It is evident that $ f $ obtains its minimum (over the sphere) at $ - \u $. This function $f$ satisfies, for any $\x \in \S^{n-1}$
\begin{align*} 
    f (\x) - f \( - \u \) 
    = 
    \< \u, \x \> + 1 
    = 
    \cos \( \D ( \u, \x ) \) + 1, 
\end{align*} 
where $\D (\u, \x) := \arccos \< \u, \x \> $ is the geodesic distance over the sphere. 
In metric space $\( \S^{n-1} , \D \)$, doubling dimention $d \asymp n-1$.
By noticing $ \cos \( \D ( \u, \x ) \) = \cos \( \pi - \D ( -\u, \x ) \) $ over the sphere, we have 
\begin{align*} 
    f (\x) - f \( - \u \) 
    =& \;  
    \cos \( \pi - \D ( -\u, \x ) \) + 1 \\ 
    =& \;  
    1 - \cos ( \D ( -\u, \x ) ) . 
\end{align*} 
From here one can easily verify that this function is a nondegenerate function. 
        Since $1-\theta^2/2 \le \cos \theta \le 1-\theta^2/5$ for $\theta \in [0,\pi]$, on the unit sphere the linear function $f$ satisfies, for all $\x \in \S^{d-1}$, 
        \begin{align*}
            \frac{\D ( -\u, \x )^2}{5}
            \le f (\x) - f \( - \u \)
            \le \frac{\D ( -\u, \x )^2}{2} . 
        \end{align*}

Also, it is worth emphasizing that nondegenerate functions can possess nonconvexity, nonsmoothness, or discontinuity. As an illustration, consider the following nondegenerate function $f(\mathbf{x})$ defined over the interval $[-2,2]$, which exhibits discontinuity: 
\begin{align} 
    f (\x) 
    = 
    \begin{cases} 
        -\x , & \text{if } \x \in [-2,-1) \\ 
        \x^2, & \text{if } \x \in [-1,1] \\ 
        \x + 1, & \text{if } \x \in (1,2] . 
    \end{cases} 
    \label{eq:exp} 
\end{align} 
A plot of the function (\ref{eq:exp}) is in Figure \ref{fig:example}, $\frac{\x^2}{2}$ (resp. $2\x^2$) is a lower bound (resp. upper bound) for $f (\x)$ over $[-2,2]$. 
More generally, over a compact Riemannian manifold, with the metric $\D$ defined by the geodesic distance, nondegenerate functions can still possess nonconvexity, nonsmoothness, or discontinuity. A specific example is shown in Figure \ref{fig:example-2}.

\begin{figure}[h!]
    \centering
    \includegraphics[width=3.25in]{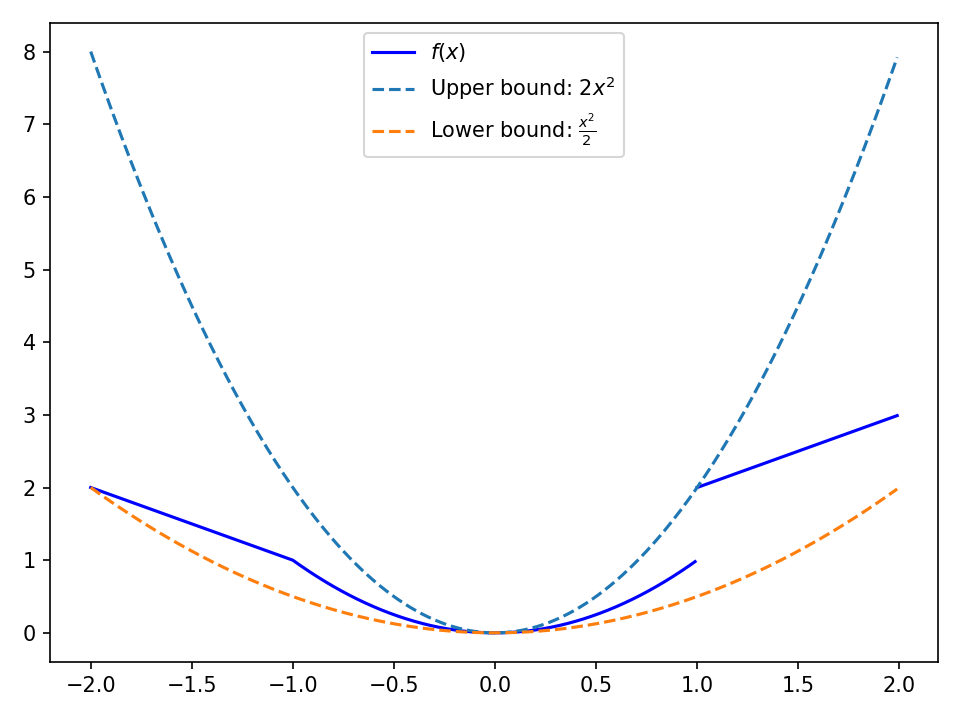}
    \caption{Plot of $f(\x)$ defined in (\ref{eq:exp}). $\frac{\x^2}{2}$ (resp. $2\x^2$) is a lower bound (resp. upper bound) for $f (\x)$ over $[-2,2]$. This plot shows that a nondegenerate function can be nonconvex, nonsmooth or discountinuous. }
    \label{fig:example}
\end{figure} 

\begin{figure}[h!]
    \centering
    \includegraphics[width=3.25in]{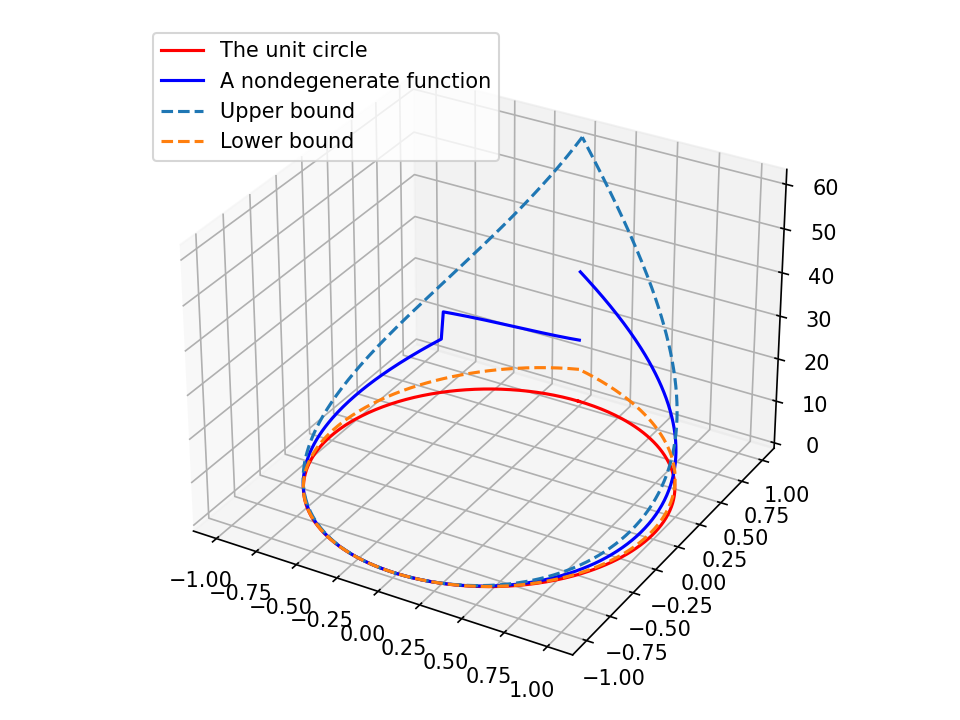}
    \caption{Plot of a nondegenerate function $f$ defined over the unit circle $\S^1$, and the metric is the arc length along the circle. This function is not convex and not continuous, but satisfies the nondegenerate condition.
    }
    \label{fig:example-2}
\end{figure}

Given the aforementioned motivating examples, developing an efficient stochastic bandit/optimization algorithm for nondegenerate functions is of great importance. In addition, we focus our study on the batched feedback setting, which is also important. 


\subsection{The Batched Bandit Setting} 

In bandit learning, more specifically stochastic bandit learning, the agent is tasked with sequentially making decisions based on noisy loss/reward samples associated with these decisions. The objective of the agent is to identify the optimal choice while simultaneously learning the expected loss function across the decision space. The effectiveness of the agent's policy is evaluated through regret, which quantifies the difference in loss between the agent's chosen decision and the optimal decision, accumulated over time. More formally, the $T$-step regret of a policy $\pi$ is defined as 
\begin{align} 
    R^\pi (T) := \sum_{t=1}^T f (\x_t) - f (\x^*), 
\end{align} 
where $\x_t \in \X$ is the choice of policy $ \pi $ at step $t$, $f$ is the expected loss function, and $\x^*$ is the optimal choice. 
Typically, the goal of a bandit algorithm is to achieve a regret rate that grows as slow as possible. 


\begin{remark} 
    For the rest of the paper, we will use a loss minimization formulation for the bandit learning problem. With a flip of sign, we can easily phrase the problem in a reward-maximization language. 
\end{remark} 



In the context of batched bandit learning, the primary objective remains to be minimizing the growth of regret. However, in this setting, the agent is unable to observe the loss sample immediately after making her decision. 
Instead, she needs to wait until a \emph{communication point} to collect the loss samples in batches. 
To elaborate further, in batch bandit problems, the agent in a $T$-step game dynamically selects a sequence of communication points denoted as $\mathcal{T}= \{t_0,\cdots,t_M \}$, where $0=t_0<t_1<\cdots<t_M=T$ and $M \ll T$. In this setting, loss observations are only communicated to the player at $t_1, \cdots, t_M $. Consequently, for any given time $t$ within the $j$-th batch ($t_{j-1}<t\leq t_j$), the reward $y_t$ remains unobserved until time $t_j$. The reward samples are corrupted by mean-zero, $iid$, 1-sub-Gaussian noise. The decision made at time $t$ is solely influenced by the losses received up to time $t_{j-1}$. The selection of the communication points $\mathcal{T}$ is \textcolor{black}{either adaptive or static. In the adaptive case, the player determines each point $t_j\in\mathcal{T}$ based on the previous operations and observations up to $t_{j-1}$. In the static case, all points $ t_j $ are specified before the algorithm starts. }

In batched bandit setting, the agent not only aims to minimize regret, but also seeks to minimize the number of communications points required. 

For simplicity, we use \emph{nondegenerate bandits} to refer to stochastic bandit problem with nondegenerate loss, and \emph{batched nondegenerate bandits} to refer to batched stochastic bandit problem with nondegenerate loss.


\subsection{Our Results} 

In this paper, we introduce an algorithm, called Geometric Narrowing (GN), that solves batched bandit learning problems for nondegenerate functions in a near-optimal way. The GN algorithm operates by successively narrowing the search space, and satisfies the properties stated in Theorem \ref{thm:upper}. 

\begin{theorem} 
    \label{thm:upper}
    Let $ \( \X, \D \) $ be a compact doubling metric space, and let $f $ be a nondegenerate function defined over $(\X,\D) $. Consider a stochastic bandit learning environment where all loss samples are corrupted by $ iid $ sub-Gaussian mean-zero noise. 
    For any $T \in \mathbb{N}_+$, with probability exceeding $1-2T^{-1}$, the $T$-step total regret of Geometric Narrowing, written $R^{GN}(T)$, 
    satisfies
    \begin{align*}
        R^{GN} (T) \le {K_+} A_+^{d} \sqrt{T \log T} \log \log \frac{T}{\log T},
    \end{align*}
    where $d$ is the doubling dimension of $(\X, \D)$, and $K_+$ and $A_+$ are constants independent of $d$ and $T$. 
    In addition, Geometric Narrowing only needs $\mathcal{O}\( \log \log T \)$ communication points to achieve this regret rate. 
        
\end{theorem} 

\textcolor{black}{The upper bound for the GN algorithm presented in Theorem \ref{thm:upper} uses adaptive batch sizes, meaning that the communication points $\mathcal{T}$ are a set of random variables that are determined dynamically. In Corollary \ref{cor:sta}, we demonstrate that by applying a slight tweak to the algorithm, we can achieve the same outcome using a static batch size that is predetermined before the algorithm begins. 
}
As a corollary of Theorem \ref{thm:upper}, we prove that the simple regret of GN is of order $ \mathcal{O} \( \sqrt{ \frac{\log T}{T} } \log\log T\) $. This result is summarized in Corollary \ref{cor:opt}. 

\begin{corollary} 
    \label{cor:opt} 
    Let $(\X, \D)$ be a compact doubling metric space. 
    Let $f$ be a nondegenerate function 
    defined over $(\X, \D)$. 
    For any $T \in \mathbb{N}_+$, with probability exceeding $1-2T^{-1}$, the GN algorithm finds a point $ \x_{out} $ such that $ f (\x_{out}) - f (\x^*) \le \mathcal{O} \( \sqrt{ \frac{ \log T }{ T } } \log\log T \) $. In addition, Geometric Narrowing only needs $\mathcal{O}\( \log \log T \)$ communication points to achieve this rate. 
\end{corollary}

Also, we prove that it is hard to outperform GN by establishing lower bound results in Theorems \ref{thm:standard}, \ref{thm:lower} and Corollary \ref{cor}. 
Theorem \ref{thm:standard}  states that no algorithm can uniformly perform better than $ \Omega \( A_-^d \sqrt{T} \) $ for some $A_-$ independent of $d$ and $T$. 

\begin{theorem}
    \label{thm:standard}
    For any $d \ge 1$ and $T \in \mathbb{N}_+$, there exists a compact doubling metric space $(\X_0, \D_0)$ that simultaneously satisfies the following: 1. The doubling dimension of $(\X_0, \D_0)$ is $\lfloor d \rfloor $. 2. 
    For any policy $\pi$, there exists a problem instance $I$ defined over $(\X_0, \D_0)$, such that the regret of running $\pi$ on $I$ satisfies 
    \begin{align*} 
        \E \[ R^\pi (T) \]  
        \ge 
        K_- A_-^{\lfloor d \rfloor } \sqrt{T} 
    \end{align*} 
    where $ \E $ is the expectation whose probability law is induced by running $\pi$ (for $T$ steps) on the instance $ I $, 
    and $ K_- $ and $A_-$ are numbers that do not depend on $d$ or $T$. 
\end{theorem} 

Theorem \ref{thm:standard} implies that the regret bound for GN is near-optimal. 
Also, we provide a lower bound analysis for the communication lower bound of batched bandit for nondegenerate functions. This result is stated below in Theorem \ref{thm:lower}. 

\begin{theorem} 
    \label{thm:lower} 
    Let $ M \in \mathbb{N}_+$ be the total rounds of communications allowed. 
    For any $d \ge 1$ and $T \in \mathbb{N}_+$ ($T \ge M$), there exists a compact doubling metric space $(\X_0, \D_0)$ that simultaneously satisfies the following: 1. the doubling dimension of $(\X_0, \D_0)$ is $\lfloor d\rfloor $, and 2. for any policy $\pi$, there exists a problem instance $I$ defined over $(\X_0, \D_0)$, such that the regret of running $\pi$ on $I$ for $T$ steps satisfies
    \begin{align*} 
        \E \[ R^{\pi} ( T ) \] 
        \geq 
        K_- A_-^{\lfloor d \rfloor } \cdot \frac{1}{M^2} \cdot T^{\frac{1}{2} \cdot \frac{1}{1-2^{-M}}}, 
    \end{align*} 
    where 
    $ K_- $ and $A_-$ are numbers that do not depend on $d$, $M$ or $T$. 
\end{theorem} 

By setting $M$ to the order of $\log \log T$ in Theorem \ref{thm:lower}, we have the following corollary. 

\begin{corollary} 
    \label{cor} 
    For any $d \ge 1$ and $T \in \mathbb{N}_+$, there exists a compact doubling metric space $(\X_0, \D_0)$ that simultaneously satisfies the following: 1. The doubling dimension of $(\X_0, \D_0)$ is $\lfloor d\rfloor $; 2. 
    If less than $\Omega (\log \log T)$ rounds of communications are allowed, no policy can achieve a regret of order $ A_-^{\lfloor d \rfloor } \sqrt{T} $ over all nondegenerate bandit instances defined over $ (\X_0, \D_0) $, where $A_-$ is a number independent of $d$ and $T$. 
\end{corollary} 

Corollary \ref{cor} implies that the communication complexity of the GN algorithm is near-optimal, since no algorithm can improve GN's communication complexity without worsening the regret. 

\emph{Note:} In Theorem \ref{thm:standard}, Theoerm \ref{thm:lower} and Corollary \ref{cor}, the specific values of $K_-$ and $A_-$ may be different at each occurrence. 

Our results also suggest a curse-of-dimensionality phenomenon, discussed below in Remark \ref{rem}. 
\begin{remark}[Curse of dimensionality] 
    \label{rem} 
    Our lower bounds (Theorems \ref{thm:standard} and \ref{thm:lower}) grow exponentially in the doubling dimension $d$. Therefore, no algorithm can uniformly improve this dependence on $d$, resulting in a phenomenon commonly referred to as curse-of-dimensionality. 
\end{remark} 


        

\subsection{Implications of Our Results}

Our results have several important implications. Firstly, our research gives a distinct method for the stochastic convex optimization with bandit feedback\citep{shamir2013complexity},  especially for the strongly-convex and smooth function which is a special kind of nondegenerate function. For the real-world problem discussed previously in \textbf{(P0)}, our GN algorithm provides a solution to the online/dynamic pricing problem (without inventory constraints) (e.g., \citet{chen2023robust,Perakis2023dynamic}, and references therein). 
Also, our results yield intriguing implications on Riemannian optimization, and offer a new perspective on stochastic Riemannian optimization problems. 



\textbf{(I) Implications on stochastic zeroth-order optimization over Riemmanian manifolds:} Our GN algorithm provides a solution for optimizing nondegenerate functions over compact finite-dimensional Riemannian manifolds (with or without boundary). Our results imply that, the global optimum of a large class of nonconvex and nonsmooth functions can be efficiently approximated. As stated in Corollary \ref{cor}, we show that GN finds the global optimum of the objective at rate $ \wt{\mathcal{O}} \( \frac{1}{\sqrt{T}} \) $. To our knowledge, for stochastic optimization problems, this is the first result that guarantees an $ \wt{\mathcal{O}} \( \frac{1}{\sqrt{T}} \) $ convergence to the global optimum for nonconvex nonsmooth optimization over compact finite-dimensional Riemannian manifolds. In addition, only $ {\mathcal{O}} \( \log \log T \) $ rounds of communication are needed to achieve this rate. 

\subsection{Challenges and Our Approach} 

As the first work that focuses on batched bandit learning for nondegenerate functions, we face several challenges throughout the analysis, especially in the lower bound proof. 
Unlike existing lower bound analyses, the geometry of the underlying space imposes challenging constraints on the problem instance construction. 
To further illustrate this challenge, we briefly review the lower bound instance construction for Lipschitz bandits \citep{kleinberg2005nearly,kleinberg2008multi,bubeck2008tree,bubeck2011x}, and explain why techniques for Lipschitz bandits do not carry through. 
Figures \ref{fig:height} illustrate some instance constructions, showing an overall picture (left figure), three instances for Lipschitz bandits lower bound (right top figure), and three instances of a ``naive attempt'' (right bottom figure). 
We start with the instances for Lipschitz bandits lower bounds (solid blue line in left figure). 
In such cases, as the ``height'' decreases with $T$, unfortunately the nondegenerate parameter $\lambda$ also decreases with $T$. 
Also, using the solid red line (left figure) instances as a ``naive attempt'' disrupts key properties of Lipschitz bandit instances. 
Specifically: (1) As shown in the right top figure, except for in $\{S_i\}_{i=1}^3$, the values of $f_{Lip}^i$ are identical. In contrast, for the ``naive attempt'' $\{f_{naive}^i\}_{i=1}^3$ in the right bottom figure, the function values vary across the domain $a.e.$ (2) From an information-theoretic perspective, distinguishing between $f_{Lip}^1$ and $f_{Lip}^2$ is as difficult as differentiating $f_{Lip}^1$ from $f_{Lip}^3$, regardless of the distance between their optima. Conversely, telling apart $f_{naive}^1$ from $f_{naive}^2$ can be harder than distinguishing $f_{naive}^1$ from $f_{naive}^3$, if the optima of $f_{naive}^1$ and $f_{naive}^2$ are closer than those of $f_{naive}^1$ and $f_{naive}^3$. 
 
\begin{figure}[h!] 
    \centering 
    \begin{minipage}{0.48\textwidth} 
        \centering 
        \includegraphics[width = \textwidth]{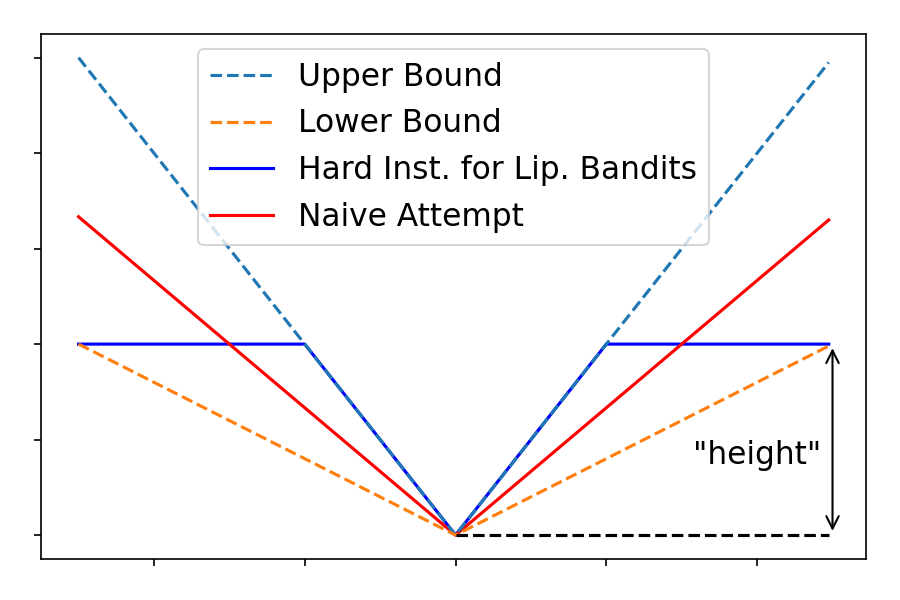} 
    \end{minipage} 
    \begin{minipage}{0.48\textwidth} 
        \centering 
        \includegraphics[width = \textwidth]{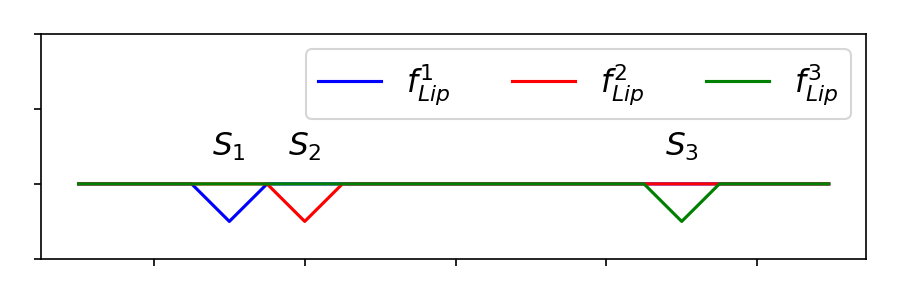} 
        \includegraphics[width = \textwidth]{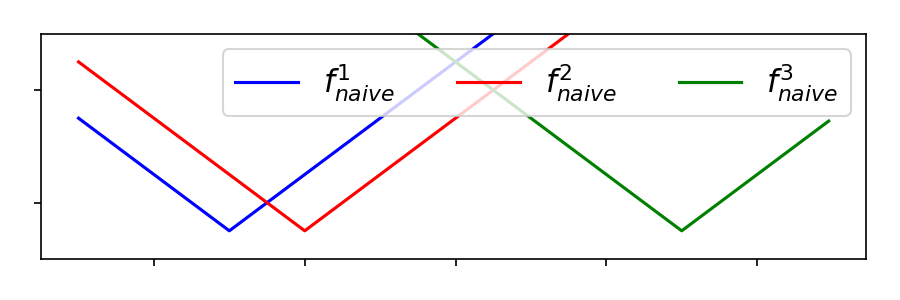} 
    \end{minipage} 
    \caption{Explanation of the instance for nondegenerate bandits} 
    \label{fig:height} 
\end{figure}

This implies we cannot change the instances as freely as previously done in the literature, not to mention that all of the lower bound arguments need to take the communication patterns into consideration. 
To overcome this difficulty, we use a trick called \emph{bitten apple} construction \textcolor{black}{(See Figure \ref{fig:fjkl})}. \textcolor{black}{
The term `bitten apple' derives from the need to focus on regions that are set-subtraction of two metric balls. 
Such regions resemble the shape of bitten apples. Using regions of this shape, we can create nondegenerate instances with suitable lower and upper bounds that are hard to distinguish from one another.
} 
This trick overcomes the constraints imposed by the nondegenerate property, and is critical in proving a lower bound that scales exponentially with the doubling dimension $d$. As a result, this trick is critical in justifying the curse-of-dimensionality phenomenon in Remark \ref{rem}. 

For the algorithm design and analysis, we need to carefully utilize the properties of nondegenerate functions to design an algorithm with regret upper bound $\wt{\mathcal{O}} ( \sqrt{T} )$ and communication complexity $\mathcal{O} (\log \log T)$. 
We carefully integrate in the properties of the nondegenerate functions in both the algorithm procedure and the communication pattern, \textcolor{black}{and introduce an algorithm framework that simultaneously supports both static communication grids and adaptive communication grids; See Section \ref{sec:related-works} for discussion in context of related works, and Corollary \ref{cor:sta} for further details}. In addition, we design the algorithm in a succinct way, so that the GN algorithm has additional advantages in time complexity (Remark \ref{remark:time}) and space complexity (Remark \ref{remark:space}).

\subsection{Related Works}
\label{sec:related-works}



Compared to many modern machine learning problems, the stochastic Multi-Armed Bandit (MAB) problem has a long history \citep{thompson1933likelihood,robbins1952some,gittins1979bandit,lai1985asymptotically,auer2002finite,auer2002nonstochastic}. Throughout the years, many solvers for this problem has been invented, including Thompson sampling \citep{thompson1933likelihood,agrawal2012analysis}, the UCB algorithm \citep{lai1985asymptotically,auer2002finite}, exponential weights \citep{auer2002nonstochastic,arora2012multiplicative}, and many more; See e.g., \citep{MAL-024,MAL-068,lattimore2020bandit} for an exposition.  

Throughout the years, multiple variations of the stochastic MAB problems have been intensively investigated, including linear bandits \citep{auer2002using,dani2007price,chu2011contextual,abbasi2011improved}, Gaussian process bandits \citep{srinivas2012information,contal2014gaussian}, bandits in metric spaces \citep{kleinberg2005nearly,kleinberg2008multi,bubeck2008tree,bubeck2011x,pmlr-v134-podimata21a}, just to name a few. 
\textcolor{black}{The problem of identifying the best arm in MAB is considerd by \citet{agarwal2017learning}, \citet{ jin2019efficient}, and \citet{ karpov2024parallel}. \citet{karpov2024parallel} studied this problem in the heterogeneous collaborative learning setting, while \citet{jin2019efficient} explored the exact top-$k$ arm identification problems in an adaptive round model.}
Among enormous arts on bandit learning, bandits in metric spaces are particularly related to our work. 
In its early stage, bandits in metric spaces primarily focus on bandit learning over $[0,1]$ \citep{agrawal1995continuum,kleinberg2005nearly,auer2007improved,cope2009regret}. Afterwards, algorithms for bandits over more general metric spaces were developed \citep{kleinberg2008multi,bubeck2008tree,bubeck2011x,bubeck2011lipschitz,magureanu2014lipschitz,lu2019optimal,krishnamurthy2020contextual,majzoubi2020efficient,10239433}.
In particular, the Zooming bandit algorithm \cite{kleinberg2008multi,slivkins2011contextual} and the Hierarchical Optimistic Optimization (HOO) algorithm \cite{bubeck2008tree,bubeck2011x} were the first algorithms that optimally solve the Lipschitz bandit problem (up to logarithmic factors). Subsequently, \citet{valko2013stochastic} considered an early version of nondegenerate functions, and built its connection to Lipschitz bandits \citep{kleinberg2008multi,bubeck2008tree,bubeck2011x}. 
\citet{valko2013stochastic} proposed StoSOO algorithm for pure exploration of function that is locally smooth with respect to some semi-metric. But, to our knowledge, bandit problems with such functions have not been explored.




In recent years, urged by the rising need for distributed computing and large-scale field experiments (e.g., \citet{berry1985bandit,cesa2013online}), the setting of batched feedback has gained attention. \citet{perchet2016batched} initiated the study of batched bandit problem, and \citet{gao2019batched} settled several important problems in batched multi-armed bandits. Over the last few years, many researchers have contributed to the batched bandit learning problem \textcolor{black}{\citep{jun2016top,agarwal2017learning,tao2019collaborative,han2020sequential,karpov2020collaborative,esfandiari2021regret,ruan2021linear,li2022gaussian,agarwal2022batched,jin2021almost, jin2021double}}. 
For example, \citet{han2020sequential} and \citet{ruan2021linear} provide solutions for batched contextual linear bandits. \citet{li2022gaussian} studies batched Gaussian process bandits. 
\textcolor{black}{\cite{jin2021double} improved the explore-then-commit strategy while addressed the batched multi-armed bandit problem. Meanwhile, \cite{jin2021almost} delved into the anytime batched multi-armed bandit problem.}

Despite all these works on stochastic bandits and batched stochastic bandits, no existing work focuses on batched bandit learning for nondegenerate functions.




\subsubsection{Additional related works from stochastic zeroth-order Riemannian optimization}

Since our work has some implications on stochastic zeroth-order Riemannian optimization, we also briefly survey some related works from there; See \citep{absil2008optimization,boumal2023introduction} for modern expositions on general Riemannian optimization. 

In modern terms, \citet{li2023stochastic} provided the first oracle complexity analysis for zeroth-order stochastic Riemannian optimization. Afterwards, \citet{li2023zeroth} introduced a new stochastic zeroth-order algorithm that leverages moving average techniques. 
In addition to works specific to stochastic zeroth-order Riemannian optimization, numerous researchers have contributed to the field of Riemannian optimization, including \citep{doi:10.1137/140955483,doi:10.1137/16M1098759,doi:10.1137/17M1116787,doi:10.1137/18M122457X,gao2021riemannian,doi:10.1137/20M1312952}, just to name a few.

Yet to the best of our knowledge, no prior art from (stochastic zeroth-order) Riemannian optimization literature focuses on approximating the global optimum over a compact Riemannian manifold for functions that can have discontinuities in its domain. Therefore, our results might be of independent interest to the Riemannian optimization community.

\textcolor{black}{
\textbf{Additional advantages contextualized in related works.} 
Compared to several most related works \citep{gao2019batched,10239433}, our method simultaneously supports both static time grids and adaptive time grids. Specifically, we introduce a \textbf{Rounded Radius (RR) Sequence} (Definition 3), which is carefully designed to facilitate the algorithm's operations and the derivation of the upper bound. With this RR sequence, our GN algorithm is versatile, simultaneously supporting both adaptive and static grid settings. This improves the BLiN algorithm of \citet{10239433} since BLiN does not achieve the near-optimal performance with a sequence of pre-determined communication points. 
}

\textbf{Paper Organization.} The rest of the paper is organized as follows. In Section \ref{sec:prelim}, we list several basic concepts and conventions for the problem. In Section \ref{sec:algo}, we introduce the Geometric Narrowing (GN) algorithm. In Section \ref{sec:lower}, we provide lower bound analysis for batched bandits for nondegenerate functions.

\section{Preliminaries} 
\label{sec:prelim}

Perhaps we shall begin with the formal definition of doubling metric spaces, since it underpins the entire problem.  

\begin{definition}[Doubling metric space] 
    \label{def:doubling} 
    The doubling constant of a metric space $(\X,\D)$ is the minimal $N$ such that for all $\x \in \X$, for all $r > 0$, the ball $\B (\x, r) := \{ \z \in \X: \D (\z, \x) \le r \}$ can be covered by $ N $ balls of radius $\frac{r}{2}$. 
    A metric space is called doubling if $N < \infty$. 
    The doubling dimension of $\X$ is 
    $d = \log_2 (N)$ where $N$ is the doubling constant of $\X$. 
\end{definition} 

An immediate consequence of the definition of doubling metric spaces is the following proposition. 

\begin{proposition} 
    \label{prop:cover} 
    Let $ (\X, \D) $ be a doubling metric space. For each $\x \in \X$ and $r \in (0,\infty)$, the ball $ \B \( \x, r \) $ can be covered by $ 2^{kd} $ balls of radius $ r\cdot 2^{-k} $ for any $k \in \mathbb{N}$, where $ d $ is the doubling dimension of $(\X, \D)$. 
\end{proposition} 

On the basis of doubling metric spaces, we formally define nondegenerate functions. 

\begin{definition}[Nondegenerate functions] 
    \label{def:nondegen} 
    Let $(\X, \D)$ be a doubling metric space. A function $f : \X \to \R$ is called nondegenerate if the followings hold: 
    \begin{itemize} 
        \item $\inf_{\x \in \X} f (\x) > -\infty $ and $ f $ attains its unique minimum at $\x^* \in \X$. 
        \item There exist $ L \ge \lambda > 0 $ and $q \ge 1$, such that $\lambda \( \D (\x, \x^*) \)^q \le f (\x) - f (\x^*) \le L \( \D (\x, \x^*) \)^q ,$ for all $ \x \in \X$. 
    \end{itemize} 
    The constants $L, \lambda, q$ are referred to as nondegenerate parameters of function $f$. 
\end{definition}

\color{black}

    When the function admits multiple global minimum, one can alter the space so that all global minima are consolidated into a single point, and adjust the metric accordingly. 
    Alternatively, we can also adopt other notions to describe the function landscape. In the Appendix, we show that our approach can also operate under assumptions similar to those in \cite{wang2018optimization}. 

\color{black}





Before proceeding further, we introduce the following notations and conventions for convenience. 
\begin{itemize} 
    \item For two set $S,S' \subset \X$, define 
    \begin{align} 
        D (S, S') := \sup_{\x\in S, \x' \in S'} \D (\x, \x') . 
    \end{align} 
    \item For any $z > 0$, define $ \[z\]_2 : = 2^{\left\lceil \log_2 z \right\rceil } $. 
    \item Throughout the paper, 
    all numbers except for the time horizon $T$, doubling dimension $d$, and rounds of communications $M$, are regarded as constants. 
\end{itemize}






\color{black}

\section{The Geometric Narrowing Algorithm} 
\label{sec:algo}
Our algorithm for solving batched nondegenerate bandits is called Geometric Narrowing (GN). As the name suggests, the GN algorithm progressively narrows down the search space, and eventually lands in a small neighborhood of $\x^*$. 
To achieve this, we need to identify the specific regions of the space that should be eliminated. Additionally, we want to achieve a near-optimal regret rate using only approximately $ \log \log T$ batches. 

Perhaps the best way to illustrate the idea of the algorithm is through visuals. In Figure \ref{fig:gn}, we provide an example of how function evaluations and nondegenerate properties jointly narrow down the search space. 
Yet a naive utilization of the observations in Figure \ref{fig:gn} is insufficient to design an efficient algorithm. Indeed, the computational cost grows quickly as the number of function value samples accumulates, even for the toy example shown in Figure \ref{fig:gn}. To overcome this, we succinctly summarize the observations illustrated in Figure \ref{fig:gn} as an algorithmic procedure. 

In addition to the narrowing procedure shown in Figure \ref{fig:gn}, we also need to determine the batching mechanism, in order to achieve the $\mathcal{O} (\log\log T)$ communication bound. This communication scheme is described through a radius sequence in Definition \ref{def:seq}. 
The procedure of GN is in Algorithm \ref{alg}. In Figure \ref{fig:example2-2}, we demonstrate an example run of the GN algorithm.

\begin{figure}[h!]
    \centering
    \includegraphics[scale = 0.55]{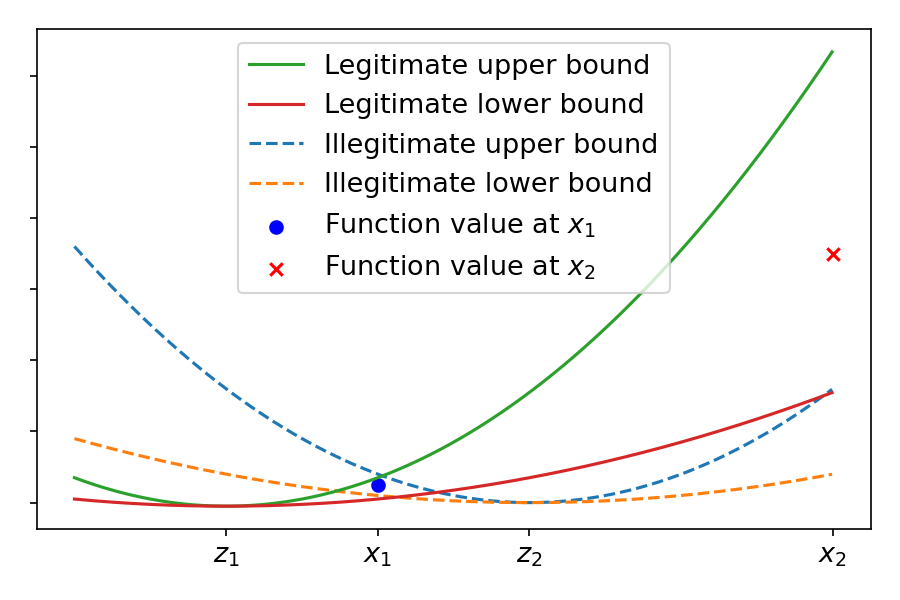}
    \caption{Illustration of the execution procedure of the GN algorithm over an interval. The function values at $\x_1$ and $\x_2$ jointly narrow down the range of $\x^*$. To ensure the function values at $\x_1$ and $\x_2$ fall between the upper and lower bounds for the nondegenerate function, the minimum of the function has to reside in a certain range. In this figure, the solid lines show a pair of legitimate bound, implying that the underlying functions may take its minimum at $\z_1$; the dashed lines show a pair of legitimate bound, implying that the underlying functions cannot take its minimum at $ \z_2$, neither in a neighborhood of $\z_2$. } 
    \label{fig:gn} 
\end{figure} 

\begin{algorithm*}[ht] 
	\caption{Geometric Narrowing (GN) for Nondegenerate Functions} 
	\label{alg} 
	\begin{algorithmic}[1]  
		\STATE \textbf{Input.} Space $(\mathcal{X}, \D)$; time horizon $T$; Number of batches $2M$.
            \STATEx /* Without loss of generality, let the diameter of $\X$ be $1$: $ \Dim \( \mathcal{X} \) = 1 $. */
		\STATE \textbf{Initialization.} 
            Rounded Radius sequence $\{\bar{r}_m\}_{m=1}^{2M}$ defined in Definition \ref{def:seq}; The first communication point $t_0=0$; Cover $ \X $ by $\bar{r}_1$-balls, and define $\mathcal{A}_{1}^{pre}$ as the collection of these balls. 
		\STATE Compute $ n_m = \frac{ 16 \log T }{ \lambda^2 \bar{r}_m^{2q} } $ for $m=1,\cdots,2M$. 
		
		\FOR{$m=1,2,\cdots,2M$}
            \STATE {If} $ \bar{r}_{m+1} > \bar{r}_{m} $, then \textbf{continue}. /* Skip the rest of the steps in the current iteration, and enter the next iteration. */ 
            \STATE For each ball $B\in\mathcal{A}_m^{pre}$, play arms $\x_{B,1},\cdots, \x_{B,n_m}$, all located at the region of $ B $. 
            \STATE Collect the loss samples $y_{B,1},\cdots, y_{B,n_m}$ associated with $\x_{B,1},\cdots, \x_{B,n_m}$. 
                Compute the average loss for each $B$, $\wh{f}_m(B):=\frac{\sum_{i=1}^{n_m}y_{B,i}}{n_m}$ for each ball $B\in\mathcal{A}_m^{pre}$. 
			Find  $\wh{f}_m^{\min}=\min_{B\in\mathcal{A}_m^{pre}}\wh{f}_m (B)$. Let $B_m^{\min}$ be the ball where $ \wh{f}_m^{\min} $ is obtained. 
            \STATE Define 
            \begin{align*} 
                \mathcal{A}_m := \left\{ B \in \mathcal{A}_m^{pre} : D (B, B_{m}^{\min}) \le 
                \( 2+ \( \frac{\lambda+L}{\lambda} \)^\frac{1}{q} \) \bar{r}_m \right\} . 
            \end{align*} 
            
            \STATE For each ball $ B \in \mathcal{A}_m$, use $\left(\bar{r}_m/\bar{r}_{m+1}\right)^d$ balls of radius $ \bar{r}_{m+1} $ to cover $ B $, and define $\mathcal{A}_{m+1}^{pre}$ as the collection of these balls.  
            \STATEx /* Due to Definition \ref{def:doubling}, we can cover $ B \in \mathcal{A}_m $ by $\left(\bar{r}_m/\bar{r}_{m+1}\right)^d$ balls of radius $ \bar{r}_{m+1} $. */ 
            \STATE Compute $t_{m+1}=t_m+(\bar{r}_{m}/\bar{r}_{m+1})^d \cdot|\mathcal{A}_m|\cdot n_{m+1}$. 
            If $t_{m+1}\geq T$ then \textbf{break}.
    \ENDFOR
    \STATE \textbf{Cleanup:} Pick a point in the region that is not eliminated, and play this point. Repeat this operation until all $T$ steps are used.  
    \STATE \textbf{Output (optional):} Arbitrarily pick $ \x_{out} \in \cup_{B \in \mathcal{A}_{2M}} B $ as an approximate for $\x^*$. /* This output step is optional, and only used for best arm identification or stochastic optimization tasks. */
	\end{algorithmic} 
\end{algorithm*}


\begin{figure*}[h!]
    \centering
    \includegraphics[width=\textwidth]{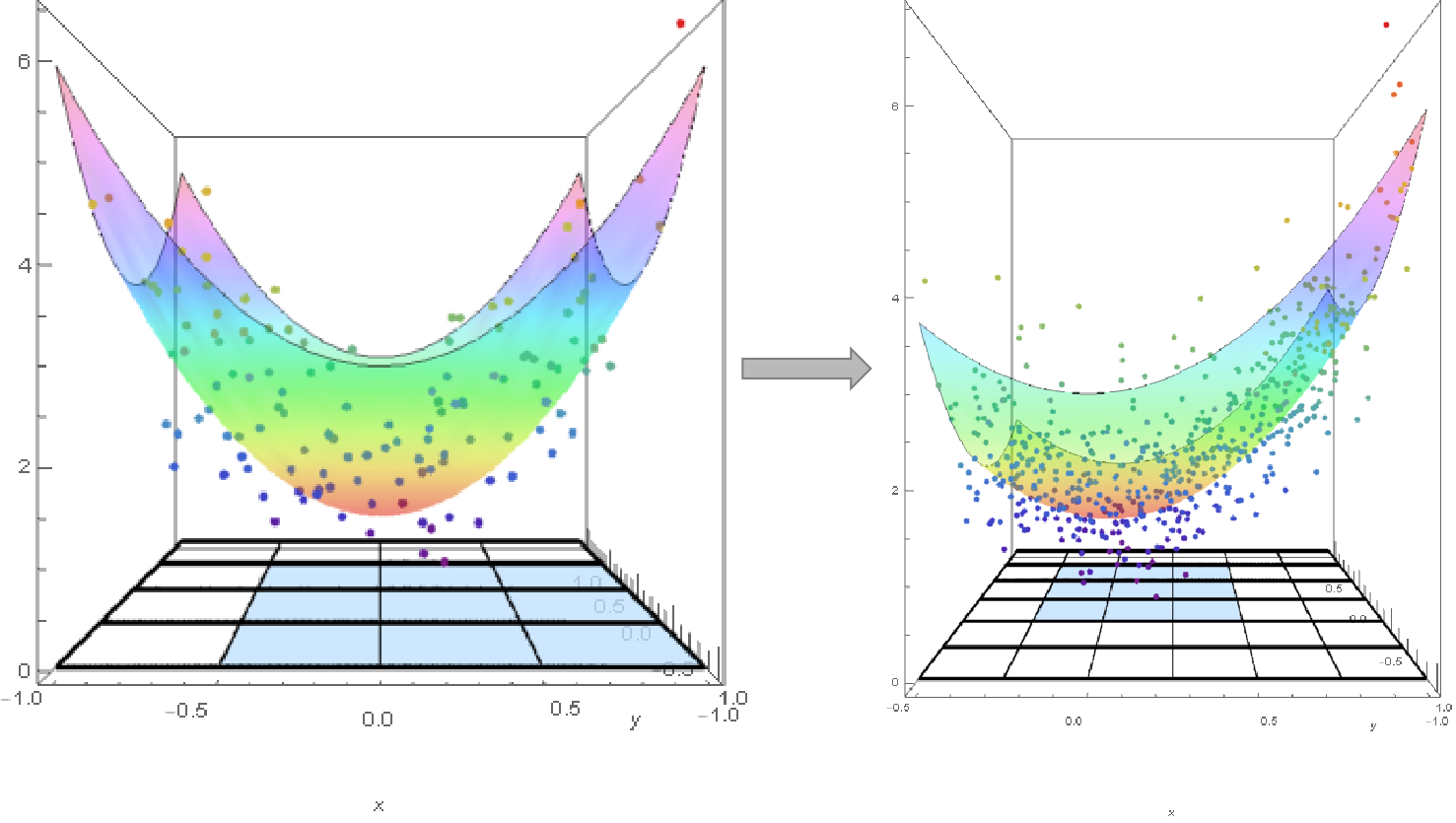}
    \caption{An example run of the GN algorithm. The surface shows the expected loss function, and the scattered points are loss samples over the current domain. These two plots describe the delete and split operations between adjacent batches of a GN run. } 
    \label{fig:example2-2}
\end{figure*}

\begin{definition} 
    \label{def:seq} 
    For $d > 0$ and $q\ge 1$, we define $\hat{c}_1=\frac{1}{2 (2q+d)\log 2 } \log\frac{T}{\log T} $ and $\hat{c}_{i+1}=\hat{\eta} \hat{c}_i$ for $i=1,2\dots$, where $\hat{\eta}=\frac{q+d}{2q+d}$. Then we define a sequence $\{\hat{r}_m\}_m$ by $\hat{r}_m=2^{-\sum_{i=1}^m \hat{c}_i}$ for $m=1,2\dots$. 
    On the basis of $\{\hat{c}_m\}_m$, we define $\hat{l}_m=\left \lfloor \sum_{i=1}^m \hat{c}_i \right \rfloor $ and $\hat{u}_m=\left \lceil \sum_{i=1}^m \hat{c}_i \right \rceil$. Then we define \textup{Rounded Radius (RR) Sequence}: $\Bar{r}_m$, $m=1,\dots,2M$: 
     \begin{align*} 
        \Bar{r}_m = 
        \begin{cases} 
            \Bar{r}_{2k-1} = 2^{-\hat{l}_k } = 2^{-\left \lfloor \sum_{i=1}^k \hat{c}_i \right \rfloor} , & \\
            \qquad \qquad \text{if } m=2k-1, k=1,\dots,M, \\ 
            \Bar{r}_{2k} = 2^{-\hat{u}_k } = 2^{-\left \lceil \sum_{i=1}^k \hat{c}_i \right \rceil} , &\\
            \qquad \qquad \text{if } m=2k, k=1,\dots,M.
        \end{cases} 
    \end{align*} 
\end{definition}

From the above definition, we have $\Bar{r}_{2k} \le \hat{r}_k \le \Bar{r}_{2k-1}$ for $k=1,\dots,M$. 
\textcolor{black}{
Below, we present an analysis of the algorithm's time complexity in remark \ref{remark:time}.
\begin{remark}
    \label{remark:time}
     The time complexity of GN algorithm is $\mathcal{O}(T)$. The algorithm consists of $2M$ cycles and a clean-up phase. In the $m$-th cycle, we perform $(t_{m+1} - t_m)$ samplings, $\mathcal{O}(t_{m+1} - t_m)$ arithmetic operations, and $|\mathcal{A}_m^{pre}|$ comparisons to identify $\wh{f}_m^{\min}$. During the clean-up phase, an additional $\mathcal{O}(T - t_{2M+1})$ samplings are required. Therefore, the total time complexity can be expressed as
        \begin{align*}
            \sum_{m=1}^{2M}\mathcal{O}(t_{m+1} - t_m)+\sum_{m=1}^{2M}|\mathcal{A}_m^{pre}|+\mathcal{O}(T - t_{2M+1}).
        \end{align*}
        It holds that $\sum_{m=1}^{2M} |\mathcal{A}_m^{pre}| \cdot n_m \leq T$. Consequently, $\sum_{m=1}^{2M} |\mathcal{A}_m^{pre}|$ is bounded by $\mathcal{O}(T)$, ensuring that the overall time complexity of GN algorithm remains $\mathcal{O}(T)$.
\end{remark}
}

\textcolor{black}{
\begin{remark} 
    \label{remark:space}
    Notice that the RR sequence $ \bar{r}_m $ can also take values $ \bar{r}_m = 2^{-m} $. With this choice of ball radii, the space complexity of GN does not increase with the time horizon $T$. 
\end{remark} 
}

\subsection{Analysis of the GN Algorithm} 

\label{sec:gn}

We start with the following simple concentration lemma. 

\begin{lemma} 
    \label{lem:concen} 
    Under Theorem \ref{thm:upper}'s assumptions, define 
    \begin{align*} 
        \mathcal{E} := & \;  
        \Bigg\{ \left| \wh{f}_m(B) - \E \[ \wh{f}_m(B) \] \right| \leq \sqrt{\frac{ 4 \log T}{n_m}}, \\ 
        & \qquad \forall 1 \le m \le 2 M , \; \forall B \in \A_m^{pre}\Bigg\} . 
    \end{align*} 
    It holds that $ \Pr \left( \mathcal{E} \right) \ge 1 - 2 T^{-1} $. 
\end{lemma}



\begin{proof}[Proof of Lemma \ref{lem:concen}]
    Fix a ball $B\in\mathcal{A}_m^{pre}$. Recall the average loss of $B \in \mathcal{A}_m^{pre}$ is defined as 
    \begin{equation*} 
        \wh{f}_m(B)=\frac{\sum_{i=1}^{n_m}y_{B,i}}{n_m}. 
    \end{equation*}
    We also have 
    \begin{equation*}
        \E \left[ \wh{f}_m(B) \right] =\frac{\sum_{i=1}^{n_m}f( \x_{B,i})}{n_m}. 
    \end{equation*}
    
    Since $\wh{f}_m (B)-\E \left[ \wh{f}_m(B) \right]$ is centered at zero, and is $\frac{1}{n_m}$-sub-Gaussian (e.g., Section 2.3 in \citet{10.1093/acprof:oso/9780199535255.001.0001}), applying the Chernoff bound gives 
    \begin{align*} 
        \Pr \left(\left|\wh{f}_m (B)-\E \left[ \wh{f}_m(B) \right]\right| \ge \sqrt{\frac{4 \log T}{n_m}}\right) \le \frac{2}{T^{2}}. 
    \end{align*} 
    
    Apparently, there are no more than $T$ balls that contain observations. Thus a union bound over these balls finishes the proof. 
\end{proof} 

Next in Lemma \ref{lem:not-eli}, we show that under event $\mathcal{E}$, the GN algorithm has nice properties. 

\begin{lemma} 
    \label{lem:not-eli} 
    Under event $\mathcal{E}$ (defined in Lemma \ref{lem:concen}), the following properties hold: 
    \begin{itemize} 
        \item The optimal point $\x^*$ is not removed; 
        \item For any $B \in \A_m$, $ \D ( \x , \x^* ) \le \Big( 2+ 2 \big( \frac{\lambda+L}{\lambda} \big)^\frac{1}{q} \Big) \bar{r}_m $ for all $\x \in \cup_{B \in \A_m} B $. 
    \end{itemize} 
\end{lemma} 

\begin{proof} 
    For each $m$, let $B_m^*$ denote the ball in $\A_m$ such that $B_m^* \ni \x^*$. For each $m$ and $ B \in \A_m $, we use $ \x_m (B) $ to denote the center of the ball $B$.  
    Let $\mathcal{E}$ be true. We know 
    \begin{align*} 
        0 
        \ge& \; 
        \wh{f}_m (B_m^{\min}) - \wh{f}_m (B_m^*) \\ 
        =& \; 
        \underbrace{\wh{f}_m (B_m^{\min}) - f ( \x_{m} ( B_m^{\min} ) )}_{\textcircled{1}} + \underbrace{f (\x_{m} (B_m^{\min}) ) - f (\x^*)}_{\textcircled{2}} \\ 
        &+ \underbrace{f (\x^*) - f ( \x_{m} ( B_m^* ) )}_{\textcircled{3}} + \underbrace{f ( \x_{m} (B_m^*) ) - \wh{f}_m ( B_m^* )}_{\textcircled{4}} \\ 
        \ge& \; 
        - 4 \sqrt{ \frac{\log T}{n_m} } + \lambda \( \D \( \x_{m} ( B_m^{\min} ) , \x^* \) \)^q - L \bar{r}_m^q, 
    \end{align*} 
    where for \textcircled{1} and \textcircled{4} we use Lemma \ref{lem:concen}, for \textcircled{3} we use property of the nondegenerate function, and \textcircled{2} is evidently nonnegative. 
    
    Since $4 \sqrt{ \frac{\log T}{n_m} } = \lambda \bar{r}_m^q$, we know that with high probability 
    \begin{align*} 
        \D \( \x_{m} ( B_m^{\min} ) , \x^* \) 
        \le 
        \( \frac{\lambda+L}{\lambda} \)^{1/q} \cdot \bar{r}_m . 
    \end{align*} 
    Let $B_m^*$ be the cube in $\mathcal{A}_m^{pre}$ that contains $\x^*$.
    This implies that $ D (B_m^*, B_m^{\min}) \le \( 2 + \( \frac{\lambda+L}{\lambda} \)^{1/q} \) \bar{r}_m $, and thus the optimal arm is not eliminated. Since (1) the optimal arm is not  eliminated, and (2) the diameter of $ \cup_{B \in \A_m} B $ is no larger than $ \( 2+ 2 \( \frac{\lambda+L}{\lambda} \)^{1/q} \) \bar{r}_m $, we have also proved the second item.  
    

    
\end{proof}

With Lemmas \ref{lem:concen} and \ref{lem:not-eli} in place, we are ready to prove Theorem \ref{thm:upper}. 


\begin{proof}[Proof of Theorem \ref{thm:upper}]

    For each $m$, we introduce $ S_m := \cup_{B \in \A_m} B $ to simplify notation. 
    By the algorithm procedure, the diameter of $S_m$ is bounded by 
    $ \( 3 + 2 \( \frac{\lambda + L}{\lambda} \)^{1/q} \) \bar{r}_m $. By Proposition \ref{prop:cover}, we have 
    \begin{align*} 
        | \A_m | 
        \le 
        \[ 3 + 2 \( \frac{\lambda + L}{\lambda} \)^{1/q} \]_2^d , 
    \end{align*} 
    which gives 
    \begin{align} 
        | \A_m^{pre} | 
        \le 
        \( \frac{ \bar{r}_{m-1} }{ \bar{r}_m} \)^d \[ 3 + 2 \( \frac{\lambda + L}{\lambda} \)^{1/q} \]_2^d . \label{eq:bound-Am}
    \end{align} 
    
    
    Note that the $m$-th batch incurs no regret if it is skipped. Thus it suffices to consider the case where the $m$-th batch is not skipped. 
    For $m = 2k-1$, we can bound the regret in the $ (2k-1) $-th batch (denoted by $R_{2k-1}$) by 
    \begin{align*} 
        R_{2k-1} 
        \le& \;  
        |\mathcal{A}_{2k-1}^{pre}| \cdot n_{2k-1} \cdot {L} A_{\lambda,L,q} \Bar{r}_{2k-2}^{q} \\  
        \le& \;  
        \( \frac{ \bar{r}_{m-1} }{ \bar{r}_m} \)^d B_{\lambda,L,q}^d n_{2k-1} \cdot {L} A_{\lambda,L,q} \Bar{r}_{2k-2}^{q} , 
    \end{align*} 
    where $ A_{\lambda, L, q}:= \( 2+ 2 \( \frac{\lambda+L}{\lambda} \)^\frac{1}{q} \)^q $ and $ B_{\lambda, L, q} := \[ 3+2\( \frac{\lambda+L}{\lambda} \)^\frac{1}{q} \]_2 $ are introduced for simplicity, and the first line uses Lemma \ref{lem:not-eli}. 
    Plugging $ n_{2k-1} = \frac{16 \log T}{\lambda^2 \bar{r}_{2k-1}^{2q} } $ into the above inequality gives 
    \begin{align*} 
         R_{2k-1} 
        \le &\; 
        \( \frac{ \bar{r}_{m-1} }{ \bar{r}_m} \)^d B_{\lambda,L,q}^d \cdot \frac{16 \log T}{\lambda^2 \bar{r}_{2k-1}^{2q} } \cdot {L} A_{\lambda,L,q} \cdot \Bar{r}_{2k-2}^{q} \\ 
        \le& \; 
        L A_{\lambda,L,q} B_{\lambda,L,q}^d \cdot \frac{16 \log T}{\lambda^2 } \cdot  \Bar{r}_{2k-2}^{q+d} \Bar{r}_{2k-1}^{-2q-d} \\ 
        \le& \; 
        L A_{\lambda,L,q} B_{\lambda,L,q}^d \cdot \frac{16 \log T}{\lambda^2 } \cdot \hat{r}_{{k-1}}^{q+d} \hat{r}_{{k}}^{-2q-d} , 
    \end{align*} 
    where the last inequality follows from the definitions of $ \hat{r}_m $ and $ \bar{r}_m $. 
    
    By definition of the sequence $\{ \hat{r}_m \}$, we have, for any $m$, $ \hat{r}_{m-1}^{q+d} \hat{r}_m^{-2q-d} = 2^{ - (q+d) \sum_{i=1}^{m-1} \hat{c}_i + (2q+d) \sum_{i=1}^m \hat{c}_i } = 2^{ (2q + d)\hat{c}_m + q \sum_{i=1}^{m-1} \hat{c}_i } = 2^{(2q + d ) \hat{c}_1 } $. Thus we can upper bound $R_{2k-1}$ by 
    \begin{align} 
        R_{2k-1} 
        \le& \;  
        L A_{\lambda,L,q} B_{\lambda,L,q}^d \cdot \frac{16}{\lambda^2} \cdot \sqrt{T \log T} . \label{eq:reg-odd} 
    \end{align} 
    
    
    For $m=2k$, the regret in batch $ 2k $ (written $R_{2k}$) is bounded by 
    \begin{align*} 
        R_{2k} & \le | \A_{2k}^{pre} | \cdot n_{2k} \cdot L \cdot A_{\lambda,L,q} \cdot \bar{r}_{2k}^q .
    \end{align*} 
    Bringing (\ref{eq:bound-Am}) and definition of $n_{2k}$ into the above inequality, and noticing $ \frac{\bar{r}_{m-1}}{\bar{r}_m} \le 2 $ (for even $m$) gives 
    \begin{align*} 
        R_{2k} 
        & \le 
        L 2^d A_{\lambda,L,q} B_{\lambda,L,q}^d \frac{16 \log T}{\lambda^2} \Bar{r}_{2k}^{-q} \\
        & \le 
        L 2^{d+q} A_{\lambda,L,q} B_{\lambda,L,q}^d \frac{16 \log T}{\lambda^2} \hat{r}_{{k}}^{-q}, 
    \end{align*} 
    where for the last inequality, we use definitions of $\bar{r}_m$ and $\hat{r}_m$ to get $ \bar{r}_{{2k}}^{-1} \le 2\cdot \hat{r}_{{k}}^{-1} $. 
    Again by definition of $\wh{r}_m$, we have $\hat{r}_{m}^{-1} \le 2^{\hat{c}_1\frac{1}{1-\hat{\eta}}}$, and thus the regret in batch $2k$ is at most 
    \begin{align} 
        R_{2k} \le &\; L 2^{d+q} A_{\lambda,L,q} B_{\lambda,L,q}^d \frac{16 }{\lambda^2} \sqrt{T \log T}. \label{eq:reg-even}
    \end{align} 

    For the cleanup phase, the regret (written $R_{2M+1}$) is bounded by  
    \begin{align} 
        R_{2M+1} 
        \le& \; 
        L A_{\lambda,L,q} \Bar{r}_{2M}^q T \nonumber \\ 
        \le&\;  
         L A_{\lambda,L,q} \sqrt{T \log T} \(\frac{T}{\log T} \)^{\frac{1}{2}\hat{\eta}^M} . \label{eq:reg-rest}
    \end{align} 

    Let there be in total $2M+1$ batches. Collecting terms from (\ref{eq:reg-odd}), (\ref{eq:reg-even}) and (\ref{eq:reg-rest}) gives 
    \begin{align*} 
         R^{GN}(T)  
        \le & \; 
        L A_{\lambda, L, q} B_{\lambda, L, q}^d \cdot \frac{16 }{\lambda^2 } \sqrt{T \log T} \cdot M \\
        &+ L2^{d+q} A_{\lambda, L, q} B_{\lambda, L, q}^d \cdot \frac{16 }{\lambda^2 } \sqrt{T \log T} \cdot M \\
        &+ L A_{\lambda,L,q} \sqrt{T \log T} \(\frac{T}{\log T} \)^{\frac{1}{2}\hat{\eta}^M} . 
    \end{align*}
     Now choose $M={\hat{M}}^* = \frac{\log\log \frac{T}{\log T}}{\log \frac{1}{\hat{\eta}}} $, we have $\hat{\eta}^{\hat{M}^*} = \( \log \frac{T}{\log T} \) ^{-1}$, then 
     \begin{align*}
          R^{GN} (T) 
         \le&\;   
         L (2^{d+q} + 1) A_{\lambda, L, q} \\ 
         & \cdot \( \frac{16 B_{\lambda,L,q}^d}{\lambda^2} \cdot  \frac{\log\log \frac{T}{\log T}}{\log (\frac{2q+d}{q+d}) } + e^\frac{1}{2} \) \sqrt{T \log T}. 
     \end{align*}
     

    With this choice of $M$, only $ \mathcal{O} \( \log \log T \) $ batches are needed. Q.E.D. 

\end{proof}

Following the proof of Theorem \ref{thm:upper}, we can readily prove Corollary \ref{cor:opt}. 

\begin{proof}[Proof of Corollary \ref{cor:opt}]

    Let the event $\mathcal{E}$ be true. 
    From Definition \ref{def:seq}, we know 
    \begin{align*} 
        \bar{r}_{2M}^q 
        \le& \; 
        2^{- q \hat{c}_1 \frac{1 - \hat{\eta}^{M} }{1 - \hat{\eta}} } \\ 
        =& \; 
        2^{ - \frac{q }{ 2 (2q+d) \log 2 } \log \frac{T}{\log T}  \cdot \frac{1 - \hat{\eta}^{M}}{1 - \hat{\eta} } } \\ 
        =& \; 
        \( \frac{T}{ \log T } \)^{ - \frac{q}{2(2q+d)} \cdot \frac{ 1 - \( \frac{q+d}{2q+d} \)^{M} }{ \frac{q}{2q+d} } } \\ 
        =& \; 
        \( \frac{T}{ \log T } \)^{ - \frac{1}{2} \(  1 - \( \frac{q+d}{2q+d} \)^{M} \) } \\ 
        =& \; 
        \sqrt{ \frac{\log T}{T} } \cdot \( \frac{T}{\log T} \)^{ \frac{1}{2} \hat{\eta}^{M} } . 
    \end{align*} 
    
    Let $ M = \frac{\log \log \frac{T}{\log T}}{\log \frac{1}{\hat{\eta}}} $, we have $\hat{\eta}^{M} = \( \log \frac{T}{ \log T} \)^{-1}$, and thus 
    \begin{align*}
        \bar{r}_{2M}^q 
        \leq e^{\frac{1}{2}} \sqrt{ \frac{ \log T }{ T } } . 
    \end{align*}
    
    By Lemma \ref{lem:not-eli}, we know, under event $\mathcal{E}$, 
    \begin{align*} 
        f (\x_{out} ) - f (\x^*) 
        \le & \; 
        L \D \( \x_{out} , \x^*\)^q \\
         \leq &\; L\( 2+ 2 \( \frac{\lambda+L}{\lambda} \)^\frac{1}{q} \)^q\bar{r}_{2M}^q \\
        \le& \;  
        e^{\frac{1}{2}} L \( 2+ 2 \( \frac{\lambda+L}{\lambda} \)^\frac{1}{q} \)^q \sqrt{ \frac{ \log T }{ T } }. 
    \end{align*} 

    We conclude the proof by noticing that the $\mathcal{E}$ holds true with probability exceeding $1 - \frac{2}{T}$. 
\end{proof}

\textcolor{black}{
In case there is a need to allocate computing resources ahead of time, meaning determining communication grids $\mathcal{T}= \{t_0,\cdots,t_M \}$ in advance.
We can predefine a sequence of communication grids $\mathcal{T}= \{t_0,\cdots,t_M \}$ while ensuring optimal $ \wt{\mathcal{O}} \( \sqrt{T} \) $ regret with $\mathcal{O}\( \log \log T \)$ communication points.
Detailed definition are provided in the proof of Corollary \ref{cor:sta}.  In the definition, two aspects require particular attention: firstly, the radius must decrease as the batch number increases, and secondly, the total number of steps must not exceed $T$.
\begin{corollary} 
    \label{cor:sta} 
    Under the assumptions of Theorem \ref{thm:upper}, we can define a static communication grids $\mathcal{T}_{s}= \{\tau_0,\cdots,\tau_{M_s} \}$ ahead of queries, where $M_s$ is the index of the last batch. If we run GN algorithm under static communication grid $\mathcal{T}_{s}$, it achieves the same regret bound and communication complexity as stated in Theorem \ref{thm:upper}. 
\end{corollary} 
}

\textcolor{black}{
\begin{proof}[Proof of Corollary \ref{cor:sta}]
    Here we show how to define $0=\tau_0, \tau_1,\tau_{2},\cdots$ sequentially.\\
    The setup of these static time grids also rely on the Rounded Radius sequence $\{\bar{r}_m\}_{m=1}^{2M}$ defined in Definition 3. We use $\{ \tau_m \}_{m=1}^{2M}$ to represent time grids, and $\{s_m\}_{m=1}^{2M}$ to denote the indexes corresponding to radius. Here $\tau_m$ is the ``deadline'' (specified below) for the $m$-th batch, and $\bar{r}_{s_m}$ is the radius of balls in the $m$-th batch.
    We will initialize the first batch using
    \begin{align*}
        \{\tau_1 , s_1\} = \left\{ \(\frac{1}{2\bar{r}_{1}}\)^d \cdot B_{\lambda,L,q}^d \cdot \frac{ 16 \log T }{ \lambda^2 \bar{r}_1^{2q} } , 1 \right\}.
    \end{align*}
    For $m \in \{1,2,\cdots,2M\}$, define $\{\tau_{m+1} , s_{m+1}\}$ based on $\{\tau_m , s_m\}$, iteratively as 
    \begin{align*} 
        \left\{
            \begin{aligned} 
                s_{m+1} :=&\; \inf\{ k : k > m \text{ and } \bar{r}_k \le \bar{r}_{s_m}\},\\
                \tau_{m+1} :=&\; \tau_m+(\bar{r}_{s_m}/\bar{r}_{s_{m+1}})^d 
                  \cdot B_{\lambda,L,q}^d \cdot \frac{ 16 \log T }{ \lambda^2 \bar{r}_{s_{m+1}}^{2q} }.
            \end{aligned} 
        \right.
    \end{align*} 
    With the complete couples $\{\tau_m , s_m\}_{m=1}^{2M}$, to fix $\mathcal{T}_{s}= \{\tau_0,\cdots,\tau_{M_s} \}$ we only need to determine the time to stop: $M_s$. Considering both the batch number and time horizon $T$, define $M_s := \max \{ m : m \le 2\hat{M}^* \text{ and } \tau_m \le T \}$ where $\hat{M}^*$ is defined in Proof of Theorem \ref{thm:upper}. By now, we figure out the static communication grids $\mathcal{T}_{s}= \{\tau_0,\cdots,\tau_{M_s} \}$.
    \\
    Under this definition, each batch will have more queries than an adaptive algorithm would, as we have substituted $\[ 3 + 2 \( (\lambda + L)/\lambda \)^{1/q} \]_2^d$ (the upper bound of $|\mathcal{A}_m|$) for $|\mathcal{A}_m|$. The GN algorithm can be fine-tuned to address this modification. After querying each ball in $\mathcal{A}_m^{pre}$ $n_m$ times, we randomly distribute the excess queries among $\mathcal{A}_m^{pre}$ until $\tau_m - \tau_{m-1}$ steps are used in batch $m$. This still maintain the same bound in lemmas and Theorem \ref{thm:upper}.
\end{proof}
}





\section{Lower Bound Analysis}

\label{sec:lower}


First of all, we need to identify a particular doubling metric space to work with. Hinted by the celebrated Assouad's embedding theorem, we turn to the Euclidean space with a specific metric. For any $d$, the doubling metric space we choose is $ (\R^{\lfloor d \rfloor }, \| \cdot \|_\infty) $. One important reason for this choice is that the doubling dimension of this space equal its dimension as a vector space. Throughout the rest of this paper, without loss of generality, we let $d$ be an integer, and consider the metric space $ (\R^{ d }, \| \cdot \|_\infty) $. 

\begin{remark}
    By Assouad’s embedding
theorem, one can embed a separable metric space $(\X, \D)$ with doubling number $N$ into a Euclidean space with some distortion, hence our research works in general doubling metric space. 
\end{remark}

After settling the metric space to work with, we still need to overcome  previously unencountered challenges. To further illustrate these challenges, let us review the lower bound strategy for Lipschitz bandits. In proving the lower bound for Lipschitz bandits \citep{kleinberg2005nearly,kleinberg2008multi,bubeck2011x}, one essentially use the packing/covering number for the underlying space, and this packing number essentially serves as number of arms in the lower bound proof. For our problem, however, the lower bound argument for Lipschitz does not carry through. The reasons are: 
\begin{itemize} 
    \item First and foremost, a nondegenerate function may be discontinuous. Restriction to Lipschitz bandit instances rules out a large class of problem instances.
    \item More importantly, in the lower bound argument for Lipschitz bandits, one construct instances with small ``peaks'' in the domain.
    We then let the height of the peak to decrease with the total time horizon $T$, so that no algorithm can quickly find the peaks for all instances. However, for nondegenerate functions, the nondegenerate parameters do not depend on $T$. Therefore, we are not allowed to tweak the landscape of the instances as freely as previously done for Lipschitz bandits. 
\end{itemize} 

On top of the above challenges, we need to incorporate the communication pattern into the entire analysis. To tackle all these difficulties, we use a \emph{bitten-apple} trick in the instance construction. Specific examples of bitten-apple instances are shown in Figures \ref{fig:lower-Mk} and \ref{fig:lower-jk}.





\subsection{The instances} 
\label{sec:inst}

To formally define the instances, we first-of-all partition the space $\R^d$ into $2^d$ orthants $O_1, O_2,\cdots, O_{2^d}$. 
We represent the natural numbers $1,2,\cdots, 2^d$ by a sequence of $+/-$ signs. That is, for any $k = 1,2,\cdots, 2^d$, we use $\( s_1^k, \cdots, s_d^k \) \in \{ -1 , +1 \}^d $ to represent $k$. This representation is equivalent to writing $k$ as a base-two number. 
For $k = 1,2,\cdots, 2^d$ and a number $\epsilon \in (0,1)$, we define $ \x_{k,\epsilon}^* = \( s_1^k \epsilon, s_2^k \epsilon, \cdots, s_d^k \epsilon \) $. Clearly, $ \| \x_{k,\epsilon}^* \|_\infty = \epsilon $ for $k = 1,2,\cdots, 2^d$. As a convention, we let $ O_1 $ be the orthant associated with $ \( +, +, \cdots, +  \) $.  

Firstly, we introduce a sequence of reference communication points $\mathcal{T}_r=\{T_1,\cdots,T_M\}$ and the corresponding gaps $\{\epsilon_1^q,\cdots,\epsilon_M^q\}$, defined as  
\begin{align} 
    \begin{split}
        T_j &=\lfloor T^{\frac{1-2^{-j}}{1-2^{-M}}} \rfloor, \\
    \qquad 
    \epsilon_j^q &= \frac{1}{4} \cdot \frac{\sqrt{2}}{2} \cdot \frac{\sqrt{2^d-1}}{2^q+2} \cdot \frac{1}{M} \cdot T^{-\frac{1}{2} \cdot \frac{1-2^{1-j}}{1-2^{-M}}},    
    \end{split}
     \label{eq:def-eps-T}
\end{align} 
for $j =1,2,\cdots,M$. 
Then we construct collections of instances $\mathcal{I}_1,\cdots,\mathcal{I}_M$. Each instance is defined by a mean loss function $f$ and a noise distribution. For our purpose, we let the noise be standard Gaussian. That is, the observed loss samples at $\x$ are $iid$ from the Gaussian distribution $\mathcal{N} (f(\x), 1)$. 
For $1\leq j \leq M-1 $, we let $\mathcal{I}_j= \left\{ I_{j,k} \right\}_{k=1}^{2^d-1}$ and the expected loss function of $I_{j,k}$ is defined as 
\begin{align} 
\label{eq:f_{j,k}} 
    f_{j,k}^{\epsilon_j} (\x) = 
    \begin{cases} 
        \| \x - \x_{k,\epsilon_j}^* \|_{\infty}^q - \| \x_{k,\epsilon_j}^*\|_{\infty}^q, &\\
        \qquad \text{if } \x \in \B (\x_{k,\epsilon_j}^*, \epsilon_j ) \backslash \B (0, \frac{\epsilon_j}{2}) , \\ 
        \| \x - \x_{2^d,\frac{\epsilon_M}{3}}^* \|_{\infty}^q - \| \x_{2^d,\frac{\epsilon_M}{3}}^*\|_{\infty}^q, &\\
        \qquad \text{if } \x \in \B (\x_{2^d,\frac{\epsilon_M}{3}}^*, \frac{\epsilon_M}{3} ) \backslash \B (0, \frac{\epsilon_M}{6}) , \\ 
        \| \x \|_{\infty}^q, \qquad \text{otherwise. } 
    \end{cases} 
\end{align} 
For $j=M$, we let $\mathcal{I}_M=\{ I_M\}$ and the expected loss function of $I_M$ is defined as 
\begin{align} 
\label{eq:f_{M,k}}
    f_{M,k}^{\epsilon_M} (\x)
    = 
    \begin{cases} 
      \| \x - \x_{2^d,\frac{\epsilon_M}{3}}^* \|_{\infty}^q - \| \x_{2^d,\frac{\epsilon_M}{3}}^*\|_{\infty}^q, &\\
      \qquad \text{if } \x \in \B (\x_{2^d,\frac{\epsilon_M}{3}}^*, \frac{\epsilon_M}{3} ) \backslash \B (0, \frac{\epsilon_M}{6}) , \\ 
        \| \x \|_{\infty}^q , \qquad \text{otherwise. } 
    \end{cases} 
\end{align} 

Note that $f_{M,k}^{\epsilon_M} (\x)$ is independent of $k$. Here we keep the subscript $k$ for notational consistency. Figures \ref{fig:lower-Mk} and \ref{fig:lower-jk} plot examples of $f_{j,k}^{\epsilon_j}$ and $f_{M,k}^{\epsilon_M}$. 

\begin{figure}[h!]
    \centering
    \begin{minipage}{0.48\textwidth}
        \centering
        \includegraphics[width = 0.8\textwidth]{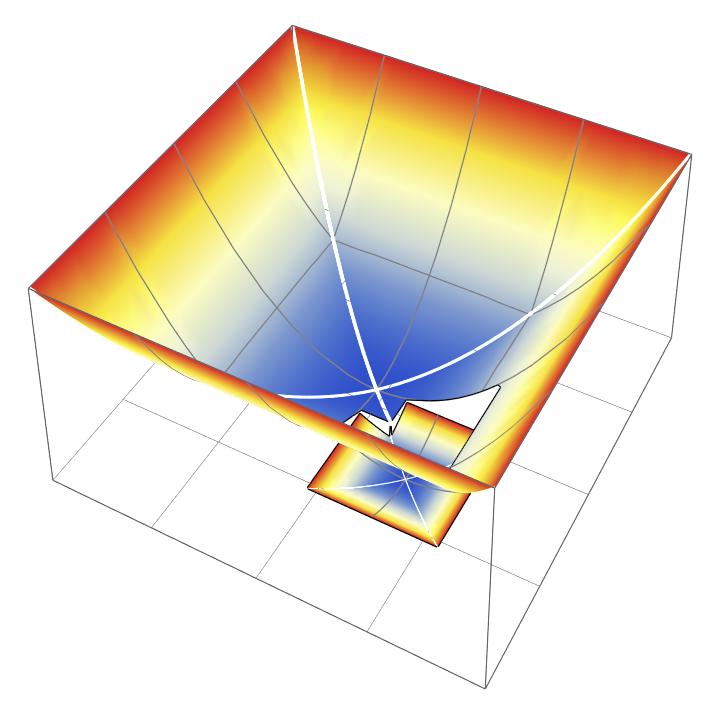}
    \end{minipage}
    \begin{minipage}{0.48\textwidth}
        \centering
        \includegraphics[width = 0.8\textwidth]{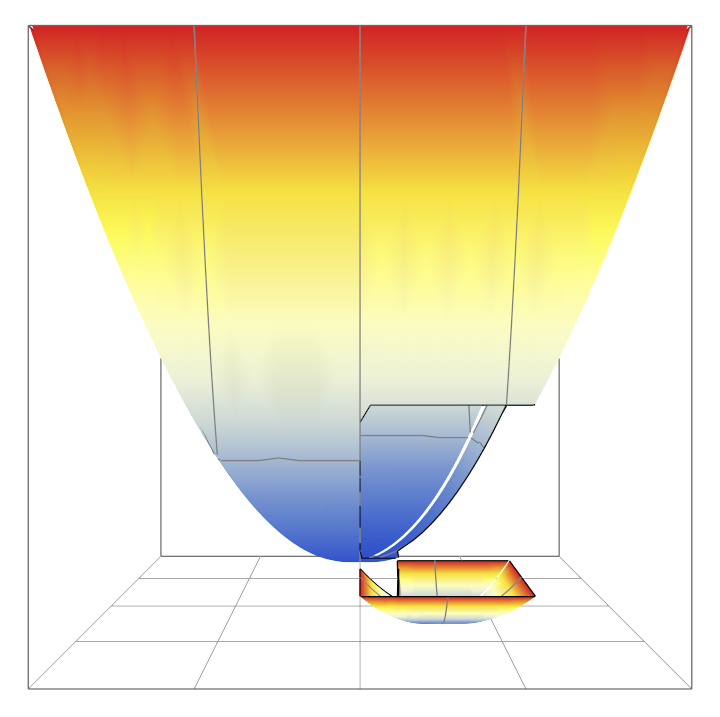}
    \end{minipage}
    \caption{Example plot of $f_{M,k}^{\epsilon_M} (\x)$ with $d=q=2$. The two graphs come from different views of the same function. 
    }
    \label{fig:lower-Mk}
\end{figure}

\begin{figure*}[h!]
    \centering
        \subfloat[$f_{j,1}^{\epsilon_j} (\x) $]{\includegraphics[width=0.32\textwidth ]{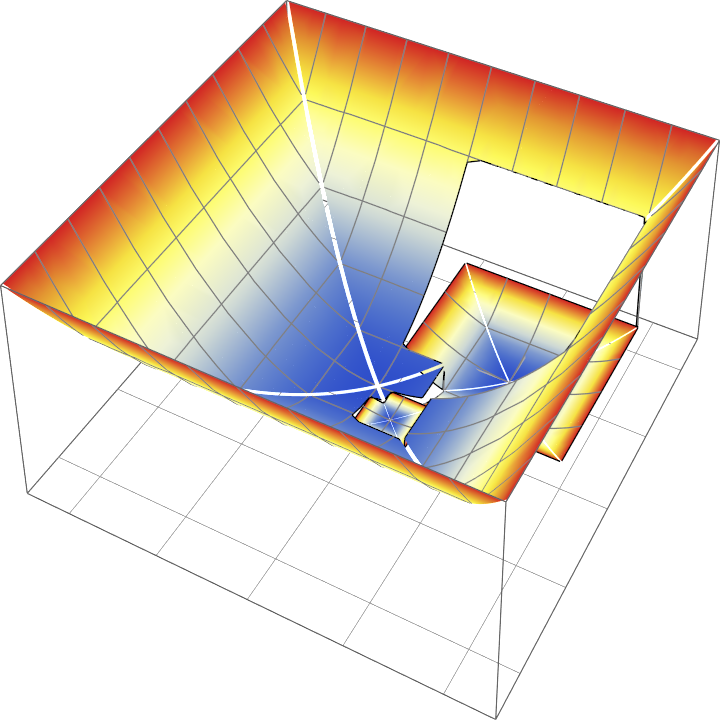} }
        \subfloat[$f_{j,2}^{\epsilon_j} (\x) $]{\includegraphics[width=0.32\textwidth]{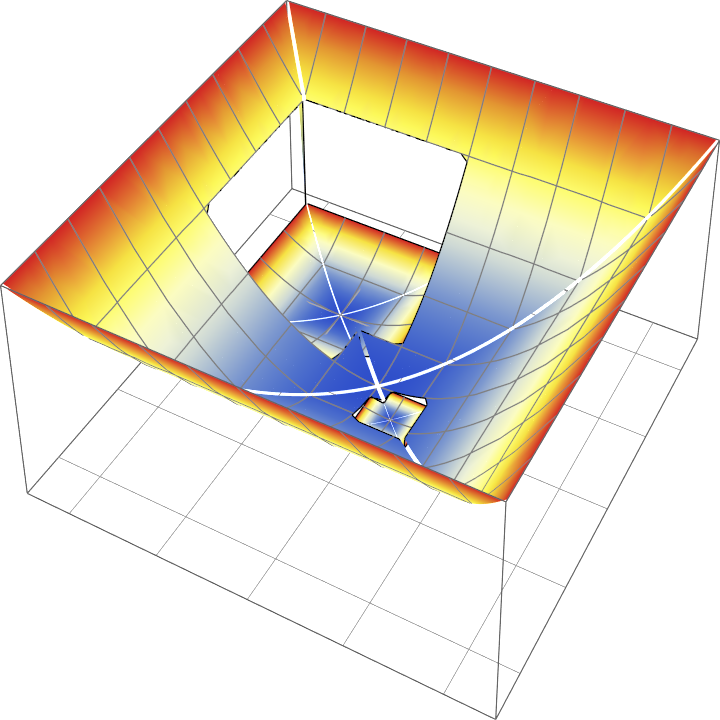} }
        \subfloat[$f_{j,3}^{\epsilon_j} ( \x ) $]{\includegraphics[width=0.32\textwidth]{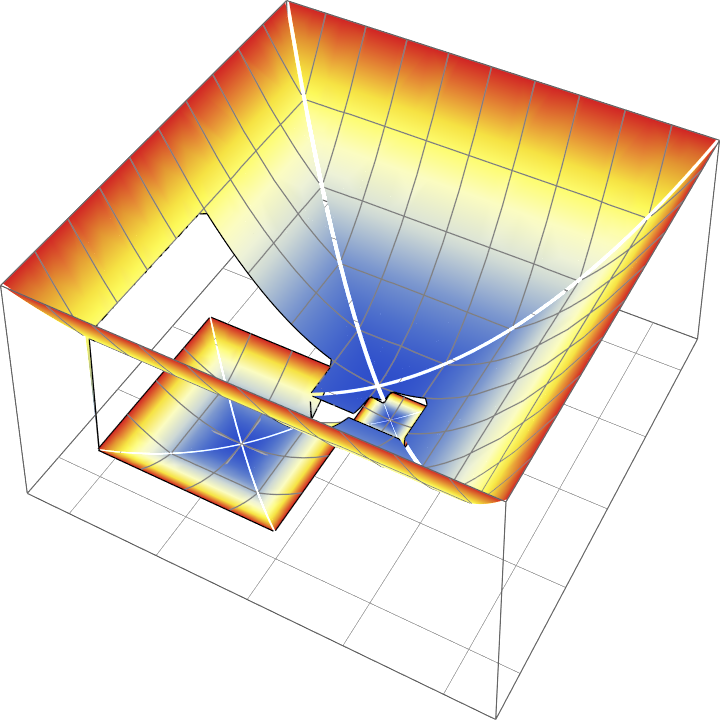} } 
    \caption{Example instance $f_{j,k}^{\epsilon_j} (\x)$ with $d=q=2$ for $1\leq j \leq M-1 $. The above three graphs from left to right show $f_{j,k}^{\epsilon_j}$ for $k=1,2,3$. }
    \label{fig:lower-jk}
\end{figure*}



On the basis of $ \{  f_{j,k} \}_{j \in [M], k \in [2^{d}-1]} $, we 
construct another series of problem instances $\{I_{j,k,l}\} _{j \in [M], k \in [2^d - 1],l \in [2^d]} $: 
\begin{itemize}[leftmargin = *]
    \item For $j < M$, $l \neq k$ and $l < 2^d$, the loss function of problems instance $I_{j,k,l}$ is defined as 
    \begin{align*} 
        &\; f_{j,k,l}^{\epsilon_j} (\x)  \\
        =
        &\; 
        \begin{cases} 
            \| \x - \x_{k,\epsilon_j}^* \|_{\infty}^q - \| \x_{k,\epsilon_j}^*\|_{\infty}^q, &\\
            \qquad \qquad \text{if } \x \in \B (\x_{k,\epsilon_j}^*, \epsilon_j ) \backslash \B (0, \frac{\epsilon_j}{2}) , \\ 
            \| \x - 2 ^\frac{1}{q} \cdot \x_{l,\epsilon_j}^* \|_{\infty}^q - \|2^\frac{1}{q} \cdot \x_{l,\epsilon_j}^*\|_{\infty}^q, & \\
            \qquad \qquad \text{if } \x \in \B (2^\frac{1}{q} \cdot \x_{l,\epsilon_j}^*, 2^\frac{1}{q} \cdot \epsilon_j ) \backslash \B (0, \frac{2^\frac{1}{q} \cdot \epsilon_j}{2}) , \\ 
           \| \x - \x_{2^d,\frac{\epsilon_M}{3}}^* \|_{\infty}^q - \| \x_{2^d,\frac{\epsilon_M}{3}}^*\|_{\infty}^q, & \\
           \qquad \qquad \text{if } \x \in \B (\x_{2^d,\frac{\epsilon_M}{3}}^*, \frac{\epsilon_M}{3} ) \backslash \B (0, \frac{\epsilon_M}{6}) , \\ 
            \| \x \|_{\infty}^q, \quad \text{otherwise}. 
        \end{cases} 
    \end{align*} 
    \item For $j < M$, $l = k < 2^d$, 
    we let $f_{j,k,l}^{\epsilon_j} (\x):=f_{j,k}^{\epsilon_j} (\x)$, which is defined in (\ref{eq:f_{j,k}}). 
    \item For $ j < M $, $k < 2^d$, and $l = 2^d$, we define 
    \begin{align*}
        &\; 
        f_{j,k,2^d}^{\epsilon_j} (\x) \\
        =& \;
        \begin{cases} 
            \| \x - \x_{k,\epsilon_j}^* \|_{\infty}^q - \| \x_{k,\epsilon_j}^*\|_{\infty}^q, & \\
            \qquad \qquad \text{if } \x \in \B (\x_{k,\epsilon_j}^*, \epsilon_j ) \backslash \B (0, \frac{\epsilon_j}{2}) , \\ 
            \| \x - 2 ^\frac{1}{q} \cdot \x_{2^d,\epsilon_j}^* \|_{\infty}^q - \|2^\frac{1}{q} \cdot \x_{2^d,\epsilon_j}^*\|_{\infty}^q, &\\
            \qquad \qquad \text{if } \x \in \B (2^\frac{1}{q} \cdot \x_{2^d,\epsilon_j}^*, 2^\frac{1}{q} \cdot \epsilon_j ) \backslash \B (0, \frac{2^\frac{1}{q} \cdot \epsilon_j}{2}) , \\ 
            \| \x \|_{\infty}^q, \quad \text{otherwise. } 
        \end{cases} 
    \end{align*} 
    \item For $j=M$, $k < 2^d$, and $l< 2^d $, the corresponding loss function is defined as
    \begin{align*}
        &\;f_{M,k,l}^{\epsilon_M} (\x) \\
        =& \;
        \begin{cases} 
            \| \x - 2 ^\frac{1}{q} \cdot \x_{l,\frac{\epsilon_M}{3}}^* \|_{\infty}^q - \| 2 ^\frac{1}{q} \cdot \x_{l,\frac{\epsilon_M}{3}}^*\|_{\infty}^q, & \\ \qquad \quad
             \text{if } \x \in \B (2 ^\frac{1}{q} \cdot \x_{l,\frac{\epsilon_M}{3}}^*,2^\frac{1}{q} \cdot \frac{\epsilon_M}{3} ) \backslash \B (0, \frac{2^\frac{1}{q} \cdot \epsilon_M}{6}) , \\ 
            \| \x - \x_{2^d,\frac{\epsilon_M}{3}}^* \|_{\infty}^q - \| \x_{2^d,\frac{\epsilon_M}{3}}^*\|_{\infty}^q, & \\ \qquad \quad
             \text{if } \x \in \B (\x_{2^d,\frac{\epsilon_M}{3}}^*, \frac{\epsilon_M}{3} ) \backslash \B (0, \frac{\epsilon_M}{6}) , \\ 
            \| \x \|_{\infty}^q, \qquad \text{otherwise. } 
        \end{cases}
    \end{align*}
    \item 
    For $j=M$, $k < 2^d$ and $l=2^d$, we define $f_{M,k,2^d}^{\epsilon_M} (\x):=f_{M,k}^{\epsilon_M} (\x)$. For the case where $j=M$, we keep the subscript $k$ for the same reason as in (\ref{eq:f_{M,k}}). 
\end{itemize}



Figure \ref{fig:fjkl} depicts the partitioning of space for the function $ f_{j,k,l}^{\epsilon_j} $ ($j< M, l < 2^d, l \neq k$). In orthant $ O_k $, $O_l$ and $O_{2^d}$, the function $ f_{j,k,l}^{\epsilon_j} $ differs from $ \| \x \|_\infty^q $ in a region of a bitten-apple shape.  

\begin{figure} 
    \centering 
    \includegraphics[width=3.25in]{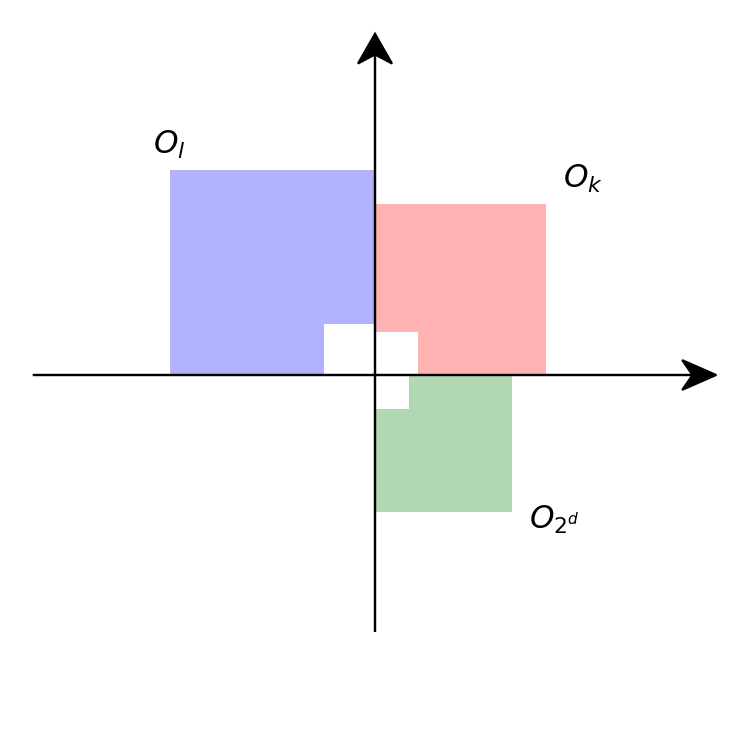}
    \caption{An illustration of how the space is partitioned for function $ f_{j,k,l}^{\epsilon_j} $. In some particular orthant, the function $ f_{j,k,l}^{\epsilon_j} $ differs from $\| \x \|_\infty^q$ in regions that resemble a bitten-apple shape. Such regions are illustrated as shaded areas in the figure. 
    \label{fig:fjkl}
    }
\end{figure}

First of all, we verify that these functions are nondegenerate functions.

\begin{proposition} 
    \label{prop:check-nondegen-1}
    The functions $ \{ f_{j,k}^{\epsilon_j} \}_{j \in [M], k \in [2^d-1]} $ 
    are nondegenerate with parameters independent of time horizon $T$, the doubling dimension $d$, and rounds of communications $M$. 
\end{proposition} 

\begin{proof}[Proof of Proposition \ref{prop:check-nondegen-1}]
    We first consider $f_{j,k}^{\epsilon_j}$. Note that the minimum of $f_{j,k}^{\epsilon_j}$ is obtained at $\x_{k,\epsilon_j}^*$. 
    
    \textbf{The lower bound: } 
    
    For $ \x \in \B (\x_{k,\epsilon_j}^*, \epsilon_j ) \backslash \B (0, \frac{\epsilon_j}{2}) $, we have $ f_{j,k}^{\epsilon_j} ( \x ) - f_{j,k}^{\epsilon_j} ( \x_{k,\epsilon_j}^* ) = \| \x - \x_{k,\epsilon_j}^* \|_\infty^q $, which clearly satisfies the nondegenerate condition. 
    
    
    For $ \x \in \B (\x_{2^d,\frac{\epsilon_M}{3}}^*, \frac{\epsilon_M}{3} ) \backslash \B (0, \frac{\epsilon_M}{6}) $, since $ \epsilon_M \le \epsilon_j $ for all $j = 1,2,\cdots,M$, we have, 
    \begin{align*} 
        &\| \x - \x_{k,\epsilon_j}^* \|_\infty^q 
        \le 
        \( 2 \| \x_{2^d, \frac{\epsilon_M}{3} }^* \|_\infty +  \| \x_{k,\epsilon_j}^* \|_\infty \)^q \\
        \le& 
        3^q \| \x_{k,\epsilon_j}^* \|_\infty^q 
        \le  
        3^q \cdot 3^q \( \| \x_{k,\epsilon_j}^* \|_\infty - \| \x_{2^d, \frac{\epsilon_M}{3} }^* \|_\infty \)^q \\
        \le&   
        9^q \( \| \x_{k,\epsilon_j}^* \|_\infty^q - \| \x_{2^d, \frac{\epsilon_M}{3} }^* \|_\infty^q + \| \x - \x_{2^d, \frac{\epsilon_M}{3} }^* \|_\infty^q \) \\
        =&  
        9^q \( f_{j,k}^{\epsilon_j} (\x) - f_{j,k}^{\epsilon_j} ( \x_{k,\epsilon_j}^* ) \) . 
    \end{align*} 


    For $\x$ in other parts of the domain, we have 
    \begin{align*} 
        &\; f_{j,k}^{\epsilon_j} (\x) - f_{j,k}^{\epsilon_j} ( \x_{k,\epsilon_j}^* ) \\
        =&\;  
        \| \x \|_{\infty}^q + \| \x_{k,\epsilon_j}^* \|_\infty^q 
        \ge
        \frac{1}{2^{q-1}} \| \x - \x_{k,\epsilon_j}^* \|_\infty^q , 
    \end{align*} 
    where the last inequality uses convexity of $ \| \cdot \|_\infty^q $ and Jensen's inequality. 

    \textbf{The upper bound: } 
    
    For $ \x \in \B (\x_{k,\epsilon_j}^*, \epsilon_j ) \backslash \B (0, \frac{\epsilon_j}{2}) $, the nondegenerate condition holds true. 

    For $ \x \in \B (\x_{2^d,\frac{\epsilon_M}{3}}^*, \frac{\epsilon_M}{3} ) \backslash \B (0, \frac{\epsilon_M}{6}) $, 
    \begin{align*}
        &\; f_{j,k}^{\epsilon_j} (\x) - f_{j,k}^{\epsilon_j} (\x_{k,\epsilon_j}^*) \\
        =&\;  
        \| \x - \x_{2^d,\epsilon_M}^* \|_\infty^q + \epsilon_j^q - \( \frac{\epsilon_M}{3} \)^q \\ 
        \le& \;  
        2^{q-1} \( \| \x - \x_{k,\epsilon_j}^* \|_\infty^q + \| \x_{k,\epsilon_j}^* - \x_{2^d,\frac{\epsilon_M}{3}}^* \|_\infty^q \) + \epsilon_j^q \\ 
        =& \;  
        2^{q-1} \| \x - \x_{k,\epsilon_j}^* \|_\infty^q + 2^{q-1} \( \epsilon_j + \frac{\epsilon_M}{3} \)^q + \epsilon_j^q \\ 
        \le&\; 
        \( 2^q + 1 \)^2 \| \x - \x_{k,\epsilon_j}^* \|_\infty^q ,  
    \end{align*} 
    where the inequality on the first line uses convexity of $ \| \cdot \|_\infty^q $ and Jensen's inequality. 

    For $ \x $ in other parts of the domain, we have 
    \begin{align*} 
        &\; f_{j,k}^{\epsilon_j} (\x) - f_{j,k}^{\epsilon_j} ( \x_{k,\epsilon_j}^* ) 
        = 
        \| \x \|_{\infty}^q + \| \x_{k,\epsilon_j}^* \|_\infty^q \\
        \le&\;  
        2^{q-1} \( \| \x - \x_{k,\epsilon_j}^* \|_{\infty}^q + \| \x_{k,\epsilon_j}^* \|_\infty^q  \) + \| \x_{k,\epsilon_j}^* \|_\infty^q . 
    \end{align*} 
    Since $ \| \x_{k,\epsilon_j}^* \|_\infty \le 2 \| \x - \x_{k,\epsilon_j}^* \|_\infty $ for $\x \notin \( \B (\x_{k,\epsilon_j}^*, \epsilon_j ) \backslash \B (0, \frac{\epsilon_j}{2}) \)$, we continue from the above inequality and get 
    \begin{align*} 
        f_{j,k}^{\epsilon_j} (\x) - f_{j,k}^{\epsilon_j} ( \x_{k,\epsilon_j}^* ) 
        \le 
        \( 2^q + 1 \)^2 \| \x - \x_{k, \epsilon_j}^* \|_\infty^q . 
    \end{align*} 

    Following the same procedure, we can check that the nondegenerate condition holds true for the function $f_{M,k}^{\epsilon_M}$. 
    
    
\end{proof}

\begin{proposition} 
    \label{prop:check-nondegen-2}
    The functions $ \{ f_{j,k,l}^{\epsilon_j} \}_{j \in [M], k \in [2^d - 1], l \in [2^d]} $ 
    are nondegenerate with parameters independent of time horizon $ T $, the doubling dimension $ d $, and rounds of communications $M$. 
\end{proposition}

\begin{proof}[Proof of Proposition \ref{prop:check-nondegen-2}]
    For $j \leq M-1$,
    first, consider $l < 2^d$. 
    For $\x \in \B (\x_{k,\epsilon_j}^*, \epsilon_j ) \backslash \B (0, \frac{\epsilon_j}{2})$, we have 
    \begin{align*} 
        & \; f_{j,k,l}^{\epsilon_j} ( \x ) - f_{j,k,l}^{\epsilon_j} (2^{1/q} \cdot\x_{l,\epsilon_j}^* ) \\ 
        =&\;  
        \| \x - \x_{k,\epsilon_j}^* \|_{\infty}^q - \| \x_{k,\epsilon_j}^*\|_{\infty}^q + \| 2^{1/q}\cdot \x_{l, \epsilon_j}^*\|_{\infty}^q \\ 
        \ge& \;  
        \| \x_{l,\epsilon_j}^* \|_\infty^q 
        \ge 
        \( \frac{1}{6} \| \x - 2^{1/q}\cdot\x_{l, \epsilon_j}^* \|_\infty \)^q , 
    \end{align*} 
    and 
    \begin{align*} 
        &\; f_{j,k,l}^{\epsilon_j} ( \x ) - f_{j,k,l}^{\epsilon_j} (2^{1/q} \cdot\x_{l,\epsilon_j}^* ) \\
        =&\;  
        \| \x - \x_{k,\epsilon_j}^* \|_{\infty}^q - \| \x_{k,\epsilon_j}^*\|_{\infty}^q + \| 2^{1/q}\cdot \x_{l, \epsilon_j}^*\|_{\infty}^q \\ 
        \le& \; 
        2^{q-1} \| \x - 2^{1/q} \cdot\x_{l,\epsilon_j}^* \|_\infty^q + 2^{q-1} \| \x_{k,\epsilon_j}^* - 2^{1/q} \cdot\x_{l, \epsilon_j}^* \|_\infty^q \\
        & - \| \x_{k,\epsilon_j}^*\|_{\infty}^q + \| 2^{1/q} \cdot\x_{l, \epsilon_j}^*\|_{\infty}^q \\ 
        \le & \; 
        2^{q-1} \| \x - 2^{1/q} \x_{l, \epsilon_j}^* \|_\infty^q + 2^{q-1}\cdot 3^q \epsilon_j^q - \epsilon_j^q + 2 \epsilon_j^q \\
        \le&\;  
        \(3^q+1\)^2 \| \x - 2^{1/q} \x_{l, \epsilon_j}^* \|_\infty^q , 
    \end{align*} 
    where the last inequality uses that $ \epsilon_j \le \| \x - 2^{1/q} \cdot \x_{l, \epsilon_j }^* \|_\infty $.

    For $\x$ in other parts of the domain, we use Proposition \ref{prop:check-nondegen-1}. 
    For the case where $l = 2^d$, we also apply Proposition \ref{prop:check-nondegen-1}. 

    For $j=M$, the proof follows analogously. 


\end{proof} 

In addition, we prove that the loss functions we construct satisfy the following properties. 

\begin{proposition}
    \label{prop:check-ins-gap-1}
    For any $j = 1,2,\cdots, M -1 $ and $k = 1,2,\cdots,2^d - 1$, it holds that 
    \begin{align*}
        &\; \left| f_{j,k}^{\epsilon_j} (\x) - f_{M,k}^{\epsilon_M} \(\x\) \right| \\ 
        \le& \;  
        \begin{cases}
            (2^q + 2) \epsilon_j^q , & \text{if } \x \in \B (\x_{k,\epsilon_j}^*, \epsilon_j ) \backslash \B (0, \frac{\epsilon_j}{2}) , \\ 
            0, & \text{otherwise}. 
        \end{cases}
    \end{align*}
\end{proposition} 

\begin{proof} 
    For $\x \in \B (\x_{k,\epsilon_j}^*, \epsilon_j ) \backslash \B (0, \frac{\epsilon_j}{2}) $, it holds that 
    \begin{align*} 
        &\; \left| f_{j,k}^{\epsilon_j} ( \x ) - f_{M,k}^{\epsilon_M} \(\x\) \right| \\
        =& \;  
        \left| \| \x - \x_{k,\epsilon_j}^* \|_{\infty}^q - \| \x_{k,\epsilon_j}^*\|_{\infty}^q - \| \x \|_\infty^q \right| \\
        \le& \;  
        \epsilon_j^q + \epsilon_j^q + 2^q \epsilon_j^q 
        = 
        (2^q + 2) \epsilon_j^q . 
    \end{align*} 
    For $\x \notin \B (\x_{k,\epsilon_j}^*, \epsilon_j ) \backslash \B (0, \frac{\epsilon_j}{2}) $, $ f_{j,k}^{\epsilon_j} (\x) $ is identical to $ f_{M,k}^{\epsilon_M} \(\x\) $. This concludes the proof. 
\end{proof}

Now for simplicity, we introduce the following notation: For $k = 1,2,\cdots, 2^d$, define 
\begin{align*} 
    S_k^\epsilon 
    := 
        \B (\x_{k,\epsilon}^*, \epsilon )  .   
\end{align*}

\begin{proposition}
    \label{prop:check-ins-gap-2}

    It holds that 
    \begin{itemize}
        \item If $ j < M $, $ k < 2^d $ and $l \neq k$
        \begin{align*}
            | f_{j,k,l}^{\epsilon_j} (\x) - f_{j,k,k}^{\epsilon_j}(\x) | 
            \le  
            \begin{cases}
                2(2^q + 2) \epsilon_j^q, & \text{if } \x \in S_l^{2^{\frac{1}{q}} \epsilon_j } 
                \\ 
                0, & \text{otherwise}. 
            \end{cases}
        \end{align*}
        \item 
        Also, if $ k < 2^d $ and $l < 2^d$, 
        \begin{align*} 
            | f_{M,k,l}^{\epsilon_M} (\x) \! - \! f_{M,k,2^d}^{\epsilon_M}(\x) | 
            \! \le \!
            \begin{cases}  
                2(2^q \!+\! 2) \epsilon_M^q, & \text{if } \x \! \in \! S_l^{2^{\frac{1}{q}} \epsilon_M } 
                \\
                0, & \text{otherwise}. 
            \end{cases} 
        \end{align*} 
        \item On instance $I_{j,k,l}$ ($j \in [M], k \in [2^d-1], l \in[2^d]$), pulling an arm that is not in $ S_l^{2^{1/q} \epsilon_j } $ incurs a regret no smaller than $\frac{\epsilon_j^q}{3^q}$. 
    \end{itemize} 
\end{proposition}

\begin{proof} 
    \textbf{The first item.} 
    
    \textbf{Case I: $j < M$ and $l < 2^d$.} 
    For $\x \in \B (2^\frac{1}{q} \cdot \x_{l,\epsilon_j}^*, 2^\frac{1}{q} \cdot \epsilon_j ) \backslash \B (0, \frac{2^\frac{1}{q} \cdot \epsilon_j}{2}) \subseteq S_l^{2^{1/q} \epsilon_j } 
    $, it holds that 
    \begin{align*} 
        &\; \left| f_{j,k,l}^{\epsilon_j} ( \x ) - f_{j,k,k}^{\epsilon_j} \(\x\) \right| \\
        =& \;  
        \left| \| \x - 2 ^\frac{1}{q} \cdot \x_{l,\epsilon_j}^* \|_{\infty}^q - \|2^\frac{1}{q} \cdot \x_{l,\epsilon_j}^*\|_{\infty}^q - \| \x \|_{\infty}^q \right| \\
        \le& \;  
        2 \( 2^{q} + 2 \) \epsilon_j^q . 
    \end{align*} 
    For $ \x \notin S_l^{2^{1/q} \epsilon_j } = \B \( 2^{1/q} \cdot \x_{l,\epsilon_j }^* , 2^{1/q} \cdot \epsilon_j \) $, $ f_{j,k,l}^{\epsilon_j} (\x)  $ is identical to $ f_{j,k,k}^{\epsilon_j} \(\x\) $. 

    \textbf{Case II: $j < M$ and $ l = 2^d $.}
    For $\x \in \B (2^\frac{1}{q} \cdot \x_{2^d,\epsilon_j}^*, 2^\frac{1}{q} \cdot \epsilon_j )  $, it holds that 
    \begin{align*} 
        & \; \left| f_{j,k,l}^{\epsilon_j} ( \x ) - f_{j,k,k}^{\epsilon_j} \(\x\) \right| \\ 
        \le & 
        \max 
        \begin{cases} 
            \big| \| \x - 2 ^\frac{1}{q} \cdot \x_{2^d,\epsilon_j}^* \|_{\infty}^q - \|2^\frac{1}{q} \cdot \x_{2^d,\epsilon_j}^*\|_{\infty}^q - \| \x \|_{\infty}^q \big|, \\ \qquad \qquad\qquad\qquad\qquad\qquad\qquad\qquad\quad\,\text{if } \textcircled{1} ; \\ 
            \big| \| \x - 2 ^\frac{1}{q} \cdot \x_{2^d,\epsilon_j}^* \|_{\infty}^q - \|2^\frac{1}{q} \cdot \x_{2^d,\epsilon_j}^*\|_{\infty}^q \\
            \;\;\; -\| \x - \x_{2^d,\frac{\epsilon_M}{3} }^* \|_{\infty}^q + \| \x_{2^d,\frac{\epsilon_M}{3} }^*\|_{\infty}^q \big|, \quad\quad \text{if } \textcircled{2} ; \\ 
            \big|\| \x \|_{\infty}^q- \| \x - \x_{2^d,\frac{\epsilon_M}{3}}^* \|_{\infty}^q + \|\x_{2^d,\frac{\epsilon_M}{3}}^*\|_{\infty}^q  \big|, \\ \qquad \qquad\qquad\qquad\qquad\qquad\qquad\qquad\quad\,\,\text{if } \textcircled{3} ; \\
            0, \qquad\qquad\qquad\qquad\qquad\qquad\quad\qquad\quad\:\, \text{if } \textcircled{4} ;
        \end{cases} \\ 
        \le& \;  
        2 \( 2^{q} + 2 \) \epsilon_j^q , 
    \end{align*} 
    where \textcircled{1} stands for 
    \begin{align*}
        \x \in \B \( 2^{1/q} \cdot \x_{2^d, \epsilon_j }^* , 2^{1/q} \cdot \epsilon_j \) \backslash \B \( 0, \frac{2\epsilon_M}{3} \),
    \end{align*}
    \textcircled{2} stands for 
    \begin{align*}
        \x \in \B \( \x_{ 2^d, \frac{\epsilon_M}{3} }^* , \frac{\epsilon_M}{3} \) \backslash \B \( 0 , \frac{2^{1/q} \epsilon_j }{2} \) , 
    \end{align*}
    \textcircled{3} stands for 
    \begin{align*}
        \x \in \B \( \x_{2^d, \frac{2^{1/q} \cdot \epsilon_j }{ 4 } }^* , \frac{2^{1/q} \epsilon_j }{ 4 } \) \backslash \B \( 0, \frac{\epsilon_M}{6} \) , 
    \end{align*}
    and \textcircled{4} stands for $\x$ in other parts of 
    $ \B \( 2^{1/q} \cdot \x_{2^d, \epsilon_j }^* , 2^{1/q} \cdot \epsilon_j \) $, and the last inequality in \textbf{Case II} uses that $\epsilon_M \le \epsilon_j$ for $j \le M $. 
    The above derivation is valid even if some of \textcircled{1}--\textcircled{4} are empty. 
    
    Outside of $ \B ( 2^{1/q} \cdot \x_{2^d,\epsilon_j}^* , 2^{1/q} \cdot \epsilon_j ) $, $ f_{j,k,2^d}^{\epsilon_j} $ is identical to $f_{j,k,k}^{\epsilon_j}$. 

    \textbf{The second item.} 
    For $\x \in \B \( 2^{\frac{1}{q}} \cdot \x_{l, \frac{\epsilon_M }{3} }^*, 2^{\frac{1}{q}} \cdot \frac{\epsilon_M }{3} \) $, 
    \begin{align*} 
        &\; | f_{M,k,l}^{\epsilon_M} (\x) - f_{M,k,2^d}^{\epsilon_M}(\x) | \\ 
        =& \;  
        \left| \| \x - 2^{\frac{1}{q}} \cdot \x_{l,\frac{\epsilon_M}{3}}^* \|_\infty^q - \| 2^{\frac{1}{q}} \cdot \x_{l,\frac{\epsilon_M}{3}}^* \|_\infty^q - \| \x \|_\infty^q \right| \\ 
        \le& \; 
        2 \epsilon_M^q + \frac{2\cdot 2^q}{3^q} \epsilon_M^q 
        \le 
        2(2^q + 2) \epsilon_M^q . 
    \end{align*} 
    Outside of $ \B \( 2^{\frac{1}{q}} \cdot \x_{l, \frac{\epsilon_M }{3} }^*, 2^{\frac{1}{q}} \cdot \frac{\epsilon_M }{3} \) $, $f_{M,k,l}^{\epsilon_M} (\x)$ is identical to $ f_{M,k,2^d}^{\epsilon_M}(\x) $. 
    


    

    \textbf{The third item.} For this part, we detail a proof for the case where $ j < M $, $l < 2^d$ and $l \neq k$. The other cases are proved using similar arguments. 
    
    \textbf{Case I: $j < M$, $l < 2^d$, and $l \neq k$.} 
    When $\x \notin S_l^{2^{\frac{1}{q}} \epsilon_j } 
    $, it holds that 
    \begin{align*} 
        & \; f_{j,k,l}^{\epsilon_j} (\x) - f_{j,k,l}^{\epsilon_j} (2^{\frac{1}{q}} \cdot \x_{l,\epsilon_j}^*) \\ 
        =&\;  
        f_{j,k,l}^{\epsilon_j} (\x) + \| 2^{\frac{1}{q}} \cdot \x_{l,\epsilon_j}^* \|_\infty^q \\ 
        \ge& \;
        \min 
        \begin{cases}
            2 \| \x_{l,\epsilon_j}^* \|_\infty^q - \| \x_{k,\epsilon_j}^* \|_\infty^q , & \\
            \qquad \qquad \text{if } \x \in \B \( \x_{k, \epsilon_j }^*, \epsilon_j \) \backslash \B \( 0, \frac{\epsilon_j}{2} \) \\ 
            2 \| \x_{l,\epsilon_j}^* \|_\infty^q - \| \x_{2^d,\frac{\epsilon_M}{3}}^* \|_\infty^q , & \\ \qquad \qquad \text{if } \x \in \B \( \x_{2^d, \frac{\epsilon_M}{3} }^*, \frac{\epsilon_M }{3} \) \backslash \B \( 0, \frac{\epsilon_M }{6} \) \\ 
            2 \| \x_{l,\epsilon_j}^* \|_\infty^q, & \\
            \qquad \qquad \text{if $\x$ is in other parts of } \\ \qquad \qquad \R^d \backslash \B \( 2^{\frac{1}{q}} \cdot \x_{l,  \epsilon_j}^* , 2^{\frac{1}{q}} \epsilon_j \). 
        \end{cases} \\ 
        \ge& \;  
        \epsilon_j^q \geq \frac{\epsilon_j^q}{3^q}.
    \end{align*} 

    \textbf{Case II: $j = M$ and $l < 2^d$.} Recall that the instance does not depend on $k$ in this case. 
    When $\x \notin S_l^{2^{\frac{1}{q}} \epsilon_j } 
    $, it holds that 
    \begin{align*} 
        & \; f_{M,k,l}^{\epsilon_M} (\x) - f_{M,k,l}^{\epsilon_M} (2^{\frac{1}{q}} \cdot \x_{l,\frac{\epsilon_M}{3}}^*) \\
        =&\;  
        f_{M,k,l}^{\epsilon_M} (\x) + \| 2^{\frac{1}{q}} \cdot \x_{l,\frac{\epsilon_M}{3} }^* \|_\infty^q \\ 
        \ge& \;  
        \min \left\{ 2 \| \x_{l,\frac{\epsilon_M}{3}}^* \|_\infty^q - \| \x_{2^d,\frac{\epsilon_M}{3} }^* \|_\infty^q , 
        2 \| \x_{l,\frac{\epsilon_M}{3}}^* \|_\infty^q \right\} \\
        \ge&\;  
        \frac{\epsilon_j^q}{3^q}.
    \end{align*} 

    \textbf{Case III: $j < M$, $l = 2^d$, (and $ k < 2^d$).} 
    For this case, when $\x \notin S_{2^d}^{2^{\frac{1}{q}} \epsilon_j }$, it holds that 
    \begin{align*} 
        &\; f_{j,k,2^d}^{\epsilon_j} (\x) - f_{j,k,2^d}^{\epsilon_j} (2^{\frac{1}{q}} \cdot \x_{2^d,\epsilon_j}^*)\\
        =&\;  
        f_{j,k,2^d}^{\epsilon_j} (\x) + \| 2^{\frac{1}{q}} \cdot \x_{2^d,\epsilon_j}^* \|_\infty^q \\ 
        \ge& \;  
        \min \left\{ 2 \| \x_{2^d,\epsilon_j}^* \|_\infty^q - \| \x_{k,\epsilon_j}^* \|_\infty^q , 2 \| \x_{2^d,\epsilon_j}^* \|_\infty^q \right\} \\
        \ge& \;  
        \epsilon_j^q \geq \frac{\epsilon_j^q}{3^q}.
    \end{align*} 

    There are some other cases. They are \textbf{Case IV: $j < M$, $l < 2^d$, and $ k = l$;} and \textbf{Case V: $j = M$, $l = 2^d$, (and $ k < 2^d$).} 
    The proof for Cases IV-V uses the same argument as that for the previous cases. Now we combine all cases to conclude the proof.

    
    
\end{proof}


\subsection{The information-theoretical argument}

First of all, we state below a classic result of Bretagnolle and Huber \citep{10.1007/BFb0064610}; See (e.g., in \citet{lattimore2020bandit}) for a modern reference. 

\begin{lemma}[Bretagnolle--Huber] 
    \label{lem:bh}
    For two distributions $P, Q$ over the same probability space, it holds that 
    \begin{align*} 
        D_{TV} (P,Q)
        \leq& \;  
        {\sqrt {1-e^{-D_{ {kl} }(P\parallel Q)}}} \\
        \le& \;  
        1 - \frac{1}{2} \exp \( -D_{ {kl} }(P\parallel Q) \) . 
    \end{align*} 
\end{lemma} 


The proof consists of two major steps. In the first step, we prove that for any policy $\pi$, there exists a long batch with high chance. In the second step, on the basis of existence of a long batch, we prove that there exists a bitten-apple instance (defined in Section \ref{sec:inst}) on which no policy performs better the lower bound in Theorem \ref{thm:lower}. Next we focus on proving the first step. 

For a policy $\pi$ that communicates at $ t_0 \le t_1 \le t_2 \le \cdots \le t_M $, we consider a set of events 
\begin{align}
    A_j:=\{t_{j-1} < T_{j-1} \text{ and } t_j \geq T_j \} , \label{eq:def-A} 
\end{align}
where $T_j$ is the reference communication point defined in (\ref{eq:def-eps-T}). Whenever the event $A_j$ is true, the $j$-th batch is large. Next we prove that some of $ A_j $ occurs under some instances, thus proving the existence of a long batch. Before proceeding, we introduce the following notation for simplicity.  



For any policy $\pi$, we define 
\begin{align} 
    p_j:= 
    \frac{1}{2^d-1} \sum_{k=1}^{2^d-1} \Pr_{j,k} ( A_j), 
    \qquad j=1,2,\cdots,M . \label{eq:def-low-p}
\end{align} 
where $\Pr_{j,k}(A_j)$  denotes the probability of the event $A_j$ under the instance $I_{j,k}$ and policy $\pi$. Next in Lemma \ref{lem:adaptive1}, we show that with constant chance, there is a long batch. 

\begin{lemma}
\label{lem:adaptive1}
    For any policy $\pi$ that adaptively determines the communications points, 
    it holds that $\sum_{j=1}^M p_j \geq \frac{7}{8}$, where $ p_j $ is defined in (\ref{eq:def-low-p}).  
\end{lemma}

\begin{proof}[Proof of Lemma \ref{lem:adaptive1}]
    Fix an arbitrary policy $\pi$. 
    For each $t$, let  $ \Pr_{j,k}^t $  (resp. $\Pr_{M,k}^t$) be the probability of $ (\x_t, y_t) $ governed by running $\pi$ in environment $f_{j,k}^{\epsilon_j}$ (resp. $f_{M,k}^{\epsilon_M}$), i.e. $\Pr_{j,k}^t = \Pr_{j,k}^t \left(\x_1 , y_1, \x_2, y_2, \cdots, \x_{t_{j-1}}, y_{t_{j-1}} \right).$ 
    The event $A_j$ is determined by the observations up to time $T_{j-1} $, since communication point $t_j$ is determined given the previous time grid $\{t_1,t_2,\cdots,t_{j-1}\}$ under a fixed policy $\pi$. 
    To further illustrate this fact, we first notice that the event $ A_j' := \{ t_{j-1} < T_{j-1} \} $ is fully determined by observations up to $T_{j-1}$. If $ t_{j-1} \ge T_{j-1} $, then the failure of $ A_j' $, thus the failure of $A_j$, is known by time $ T_{j-1} $. 
    If $ t_{j-1} < T_{j-1} $, then based on observations up to time $ t_{j-1} < T_{j-1} $, the policy $\pi$ determines $t_j$, thus $A_j$. In both cases, the success of $A_j$ is fully determined by observations up to time $T_{j-1}$. 
    It is also worth emphasizing that the policy $\pi$ does not communicate at $\{ T_j \}_{j \in [M]}$. We use $\{ T_j \}_{j \in [M]}$ only as a reference. 
    With the above argument, we get 
    \begin{align}
         |\Pr_{M,k} (A_j)-\Pr_{j,k}(A_j)|  
            = & \;  
            | \Pr_{M,k}^{T_{j-1}} (A_j) - \Pr_{j,k}^{T_{j-1}} (A_j) | \nonumber \\
    	\leq&\;   
            D_{TV} \(\Pr_{M,k}^{T_{j-1}},\Pr_{j,k}^{T_{j-1}} \) . \label{eq:pause-1}
    \end{align}
    By Lemma \ref{lem:bh},  
    \begin{align}
        &\; \frac{1}{2^d-1} \sum_{k=1}^{2^d-1} D_{TV} \( \Pr_{M,k}^{T_{j-1}},\Pr_{j,k}^{T_{j-1}}\) \nonumber \\
    	\leq&\; 
    	\frac{1}{2^d-1} \sum_{k=1}^{2^d-1} \sqrt{1-\exp\(-D_{kl}\(\Pr_{M,k}^{T_{j-1}}\|\Pr_{j,k}^{T_{j-1}}\)\)} \,. \label{eq:pause-2}
    \end{align}
    Note that $f_{j,k}^{\epsilon_j}$ differs from $f_{M,k}^{\epsilon_M}$ only in $\B (\x_{k,\epsilon_j}^*, \epsilon_j ) \backslash \B (0, \frac{\epsilon_j}{2})$.  Hence the chain rule for KL-divergence gives, for any $t \in [T_{j-1},T_j) $, 
    \begin{align} 
        & D_{kl} \( \Pr_{M,k}^t \| \Pr_{j,k}^t \) \nonumber \\ 
        =&  
        D_{kl} \( \Pr_{M,k}^t \( \mathbf{X}_{:T_{j-1}+1} \) \| \Pr_{j,k}^t \( \mathbf{X}_{:T_{j-1}+1}  \) \) \nonumber \\ 
        =&
        \E_{\Pr_{M,k}^t }  
        D_{kl} \( \mathcal{N} \( f_{M,k}^{\epsilon_M} (\x_{T_{j-1}}) \) \| \mathcal{N} \( f_{j,k}^{\epsilon_j} (\x_{T_{j-1}}) \) \)  \nonumber 
        \\ 
        &+ D_{kl} \( \Pr_{M,k}^t \( \x_{{T}_{j-1}} | \mathbf{X}_{:T_{j-1}} \) \| \Pr_{j,k}^t \( \x_{T_{j-1}}|  \mathbf{X}_{:T_{j-1}} \) \) 
        \nonumber
        \\
        &+
        D_{kl} \( \Pr_{M,k}^t \( \mathbf{X}_{:T_{j-1}} \) \| \Pr_{j,k}^t \( \mathbf{X}_{:T_{j-1}} \) \), 
        \label{eq:pause-3} 
    \end{align} 
    where the notation $ \mathbf{X}_{:t+1} := \{ \x_1 , y_1, \x_2 , y_2, \cdots, \x_t, y_t \} $ ($t \in \mathbb{Z}^{+}$) is introduced for simplicity, and $ \mathcal{N} \(\mu \) $ is the Gaussian random variable of mean $\mu$ and variance 1. 
    Under the fixed policy $\pi$, $ \x_{T_{j-1}} $ is fully determined by choices and observations before it. Thus 
    \begin{align*}
        D_{kl} \( \Pr_{M,k}^t \( \x_{{T}_{j-1}} | \mathbf{X}_{:T_{j-1}} \) \| \Pr_{j,k}^t \( \x_{{T}_{j-1}}| \mathbf{X}_{:T_{j-1}} \) \) 
        = 0. 
    \end{align*}
    By Proposition \ref{prop:check-ins-gap-1}, 
    \begin{align*}
        &\; D_{kl} \( \mathcal{N} \( f_{M,k}^{\epsilon_M} (\x_{T_{j-1}})  \) \| \mathcal{N} \( f_{j,k}^{\epsilon_j} (\x_{T_{j-1}})  \) \) \\ 
        =& \;  
        \frac{1}{2} \( f_{M,k}^{\epsilon_M} (\x_{{T}_{j-1}}) - f_{j,k}^{\epsilon_j} (\x_{{T}_{j-1}}) \)^2 \\ 
        \le& \;  
        \frac{(2^q + 2)^2}{2} \epsilon_j^{2q} \Ind_{ \{ \x_{{T_{j-1}}} \in S_k^{\epsilon_j} \} } . 
    \end{align*}
    
    We plug the above results into (\ref{eq:pause-3}) and get, for any $k \ge 2$,  
    \begin{align*} 
        &\; D_{kl} \( \Pr_{M,k}^t \| \Pr_{j,k }^t \) \\ 
        =& \; 
        D_{kl} \( \Pr_{M,k}^t \( \mathbf{X}_{:T_{j-1}} \) \| \Pr_{j,k}^t \( \mathbf{X}_{:T_{j-1}} \) \) \\ 
        &+ 
        \E_{\Pr_{M,k}^t } \[ \frac{1}{2} \( f_{M,k}^{\epsilon_M} (\x_{T_{j-1}}) - f_{j,k}^{\epsilon_j} (\x_{T_{j-1}}) \)^2 \] \\ 
        \le& \; 
        D_{kl} \( \Pr_{M,k}^t \( \mathbf{X}_{:T_{j-1}} \) \| \Pr_{j,k}^t \( \mathbf{X}_{:T_{j-1}} \) \) \\ 
        &+ 
        \frac{(2^q+2)^2 }{2} \E_{\Pr_{M,k}^t } \[  \epsilon_j^{2q} \Ind_{ \left\{ \x_{T_{j-1}} \in S_k^{\epsilon_j} \right\} } \] \\
        =& \; 
        D_{kl} \( \Pr_{M,k}^t \( \mathbf{X}_{:T_{j-1}} \) \| \Pr_{j,k}^t \( \mathbf{X}_{:T_{j-1}} \) \) \\ 
        &+ 
        \frac{ (2^q+2)^2 \epsilon_j^{2q} }{2} \Pr_{M,k}^t \( \x_{T_{j-1}} \in S_k^{\epsilon_j} \) . 
    \end{align*} 
    
    We can then recursively apply chain rule and the above calculation, and obtain 
    \begin{align*} 
         D_{kl} \( \Pr_{M,k}^t \| \Pr_{j,k}^t \)  
        \le  
        \frac{ (2^q+2)^2 \epsilon_j^{2q} }{2} \sum_{s\le T_{j-1}} \Pr_{M,k}^t \( \x_s \in S_k^{\epsilon_j} \) 
    \end{align*} 
    for each $ t: T_{j-1} \le t < T_j$. Therefore, we have 
    \begin{align}
         &\; D_{kl} \( \Pr_{M,k}^{T_{j-1}} \| \Pr_{j,k}^{T_{j-1}} \) \nonumber \\
        \le &\; 
        \frac{ (2^q+2)^2 \epsilon_j^{2q} }{2} \sum_{s\le T_{j-1}} \Pr_{M,k}^{T_{j-1}} \( \x_s \in S_k^{\epsilon_j} \) . \label{eq:pause-4}
    \end{align}
        

    
    Combining the above inequalities (\ref{eq:pause-2}) and (\ref{eq:pause-4}) and Jensen's inequality yields that 
    \begin{align} 
        &\; \frac{1}{2^d-1} \sum_{k=1}^{2^d-1} D_{TV} \( \Pr_{M,k}^{T_{j-1}},\Pr_{j,k}^{T_{j-1}}\) \nonumber \\
        \leq & \;
        \frac{1}{2^d-1} \sum_{k=1}^{2^d-1} \sqrt{1-\exp\(-D_{kl}\(\Pr_{M,k}^{T_{j-1}}\|\Pr_{j,k}^{T_{j-1}}\)\)} \nonumber \\ 
        \leq& 
        \sqrt{1-\exp \(-\frac{ \epsilon_{j}^{2q} }{ C_{d,q} }\sum_{k=1}^{2^d-1}\sum_{s=1}^{ T_{j-1}} \Pr_{M,k}^{T_{j-1}} ( \x_s \in S_k^{\epsilon_j} ) \) } , 
        \label{eq:concave_1}
    \end{align}
    where $ C_{d,q} := \frac{2(2^d-1)}{(2^q + 2)^2} $ is introduced to avoid clutter, and the last inequality follows from Jensen. Since $ \sum_{k=1}^{2^d-1} \Pr_{M,k}^{T_j-1} \( \x_s \in S_k^{\epsilon_j} \) \le 1 $ ($ S_k^{\epsilon_j} $ are disjoint), we continue from (\ref{eq:concave_1}) and get 
    \begin{align}
        (\ref{eq:concave_1})
        \leq& \;
            \sqrt{1-\exp\(-\frac{ (2^q+2)^2 \epsilon_{j}^{2q} T_{j-1} }{2(2^d-1)} \)} \nonumber \\ 
         \overset{(i)}{\leq}&\;  
         \sqrt{1-\exp\( - \frac{1}{64} \cdot \frac{1}{M^2}\)} 
         \overset{(ii)}{\leq}
         \frac{1}{8} \cdot \frac{1}{M}, \label{eq:D_TV_leq_1/B}
        \end{align}
    where \emph{(i)} uses definitions of $\epsilon_j$ and $T_j$ (\ref{eq:def-eps-T}), \emph{(ii)} uses a basic property of the exponential function: $\exp (-x) \ge 1 - x$ for each $x \in \R$. 
    Combining (\ref{eq:pause-1}) and (\ref{eq:D_TV_leq_1/B}) gives that, for each $j = 1,2,\cdots,M$,  
    \begin{align*}
        &\; |\Pr_{M,k}(A_j) - p_j| \\
        \le&\; 
        \frac{1}{2^d -1 } \sum_{k=1}^{2^d-1} |\Pr_{M,k} (A_j)-\Pr_{j,k}(A_j)| 
        \le  \frac{1}{8M},
    \end{align*}
    and thus  
    \begin{align*} 
         \sum_{j=1}^M p_j 
        \geq 
        \sum_{j=1}^M \Pr_{M,k}(A_j)-\frac{1}{8} 
        \geq  \Pr_{M,k}(\cup_{j=1}^M A_j) -\frac{1}{8} \geq \frac{7}{8}, 
    \end{align*}
    where the last inequality holds since at least one of $\{A_1,A_2,\cdots,A_M\}$ must be true.  

\end{proof} 

Now that Lemma \ref{lem:adaptive1} is in place, we can prove the existence of a bad bitten-apple instance, which concludes the proof of Theorem \ref{thm:lower}.


\begin{proof}[Proof of Theorem \ref{thm:lower}]

Fix any policy $\pi$. Let $ \Pr_{j,k,l} $ be the probability of running $\pi$ on $f_{j,k,l}^{\epsilon_j}$. Let $ \Pr_{j,k,l}^t $ be the probability of $ (\x_1, y_1, \x_2, y_2, \cdots , \x_t, y_t) $ governed by running $\pi$ in environment $f_{j,k,l}^{\epsilon_j}$, Proposition \ref{prop:check-ins-gap-2} gives that
    \begin{align}
        &\; \sup_{I \in \{I_{j,k,l}\} _{j \in [M], k<2^d, l \in [2^d]}}  \E \[ R^\pi (T) \] \nonumber \\
         \geq&\;  
         \frac{1}{M} \sum_{j=1}^M \frac{\epsilon_j^q}{3^q} \sum_{t=1}^{T} \frac{1}{2^d-1}  \cdot\frac{1}{2^d} \sum_{k=1}^{2^d-1} \sum_{l =1 }^{2^d} \Pr_{j,k,l}^t \(\x_t \notin S_l^{2^{\frac{1}{q}} \cdot \epsilon_j}\) \nonumber
         \\
         =&\; \frac{1}{3^q}\cdot\frac{1}{M} \sum_{j=1}^{M} \epsilon_j^q \sum_{t=1}^{T} \frac{1}{2^d-1}  \nonumber \\ 
         &\;\cdot \frac{1}{2^d} \sum_{k=1}^{2^d-1} \sum_{l =1 }^{2^d} 
         \( 1- \Pr_{j,k,l}^t \(\x_t \in S_l^{2^\frac{1}{q} \cdot \epsilon_j}\) \)  
         \label{eq:sup-reg}.
    \end{align}
    By definition of total-variation distance, we have 
    \begin{align*}
        &\; D_{TV} \(\Pr_{j,k,l}^{t},\Pr_{j,k,k}^{t}\) \\
        \ge&\;  \Pr_{j,k,l}^t \(\x_t \in S_l^{2^\frac{1}{q} \cdot \epsilon_j}\)
    - \Pr_{j,k,k}^t \(\x_t \in S_l^{2^\frac{1}{q} \cdot \epsilon_j}\)  
    \end{align*}
    and
    \begin{align*}
    &\; D_{TV} \(\Pr_{M,k,l}^{t},\Pr_{M,k,2^d}^{t}\) \\
    \ge& \;  
    \Pr_{M,k,l}^t \(\x_t \in S_l^{2^\frac{1}{q} \cdot \epsilon_M}\)
    - \Pr_{M,k,2^d}^t \(\x_t \in S_l^{2^\frac{1}{q} \cdot \epsilon_M}\).  
    \end{align*}
    The above inequalities, together with the fact that $\sum_{l=1}^{2^d} \Pr_{j,k,k}^t \(\x_t \in S_l^{2^\frac{1}{q} \cdot \epsilon_j}\) \leq 1 $, yield  
    \begin{align}
        &\; \sum_{l=1}^{2^d} \Pr_{j,k,l}^t \(\x_t \in S_l^{2^\frac{1}{q} \cdot \epsilon_j}\) \nonumber \\
        \le&
        \begin{cases}
            1 + \sum\limits_{l=1}^{2^d} D_{TV} (\Pr_{j,k,l}^{t},\Pr_{j,k,k}^{t}) , 
            & \text{ if } j < M, \\
            1 + \sum\limits_{l=1}^{2^d} D_{TV} (\Pr_{M,k,l}^{t},\Pr_{M,k,2^d}^{t}) , 
            & \text{ if } j = M
        \end{cases} \label{eq:2-cases}
    \end{align} 
    
    Next we study the first case in (\ref{eq:2-cases}). 
    Incorporating the equation $D_{TV}(\Pr,\mathbb{Q})=\frac{1}{2} \int | d\Pr-d\mathbb{Q}|$, we obtain
    \begin{align}
         \;& \sum_{t=1}^{T}  \sum_{l =1 }^{2^d} 
         \bigg( 1- \Pr_{j,k,l}^t \(\x_t \in S_l^{2^\frac{1}{q} \cdot \epsilon_j}\)\bigg) \nonumber
         \\
         \ge &\;  \sum_{t=1}^{T}  \[\sum_{l =1 }^{2^d} 
         \bigg( 1   
         - D_{TV} \(\Pr_{j,k,l}^{t},\Pr_{j,k,k}^{t}\) \bigg) -1\]
         \nonumber \\ 
         \ge &\;
         \sum_{t=1}^T  \sum_{l\neq k}\( 1-
        \frac{1}{2} \int \left| d \Pr_{j,k,k}^t-d \Pr_{j,k,l}^t\right| \)
        \nonumber
        \\
        \geq &\;  
         \sum_{t=1}^{T_j} \sum_{l\neq k}\( 1-
        \frac{1}{2} \int \left| d \Pr_{j,k,k}^t-d \Pr_{j,k,l}^t\right| \) \nonumber
        \\
        \geq &\;
        T_j \sum_{l\neq k}\( 1-
        \frac{1}{2} \int \left| d \Pr_{j,k,k}^{T_j}-d \Pr_{j,k,l}^{T_j}\right| \) \label{eq:data-processing}
       \\
       = &\;
        \frac{T_j}{2}\sum_{l\neq k} \int d \Pr_{j,k,k}^{T_{j-1}}+ d \Pr_{j,k,l}^{T_{j-1}} 
       - \left| d \Pr_{j,k,k}^{T_{j-1}}-d \Pr_{j,k,l}^{T_{j-1}} \right| \nonumber
       \\
       \geq &\; 
       \frac{T_j}{2}\sum_{l\neq k} \int_{A_j} d \Pr_{j,k,k}^{T_{j-1}}+ d \Pr_{j,k,l}^{T_{j-1}} 
       - \left| d \Pr_{j,k,k}^{T_{j-1}}-d \Pr_{j,k,l}^{T_{j-1}} \right| , 
       \label{eq:under_A_j} 
    \end{align}
where (\ref{eq:data-processing}) follows from data processing inequality of total variation distance, and the last equation (\ref{eq:under_A_j}) holds because the observations at time $T_j$ are the same as those at time $T_{j-1}$ under event $A_j$ and the fixed policy $\pi$. Furthermore, we have
\begin{align}
    &\frac{1}{2} \( \int_{A_j} d \Pr_{j,k,k}^{T_{j-1}}+ d \Pr_{j,k,l}^{T_{j-1}} - \left| d \Pr_{j,k,k}^{T_{j-1}}-d \Pr_{j,k,l}^{T_{j-1}}\right| \) \nonumber  \\
        =&\; 
        \frac{\Pr_{j,k,k}^{T_{j-1}}(A_j) \hspace*{-2pt} + \hspace*{-2pt} \Pr_{j,k,l}^{T_{j-1}}(A_j)}{2}-\frac{\int_{A_j}\left|d\Pr_{j,k,k}^{T_{j-1}}-d\Pr_{j,k,l}^{T_{j-1}}\right|}{2}\nonumber\\
        \geq&\;  
        \(\Pr_{j,k,k}^{T_{j-1}}(A_j)-\frac{1}{2}D_{TV} \(\Pr_{j,k,k}^{T_{j-1}},\Pr_{j,k,l}^{T_{j-1}}\)\) \nonumber \\
        &-D_{TV} \(\Pr_{j,k,k}^{T_{j-1}},\Pr_{j,k,l}^{T_{j-1}}\)\label{eq:P-Q_leq_TV}\\ 
        =&\; 
        \Pr_{j,k}(A_j)-\frac{3}{2}D_{TV} \( \Pr_{j,k,k}^{T_{j-1}},\Pr_{j,k,l}^{T_{j-1}} \), \label{eq:def_of_A_j} 
\end{align}
where (\ref{eq:P-Q_leq_TV}) follows from $|\Pr(A)-\mathbb{Q}(A)|\leq D_{TV}(\Pr,\mathbb{Q})$, and (\ref{eq:def_of_A_j}) is attributed to the fact that $A_j$ is determined by the observations up to time $T_{j-1}$.\\
Similar to the argument for (\ref{eq:pause-4})-(\ref{eq:D_TV_leq_1/B}), we have, for each fixed $k$
\begin{align}
    \frac{1}{2^d} \sum_{l\neq k} D_{TV} \( \Pr_{j,k,k}^{T_{j-1}},\Pr_{j,k,l}^{T_{j-1}}\) 
     \leq & \;
      \frac{1}{4M}\cdot \frac{2^d-1}{2^d}. \label{eq:comb-2}
\end{align}
For the second case in (\ref{eq:2-cases}), we have the same inequality by subtituting $\Pr_{j,k,l}^t$ (resp. $\Pr_{j,k,k}^t$) with $\Pr_{M,k,l}^t$ (resp. $\Pr_{M,k,2^d}^t$). 

Combining (\ref{eq:sup-reg}), (\ref{eq:2-cases}), (\ref{eq:under_A_j}), (\ref{eq:def_of_A_j}) and (\ref{eq:comb-2}), we have
\begin{align*}
    & \sup_{I \in \{I_{j,k,l}\} _{j \in [M], k < 2^d, l \in [2^d]}}   \E \[ R^\pi (T) \] \\ 
    \geq& 
    \frac{2^d-1}{ M 3^q 2^d}  \sum_{j=1}^{M-1} \epsilon_j^q T_j \Bigg[ \frac{1}{2^d-1} \sum_{k=1}^{2^d-1} \Pr_{j,k}\( A_j\)
    -\frac{3}{8}  \cdot \frac{1}{M}\Bigg] \\
    &+ \frac{2^d-1}{M 3^q 2^d} \epsilon_M^q T_{M} \Bigg[ \frac{\sum_{k=1}^{2^d-1} \Pr_{M,k}(A_M)}{2^d-1}-\frac{3}{8}  \cdot \frac{1}{M} \Bigg]
    \\
    =&  
    \frac{1}{3^q} \cdot \frac{1}{M} \cdot \frac{2^d-1}{2^d}\sum_{j=1}^M \epsilon_j^q T_j \( p_j -\frac{3}{8} \cdot \frac{1}{M}\). 
\end{align*} 

By definition of $\epsilon_j$ and $T_j$ in (\ref{eq:def-eps-T}), we have $ \epsilon_j^q T_j = \frac{\sqrt{2}}{8} \cdot \frac{\sqrt{2^d-1}}{2^q + 2} \cdot \frac{1}{M} \cdot T^{\frac{1}{2} \cdot \frac{1}{1 - 2^{-M}}} $ for all $j \in [M]$. Therefore, we continue from the above inequalities and get  
\begin{align*} 
    &\; \sup_{I \in \{I_{j,k,l}\} _{j \in [B], k < 2^d, l \in [2^d]}  } \E \[ R^\pi (T) \] \\ 
    \ge& \; 
     \frac{\sqrt{2} (2^d-1)^\frac{3}{2}}{8\cdot 3^q M^2(2^q+2)2^d}  \cdot T^{\frac{1}{2} \cdot \frac{1}{1 - 2^{-M}}} \( \sum_{j=1}^M  p_j -\frac{3}{8} \) \\
    \ge& \;  
      \frac{\sqrt{2}}{16} \cdot \frac{1}{M^2} \cdot \frac{1}{3^q(2^q + 2)} \cdot \frac{(2^d-1)^\frac{3}{2}}{2^d} \cdot T^{\frac{1}{2} \cdot \frac{1}{1 - 2^{-M}}}, 
\end{align*} 
where the last inequality uses Lemma \ref{lem:adaptive1}. 

\end{proof}

\begin{proof}[Proof of Corollary \ref{cor}.]

    From Theorem \ref{thm:lower}, the expected regret is lower bounded by 
    \begin{align*} 
	    \E \left[ R_T(\pi) \right] 
	    \geq 
        \frac{\sqrt{2}}{16} \cdot \frac{1}{M^2} \cdot \frac{1}{3^q(2^q + 2)} \cdot \frac{(2^d-1)^\frac{3}{2}}{2^d} T^{\frac{1}{2} \cdot \frac{1}{1 - 2^{-M}}} . 
    \end{align*}
    Here we seek for the minimum $M$ such that 
    \begin{align}\label{lower_b_1} 
        \frac{ \frac{1}{M^2} \cdot T^{\frac{1}{2} \cdot \frac{1}{1 - 2^{-M}}} }{ \sqrt{T } } 
        \leq e . 
    \end{align} 
    
    Calculation shows that
    \begin{align}\label{lower_b_2}
        \frac{ \frac{1}{M^2} \cdot T^{\frac{1}{2} \cdot \frac{1}{1 - 2^{-M}}} }{ \sqrt{T} } =
        \frac{1}{ M^2 } T^{ \frac{1}{2} \cdot \frac{1}{2^M - 1} }.
    \end{align}
    
    Substituting (\ref{lower_b_2}) to (\ref{lower_b_1}) and taking log on both sides yield that 
    \begin{align*} 
        \frac{1}{2} \cdot \frac{1}{2^M - 1} \log T 
        \leq 
        \log (M^2 e ) 
    \end{align*} 
    and thus 
    \begin{align} 
        M
        \ge 
        \log_2 \( 1 + \frac{\log T}{2 \log (M^2 e )} \) . \label{lower-coro} 
    \end{align} 
    We use $M_{\min}$ to denote the minimum $ M $ such that inequality (\ref{lower-coro}) holds. Calculation shows that (\ref{lower-coro}) holds for 
    \begin{equation*} 
        M_*:= \log_2 \( 1 +  \frac{ \log T }{ 2 } \) , 
    \end{equation*} 
    so we have $M_{\min}\leq M_*$. Then since the RHS of (\ref{lower-coro}) decreases with $M$, we have 
    \begin{align*} 
        M_{\min} 
        \geq&\;  
        \log_2 \( 1 + \frac{\log T}{2 \log ( M_{\min}^2 e )} \) \\
        \geq&\;  
        \log_2 \( 1 + \frac{ \log T }{ 2 \log ( M_*^2 e ) } \) .  
    \end{align*} 
    Therefore, $\Omega(\log\log T)$ rounds of communications are necessary for any algorithm to achieve a regret rate of order $ K_- A_-^d \sqrt{T} $, where $K_-$ depends only on $q$ and $A_-$ is an absolute constant. 

\end{proof}


\subsection{Lower bound for nondegenerate bandits without communication constraints} 

Having established the lower bound with communication constraints in the previous section, it is worth noting that the existing literature lacks a standard lower bound result specifically tailored for nondegenerate bandits. To this end, we proceed to fill this gap by presenting a lower bound that does not incorporate any communication constraints.

To prove this result, we need a different set of problem instances, which we introduce now. 
For any fixed $\epsilon$, we partition the space $\R^d$ again into $2^d$ disjoint parts $U_1^\epsilon, U_2^\epsilon,\cdots, U_{2^d}^\epsilon$. 
For $k = 1$, we define $ U_1^\epsilon = O_1 \cup \B (0, \frac{\epsilon}{2}) $. For $k = 2,\cdots, 2^d$ , we define $ U_k^\epsilon = O_k \backslash \B \( 0, \frac{\epsilon}{2} \) $. 

For any $k = 2,\cdots, 2^d$, and $\epsilon > 0$, define  
\begin{align} 
    f_k^\epsilon (\x) = 
    \begin{cases} 
        \| \x - \x_{k,\epsilon}^* \|_{\infty}^q - \| \x_{k,\epsilon}^* \|_{\infty}^q, & \\ 
        \qquad \qquad \quad \text{if } \x \in \B (\x_{k,\epsilon}^* , \epsilon) \backslash \B (0, \frac{\epsilon}{2}) , \\ 
        \| \x \|_{\infty}^q, \qquad \text{otherwise. } 
    \end{cases} \label{eq:def-fk}
\end{align} 

In addition, we define the function $f_1^\epsilon$ as 
\begin{align} 
    f_1^\epsilon (\x) = \| \x \|_\infty^q , \label{eq:def-f1}
\end{align} 
and slightly overload the notations to define $\x_{1,\epsilon}^* := 0$. 
Note that $ f_1^\epsilon (\x) $ and $ \x_{1,\epsilon}^* $ do not depend on $\epsilon$. We keep the $\epsilon$ superscript for notational consistency. 



Firstly, we observe that instances specified by $ \{ f_k^\epsilon \}_{k \in [2^d]} $ satisfy the properties stated in Proposition \ref{prop:lb-instance}.

\begin{proposition}  
    \label{prop:lb-instance} 
    The functions $f_k^\epsilon$ satisfies 
    \begin{enumerate} 
        \item For each $k=1,2,\cdots, 2^d$,  $ \frac{1}{2^{q-1}} \| \x - \x_{k,\epsilon}^* \|_\infty^q \le f_k^\epsilon (\x) - f_k^\epsilon (\x_{k,\epsilon}^*) $, for all $\x \in \R^d$. 
        \item For each $k=2,3,\cdots, 2^d$, 
        \begin{align*} 
            \begin{cases} 
                | f_{k}^\epsilon (\x) - f_{1}^\epsilon (\x) | \le (2^q+2) \epsilon^q , & \forall \x \in U_k^\epsilon, \\ 
                | f_{k}^\epsilon (\x) - f_1^\epsilon (\x) | = 0  , & \forall \x \notin U_k^\epsilon . 
            \end{cases} 
        \end{align*} 
        \item For each $k=1,2,\cdots, 2^d$, $  f_k^\epsilon (\x) - f_k^\epsilon (\x_{k,\epsilon}^*) \le 3^{q+1} \| \x - \x_{k,\epsilon}^* \|_\infty^q $, for all $\x \in \R^d$. 
    \end{enumerate} 
\end{proposition} 

\begin{proof} 
    Item 1 is clearly true when $\x \in U_k^\epsilon$, it remains to consider $\x \notin U_k^\epsilon $. 
    For item 1, we use Jensen's inequality to get 
    \begin{align*}
        \left\| \frac{ \x - \y }{2} \right\|_{\infty}^q \le \frac{ \| \x \|_{\infty}^q + \| \y \|_{\infty}^q }{2}, \; \forall q \ge 1 , \forall \x,\y \in \R^d . 
    \end{align*}
    Rearranging terms, and substituting $\y = \x_{k,\epsilon}^* $ in the above inequality gives that, for any $\x \notin U_k^\epsilon$,  
    \begin{align*} 
         \frac{1}{2^{q-1}} \| \x - \x_{k,\epsilon}^* \|_\infty^q 
        \le \;  
        \| \x \|_\infty^q + \| \x_{k,\epsilon}^* \|^q 
        = \; f (\x) - f (\x_{k,\epsilon}^*) . 
    \end{align*} 
    For item 2, we have, for each $k$ and $\x \in \B (\x_{k,\epsilon}^*, \epsilon ) \backslash \B (0, \frac{\epsilon}{2})$, 
    \begin{align*} 
         | f_{k}^\epsilon (\x) - f_1^\epsilon (\x) | 
        =& 
        \left| \left\| \x - \x^* \right\|_\infty^q - \| \x^* \|_\infty^q - \| \x \|_\infty^q \right|  \\ 
        \le& \epsilon^q + \epsilon^q + (2\epsilon)^q = (2^q+2) \epsilon^q
    \end{align*} 
    where the last inequality uses $ \| \x \|_\infty \le 2 \epsilon $ for all $ \x \in \B (\x_{k,\epsilon}^*, \epsilon ) $. 

    Next we proof item 3. Fix any $r \in ( \frac{\epsilon}{2}, \infty) $. For any $\x \in \S ( \x_{k,\epsilon}^*, r ) $, we have $ \| \x \|_\infty \le r + \epsilon $, and thus 
    \begin{align*} 
        3^q \| \x - \x_{k,\epsilon}^* \|_\infty^q 
        = 
        3^q r^q 
        \ge 
        \( r + \epsilon \)^q \ge \| \x \|_\infty^q . 
    \end{align*} 
    The above inequality gives, $ \forall \x \notin \B (\x_{k,\epsilon}^* , \epsilon) \backslash \B(0,\frac{\epsilon}{2}) $
    \begin{align*}
        &\; 3^{q+1} \| \x - \x_{k,\epsilon}^* \|_\infty^q 
        \ge 
        (2^q+3^q) \| \x - \x_{k,\epsilon}^* \|_\infty^q \\
        \ge&\;  \| \x \|_\infty^q + \| \x_{k,\epsilon}^* \|_\infty^q = f (\x) - f (\x^*) . 
    \end{align*}
    We conclude the proof by noticing that item 3 is clearly true when $ \x \in \B (\x_{k,\epsilon}^*, \epsilon) \backslash \B(0,\frac{\epsilon}{2})$. 
    
\end{proof} 



\begin{proof}[Proof of Theorem \ref{thm:standard}] 
    
    Fix any policy $\pi$. Let $ \Pr_{k,\epsilon} $ be the probability of running $\pi$ on $f_{k}^\epsilon$. 
    Let $ \E_{k,\epsilon} $ be the expectation with respect to $ \Pr_{k,\epsilon} $. 
    
    Firstly, we note that $ \{ \x_t \notin U_k^\epsilon \} \implies \{ f_k^{\epsilon} (\x_t) - f_k^\epsilon (\x_k^*) \ge 2^{-2q+1} \epsilon^q \} $. 
    Thus we have 
    \begin{align} 
        & \frac{1}{2^d} \sum_{k=1}^{2^d} \E_{k,\epsilon} \[ R_T(\pi) \] \nonumber \\
        \ge& \; 
        \frac{1}{2^d} \sum_{k=1}^{2^d} \sum_{t=1}^T \E_{k,\epsilon}^t \[ f_{k}^\epsilon (\x_t) - f_k^\epsilon (\x_{k,\epsilon}^*) \] \nonumber \\ 
        \ge& \; 
        \frac{2^{-2q+1} \epsilon^q }{2^d} \sum_{k=1}^{2^d} \sum_{t=1}^T \Pr_{k,\epsilon} \( f_{k}^\epsilon (\x_t) - f_k^\epsilon (\x_{k,\epsilon}^*) \ge  \frac{\epsilon^q}{2^{2q-1}} \) \nonumber \\ 
        \ge& \; 
        \frac{2^{-2q+1} \epsilon^q }{2^d} \sum_{k=1}^{2^d} \sum_{t=1}^T \Pr_{k,\epsilon} \( \x_t \notin U_k^\epsilon \) . 
        \label{eq:tmp-1} 
    \end{align} 
    
    Building on the previous derivation, we now turn our attention to the summation term and arrive at 
    \begin{align} 
        &\; \sum_{k=1}^{2^d} \sum_{t=1}^T \Pr_{k,\epsilon} \( \x_t \notin U_k^\epsilon \) \nonumber \\ 
        \ge& \; 
        \sum_{k=1}^{2^d} \sum_{t=1}^T \( 1 - \Pr_{k, \epsilon} \( \x_t \in U_k^\epsilon \) \) \nonumber \\ 
        \ge& \; 
        \sum_{k=1}^{2^d} \sum_{t=1}^T \( 1 - \Pr_{1,\epsilon} \( \x_t \in U_k^\epsilon \) 
        - D_{TV} \( \Pr_{1,\epsilon}, \Pr_{k,\epsilon} \) \) \nonumber \\ 
        =& \; 
        (2^d-1)T
        - 
        \sum_{k=2}^{2^d} \sum_{t=1}^T  D_{TV} \( \Pr_{1,\epsilon}, \Pr_{k,\epsilon} \) \nonumber \\
        \ge& \; 
        (2^d-1)T
        - 
        \sum_{k=2}^{2^d} \sum_{t=1}^T \( 1 
        - \frac{\exp \( - D_{kl} \( \Pr_{1,\epsilon} \| \Pr_{k,\epsilon} \) \)}{2}  \) \nonumber \\
        =& \; 
        \frac{1}{2}
        \sum_{k=2}^{2^d} \sum_{t=1}^T \exp \( - D_{kl} \( \Pr_{1,\epsilon} \| \Pr_{k,\epsilon} \) \) \nonumber \\ 
        \ge&\; 
        \frac{2^d-1}{2}
        \sum_{t=1}^T  \exp \( - \frac{1}{2^d - 1} 
        \sum_{k=2}^{2^d} D_{kl} \( \Pr_{1,\epsilon} \| \Pr_{k,\epsilon} \) \) , 
        \label{eq:tmp-2}
    \end{align} 
    where the fourth line uses $ \sum_{k=1}^{2^d} \Pr_{1,\epsilon} \( \x_t \in U_k^\epsilon \) = 1 $, the fifth line uses Lemma \ref{lem:bh}, and the last line uses Jensen's inequality. 
    
    By the chain rule of KL-divergence, we have 
    \begin{align} 
        &\; D_{kl} \( \Pr_{1,\epsilon} \| \Pr_{k,\epsilon} \) \nonumber \\ 
        =& \; 
        D_{kl} ( \Pr_{1,\epsilon} \( \mathbf{X}_{: T+1} \) 
        \| \Pr_{k,\epsilon} \( \mathbf{X}_{: T+1} \) ) \nonumber \\ 
        =& \; 
        D_{kl} ( \Pr_{1,\epsilon} \( \mathbf{X}_{: T} \) 
        \| \Pr_{k,\epsilon} \( \mathbf{X}_{: T} \) ) \nonumber \\  
        &+ 
        \E_{\Pr_{1,\epsilon} } \[ D_{kl} \( \mathcal{N} \( f_{1}^\epsilon (\x_{T})  \) \| \mathcal{N} \( f_{k}^\epsilon (\x_{T} )  \) \) \] \nonumber \\ 
        & + D_{kl} ( \Pr_{1,\epsilon} \( \x_{T } | \mathbf{X}_{: T} \) 
        \| \Pr_{k,\epsilon} \( \x_{T }| \mathbf{X}_{: T} \) ) 
        \label{eq:tmp-3} 
    \end{align} 
    where $ \mathbf{X}_{: t+1} = \{\x_1, y_1, \cdots, \x_t, y_t\} $, and $ \mathcal{N} \(\mu \) $ is the Gaussian random variable of mean $\mu$ and variance 1. Under the fixed policy $\pi$, $ \x_{T } $ is fully determined by choices and observations before it. Thus 
    \begin{align*}
        &\;D_{kl} ( \Pr_{1,\epsilon} \( \x_{T } | \mathbf{X}_{: T} \) 
        \| \Pr_{k,\epsilon} \( \x_{T }| \mathbf{X}_{: T} \) ) 
        = 0. 
    \end{align*} 
    Also, it holds that $ D_{kl} \( \mathcal{N} \( f_1^\epsilon (\x_{T })  \) \| \mathcal{N} \( f_{k}^\epsilon (\x_{T }) \) \) = \frac{1}{2} \( f_1^\epsilon (\x_{T }) - f_k^\epsilon (\x_{T }) \)^2 $. We plug the above results into (\ref{eq:tmp-3}) and get, for any $k \ge 2$,  
    \begin{align*} 
        &\;D_{kl} \( \Pr_{1,\epsilon} \| \Pr_{k,\epsilon} \)
        \\
        =& \; 
        D_{kl} \( \Pr_{1,\epsilon} \( \mathbf{X}_{: T} \) 
        \| \Pr_{k,\epsilon} \( \mathbf{X}_{: T} \) \) \\ 
        &+ 
        \E_{\Pr_{1,\epsilon} } \[ \frac{1}{2} \( f_1^\epsilon (\x_{T }) - f_k^\epsilon (\x_{T }) \)^2 \] \\ 
        \le& \; 
        D_{kl} \( \Pr_{1,\epsilon} \( \mathbf{X}_{: T} \) \| \Pr_{k,\epsilon} \( \mathbf{X}_{: T} \) \) \\ 
        &+ 
        \frac{(2^q+2)^2 }{2} \E_{\Pr_{1,\epsilon} } \[  \epsilon^{2q} \Ind_{ \left\{ \x_{T } \in U_k^\epsilon \right\} } \] \\
        =& \; 
        D_{kl} \( \Pr_{1,\epsilon} \( \mathbf{X}_{: T} \) \| \Pr_{k,\epsilon} \( \mathbf{X}_{: T} \) \) \\ 
        &+ 
        \frac{ (2^q+2)^2 \epsilon^{2q} }{2} \Pr_{1,\epsilon} \( \x_{T } \in U_k^\epsilon \) . 
    \end{align*} 
    
    We can then recursively apply chain rule and the above calculation, and obtain 
    \begin{align} 
        D_{kl} \( \Pr_{1,\epsilon} \| \Pr_{k,\epsilon} \)  
        \le 
        \frac{ (2^q+2)^2 \epsilon^{2q} }{2} \sum_{s = 1}^T \Pr_{1,\epsilon} \( \x_s \in U_k^\epsilon \) . \nonumber 
    \end{align} 
    
    
    
    Combining the above inequality with (\ref{eq:tmp-1}) and (\ref{eq:tmp-2}) gives 
    \begin{align} 
        & \; \frac{1}{2^d} \sum_{k=1}^{2^d} \E_{k,\epsilon} \[ R_T(\pi) \] \nonumber \\ 
        \ge& \;  
         \frac{1}{2^{2q+1}} \sum_{t=1}^T \epsilon^q \exp \( - \frac{1}{2^d - 1} \sum_{k=2}^{2^d} D_{kl} \( \Pr_{1,\epsilon} \| \Pr_{k,\epsilon} \) \) \nonumber \\ 
        \ge& \; 
         \frac{1}{2^{2q+1}} \sum_{t=1}^T \epsilon^q \exp \( - \frac{ \epsilon^{2q}}{C_{d,q}} \sum_{k=2}^{2^d}  \sum_{s =1 }^T \Pr_{1,\epsilon} \( \x_s \in U_k^\epsilon \) \) \nonumber \\ 
        \ge& \; 
         \frac{1}{2^{2q+1}} \sum_{t=1}^T \epsilon^q \exp \( - \frac{1}{2^d - 1} \cdot \frac{ (2^q+2)^2 \epsilon^{2q} }{2} T \) , \nonumber  
    \end{align}
    where $ C_{d,q} = \frac{2(2^d-1)}{(2^q + 2)^2} $, and the last line uses $ \sum_{k=2}^{2^d} \Pr_{1,\epsilon} \( \x_s \in U_k^\epsilon \) \le 1 $, since $U_k^\epsilon$ are disjoint. 

    By picking $ \epsilon^q = \frac{ \sqrt{ 2(2^d - 1) } }{ 2^q + 2 } \cdot \sqrt{ \frac{1}{T} } $, we have 
    \begin{align*}
        \frac{1}{2^d} \sum_{k=1}^{2^d} \E_{k,\epsilon} \[ R_T(\pi) \]  
        \ge&  \; 
        \frac{\sqrt{2(2^d-1)}}{ (2^q + 2) 2^{2q+1} } e^{-1} \sqrt{T} . 
    \end{align*}
    
    

\end{proof}

\section{Conclusion}

This paper studies the nondegenerate bandit problem with communication constraints. 
The nondegenerate bandit problem is important in that it encapsulates important problem classes, ranging from dynamic pricing to Riemannian optimization. 
We introduce the Geometric Narrowing (GN) algorithm that solves such problems in a near-optimal way. 
We establish that, when compared to GN, there is little room for improvement in terms of regret order or communication complexity. 


\appendix

%

\section{The Geometric Narrowing Algorithm for Level-Smooth Functions} 
\label{appe:L-S} 


When the space is Euclidean with the Lebesgue measure $\mu_0$, the Geometric Narrowing algorithm admits a variation that efficiently solves stochastic optimization problems under mild conditions. In this part, we introduce a slight variant of the GN algorithm. This algorithm solves the stochastic optimization problem proposed by \citet{wang2018optimization}, under conditions milder than those used by \citet{wang2018optimization}. 

\begin{remark} 
    It is important to note that the original GN algorithm, as outlined in Section \ref{sec:algo}, does not depend on the existence of a measure or the well-definedness of partial differentiability. In the Appendix, we focus on scenarios where such notions are well-defined, and explore stochastic optimization problems in these contexts. 
\end{remark} 




\begin{figure*}[h!] 
        \centering 
        \includegraphics[scale=0.38]{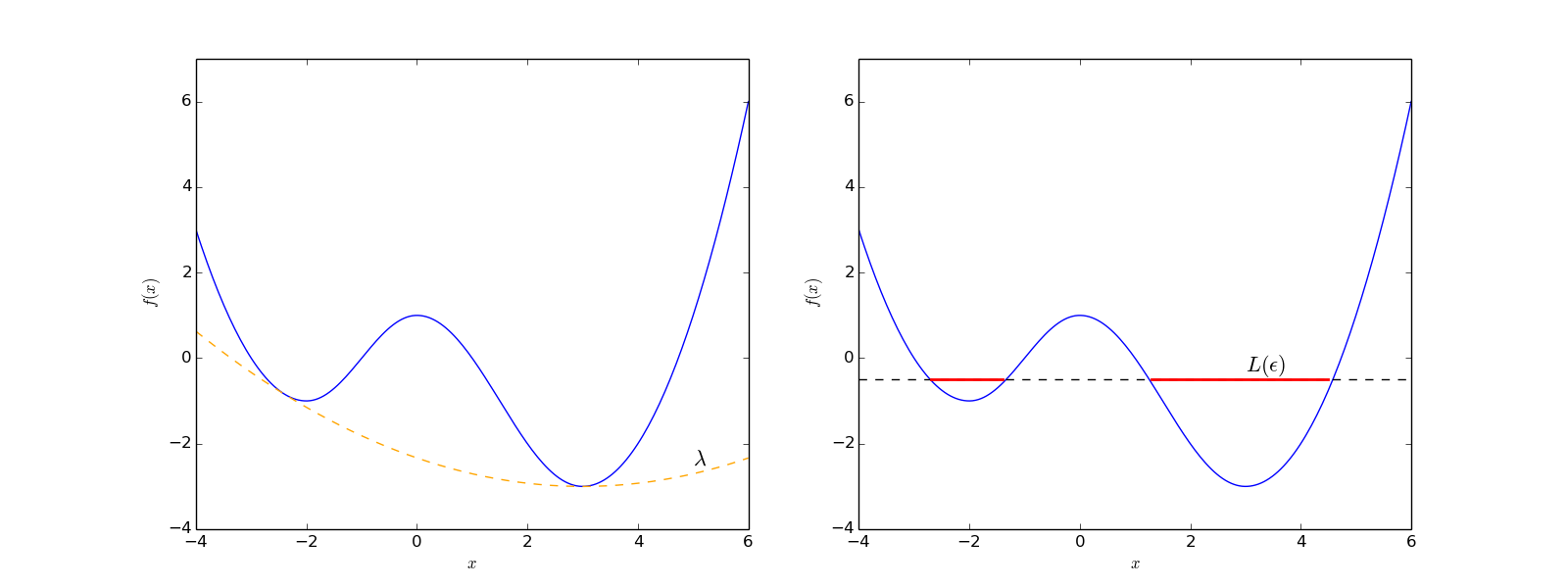} 
        \caption{Approaches to handling two distant local minima. Identify parameters that satisfy the nondegenerate assumptions for function $f$ (Left), or evaluate the $\epsilon$-level set of function $f$ (Right).} 
        \label{fig:response-level-set} 
\end{figure*} 

Below we first specify a condition weaker than that of \citet{wang2018optimization}, and demonstrate that GN successively solves the stochastic optimization problem under this weaker assumption. 

\subsection{The (L-S) condition} 

        The Level-Smooth (L-S) assumption encompasses conditions related to both the level set and the smoothness of the function, hence its name. This condition is weaker than that used by \citet{wang2018optimization}, but is sufficient to guarantee the success of Geometric Narrowing. 
        Given any $\epsilon > 0$, the $\epsilon$-level set for function $f$ is defined as $ L (\epsilon) := \{ \x \in \mathcal{X} : f (\x) \le f (\x^*) + \epsilon \} $, where $\x^* = \arg\min_{\x \in \X} f(\x)$.


\begin{definition}[Level-Smooth assumption] 
        Let $\X \subseteq \R^d$ be measurable. Let $\D$ be a doubling metric on $\X$, and let $ \mu_0 $ be the Lebesgue measure on $\R^d$. 
        We say differentiable function $f$ satisfies the Level-Smooth assumption, abbreviated as (L-S) assumption, if there exist positive constants $\lambda$ and $\ell$ such that 
    \begin{align} 
        \left\{
            \begin{aligned}
                 &\mu_0 (L(\epsilon)) \le \mu_0 \(\B \( \x^*, \frac{\epsilon}{\lambda} \) \),  && \forall \epsilon \in \R_+, \\
                 &| f'(\x) | \le \ell,  &&\forall \x \in \X, 
            \end{aligned}
        \right.
    \end{align}
    where $f'$ is the gradient of $f$.
    \label{def:l-s}
\end{definition}

    The (L-S) assumption encompasses two requirements. The former emphasizes the measure of $\epsilon$-level set, which should remain within a reasonable bound. The latter enhances the smoothness of function $f$.
    Without loss of generality, both of the inequalities in (L-S) assumption hold globally over $\X$. 

   \begin{remark}
        Let $\x^* = \arg\min_{\x \in \X} f(\x)$. If there are several minima, consider arbitrary $\x^*$ such that $f(\x^*) = \min_{\x \in \X} f(\x)$. Due to the translation-invariant property of the measure $\mu_0$, subsequent analysis does not rely on the uniqueness of $\x^*$.
    \end{remark}

    It is worth reiterating that the (L-S) assumption is weaker than the assumptions adopted by \citet{wang2018optimization}. These two assumptions of \citet{wang2018optimization} are listed below as (A1) and (A2). 
    

        

    \begin{itemize} 
        \item[(A1)] There exist constants $M > 0$ and $\alpha \ge 1$, $f$ belongs to the Hölder class $\Sigma^\alpha(M)$ and for any $x, x^{\prime} \in \X \subseteq \mathbb{R}^d$, the partial derivative satisfies 
    \begin{align*} 
        &\sum_{j=0}^k \sum_{\beta_1+\ldots+\beta_d=j}\left|f^{(\bm{\beta}, j)} (x)\right| \\ 
        &+\sum_{\beta_1+\ldots+\beta_d=k} \frac{\left|f^{(\bm{\beta}, k)}(x)-f^{(\bm{\beta}, k)}\left(x^{\prime}\right)\right|}{\left\|x-x^{\prime}\right\|_{\infty}^{\alpha-k}} \leqslant M , 
    \end{align*} 
    where $k = \lfloor \alpha \rfloor $ is the largest integer lower bounding $\alpha$, and $f^{(\bm{\beta}, j)}:= \frac{\partial^j f (x)}{\partial x_1^{\beta_1} \cdots \partial x_d^{\beta_d}}$. 
        \item[(A2)] There exist constants $ C_0$, s.t. Lebesgure measure $\mu_0$ and covering number $N$ satisfy $N(L(\epsilon), \delta) \le C_0 [ 1+ \mu_0(L(\epsilon))\delta^{-d} ]$ for any $\epsilon, \delta$.
    \end{itemize}

    In simple terms, (A1) requires that function $f$ is $k$-times differentiable; (A2) ensures that the set of points with values close to $f^*$ is not excessively large, where $f^* = \inf f(\x)$. 

    \begin{remark} 
        All of (L-S) and (A1, A2) need only hold true on a set $L(\kappa)$ for some $\kappa > 0$. In such cases, a pre-screening procedure can be utilized \citep{wang2018optimization} to narrow the focus to $ L(\kappa) $. 
    \end{remark} 

    \begin{proposition}
        The condition (A1) and (A2) jointly imply (L-S).  
    \end{proposition}

    \begin{proof}
        In (A1), $f$ is $k$-H\"older where $k=\lfloor\alpha\rfloor$ is the largest integer lower bounding $\alpha$.
        On contrary, the second condition in (L-S) only requires boundedness of the first-order differential of function $f$, which is weaker. 

        Next we show that (A2) implies the first condition in (L-S). For this, we notice that $\forall \epsilon, \delta > 0,$
        \begin{align*} 
            \mu_0 (L (\epsilon)) 
            &\le 
            \mu_0 ( \mathbb{B} (0, \delta ) )  N (L (\epsilon), \delta) \\
            &\le
            \mu_0 ( \mathbb{B} (0, \delta ) ) C_0 [ 1+ \mu_0(L(\epsilon))\delta^{-d} ].
        \end{align*} 
        Choosing $\delta = \frac{\epsilon}{\lambda}$, we have
        \begin{align*}
            \mu_0 (L (\epsilon))
            \!\le\!
            C_0 \!\( \!  1\!+\!\( \frac{\epsilon}{\lambda} \)^{-d}\! \mu_0 (L(\epsilon)) \!\) \!\mu_0\! \( \mathbb{B}\! \(0, \frac{\epsilon}{\lambda}\) \!\) .
        \end{align*}
        We can select an appropriate parameter $\lambda$. Hence, (A2) implies the first condition of (L-S). 
        
        


    \end{proof}

In the following part, we will explain in detail how our approach can operate under the (L-S) assumption, which is weaker than (A1) and (A2). 
That is to say, when only requiring smoothness and an upper bound on the measure of points with values close to $f^*$, our modified algorithm can still achieve a $\widetilde{O}(\sqrt{T})$ regret bound and $O(\log \log T )$ communications with minor adjustments. The specific results are articulated in Proposition \ref{prop:l-s}. Below is the detailed explanation of our approach. 

    \subsection{The GN' Algorithm for (L-S) Functions}




This part introduces the GN' algorithm, a variation of the GN algorithm for (L-S) functions. 
Compared to GN algorithm, the narrowing procedure of GN' shifts to retaining balls solely on the basis of estimated values. This shift is necessitated by the fact that regions with smaller values (the $\epsilon$-level set) may not necessarily be located in proximity to $\x^*$. 
To adopt a smaller batch number $M$, we continue to utilize the Rounded Radius sequence $\{\bar{r}_m\}_{m=1}^{2M}$ with $q=1$ from GN algorithm.

\begin{figure*}[t]
    \begin{minipage}{\textwidth}
\begin{algorithm}[H] 
	\caption{Geometric Narrowing' (GN') for (L-S) Functions}
    \label{alg:GN'}
	\begin{algorithmic}[1]  
		\STATE \textbf{Input.} Space $(\mathcal{X}, \D, \mu_0)$; time horizon $T$; Number of batches $2M$; Diameter $ \mathrm{Dim} \( \mathcal{X} \) = 1 $. 
		\STATE \textbf{Initialization.} 
            Rounded Radius sequence $\{\bar{r}_m\}_{m=1}^{2M}$ with $q=1$; The first communication point $t_0=0$; Cover $ \X $ by $\bar{r}_1$-balls, and define $\mathcal{A}_{1}^{pre}$ as the collection of these balls. 
		\STATE Compute $ n_m = \frac{ 16 \log T }{ \lambda^2 \bar{r}_m^{2} } $ for $m=1,\cdots,2M$. 
		
		\FOR{$m=1,2,\cdots,2M$}
            \STATE If $ \bar{r}_{m+1} > \bar{r}_{m} $, then \textbf{continue}.
            \STATE For each ball $B\in\mathcal{A}_m^{pre}$, play arms $\x_{B,1},\cdots, \x_{B,n_m}$, all located at the region of $ B $. 
            \STATE Collect the loss samples $y_{B,1},\cdots, y_{B,n_m}$ and compute the average loss  $\wh{f}_m(B):=\frac{\sum_{i=1}^{n_m}y_{B,i}}{n_m}$ for each ball $B\in\mathcal{A}_m^{pre}$. 
			\STATE Sort $B$ based on $\wh{f}_m(B)$, and define $s(B)$ as the reverse order of $\wh{f}_m(B)$ in $\{\wh{f}_m(B) : B\in\mathcal{A}_m^{pre} \}$. That is to say, $s(B_m^{\min})=1$ when $\wh{f}_m(B_m^{\min})=\min_{B\in\mathcal{A}_m^{pre}}\wh{f}_m (B)$. 
            \STATE Define 
            \begin{align*} 
                \mathcal{A}_m := \left\{ B \in \mathcal{A}_m^{pre} : s(B) \le \[ \frac{9 \ell + 2\lambda }{\lambda} \]_2^d \right\}. 
            \end{align*} 
            
            \STATE For each ball $ B \in \mathcal{A}_m$, use $\left(\bar{r}_m/\bar{r}_{m+1}\right)^d$ balls of radius $ \bar{r}_{m+1} $ to cover $ B $, and define $\mathcal{A}_{m+1}^{pre}$ as the collection of these balls.  
            \STATE Compute $t_{m+1}=t_m+(\bar{r}_{m}/\bar{r}_{m+1})^d \cdot|\mathcal{A}_m|\cdot n_{m+1}$. 
            If $t_{m+1}\geq T$ then \textbf{break}.
    \ENDFOR
    \STATE \textbf{Cleanup and Output } 
	\end{algorithmic} 
\end{algorithm} 
    \end{minipage}
\end{figure*}

GN' algorithm solves batched bandit learning problems for (L-S) functions. In Proposition \ref{prop:l-s}, we present a regret upper bound for GN', which demonstrates the algorithm's effectiveness.

\begin{proposition} 
    \label{prop:l-s} 
    Let $\X \subseteq \R^d$ be measurable. Let $\D$ be a doubling metric on $\X$, and let $ \mu_0 $ be the Lebesgue measure on $\R^d$. 
    Let function $f $ satisfy (L-S) assumption. 
    Consider a stochastic bandit learning environment. 
    For any $T \in \mathbb{N}_+$, with probability exceeding $1-2T^{-1}$, the $T$-step total regret of Geometric Narrowing', written $R^{GN'}(T)$, 
    satisfies 
    \begin{align*}
        R^{GN'} (T) \le {K_+} A_+^{d} \sqrt{T \log T} \log \log \frac{T}{\log T},
    \end{align*}
    where $d$ is the doubling dimension of $(\X, \D)$, and $K_+$ and $A_+$ are constants independent of $d$ and $T$. 
    In addition, only $\mathcal{O}\( \log \log T \)$ communication points are needed to achieve this regret rate. 
\end{proposition} 

To prove Proposition \ref{prop:l-s}, we first measure the regularity of $L(\epsilon)$ using its covering number.

    \begin{lemma} 
        \label{lem:l-s}
        For (L-S) function $f$, define $N(L(\epsilon), \delta)$ as the covering number of cover $L(\epsilon)$ by $\delta$-balls. Then,
        \begin{align}
            \[ \frac{\epsilon}{\delta\ell} \]_2^d \le
            N(L(\epsilon), \delta) \le
            \[ \frac{2\epsilon + \delta\ell}{\delta \lambda} \]_2^d. \label{eq:l-s}
        \end{align} 
    \end{lemma} 
    
    \begin{proof}[Proof of Lemma \ref{lem:l-s}]
        Lower bound of $N(L(\epsilon), \delta)$ is straightforward using the bounded derivative property, and we proceed to prove the upper bound. Consider cover $L(\epsilon)$ by $\delta$-balls: $\{ \B(\x_i, \delta) \}_{i=1}^{N(L(\epsilon), \delta)}$. 
        For $i=1,2,\cdots,N(L(\epsilon), \delta)$, it holds that 
        \begin{align*} 
            \min_{\x \in \B(\x_i, \delta) } f(\x) \le f(\x^*)+\epsilon . 
        \end{align*} 
        Thanks to the bound on maximal absolute derivative $\ell$, for all $i$ we have $\max_{\x \in \B(\x_i, \delta) } f(\x) \le f(\x^*)+\epsilon + 2\delta\ell$, which implies balls $\{ \B(\x_i, \delta) \}_{i=1}^{N(L(\epsilon), \delta)}$ are part of $L(\epsilon + 2\delta\ell)$.
    
        According to the relationship between covering and packing, the number of $\delta$-packing in some space is larger than the number of $\delta$-covering in this space. The property of bounded gradient $\ell$ is used here.
        Subsequently,  $N(L(\epsilon), \delta)$ is smaller than the packing number of $\delta/2$-balls in $L(\epsilon + \delta \ell / 2)$. 
        Since $\mu_0 (L(\epsilon + \delta \ell / 2)) \le \mu_0 \( \B \(\x^*, \(\epsilon + \delta \ell/2\) /\lambda\) \)$, we have $N(L(\epsilon), \delta) \le \[ \frac{2\epsilon + \delta\ell}{\delta \lambda} \]_2^d$.
    \end{proof}

The proof structure of Proposition \ref{prop:l-s} is similar to analysis of the GN algorithm. We estimate the regret for each batch individually, and then sum them up.
Below is a detailed proof. 

\begin{proof}[Proof of Proposition \ref{prop:l-s}]
    Under event $\mathcal{E}$ (defined in Lemma \ref{lem:concen}), the GN' algorithm guarantees that  for any $m$ and $B \in \mathcal{A}_m^{pre}$,
    \begin{align*} 
        \left| \wh{f}_m(B) - \E \[ \wh{f}_m(B) \] \right| \leq \sqrt{\frac{ 4 \log T}{n_m}}= \frac{\lambda \bar{r}_m}{2}. 
    \end{align*} 
     By smoothness, 
    \begin{align*} 
        \min_{\x \in B}\! f(\x) \!-\! \frac{\lambda \bar{r}_m}{2} \!\le\! \wh{f}_m(B) \!\le\! \min_{\x \in B}\! f(\x) \!+\! 2 \bar{r}_{m} \ell \!+\! \frac{\lambda \bar{r}_m}{2}.
    \end{align*}
    Note that for $B \in \mathcal{A}_m^{pre}$, $\wh{f}_m(B)$ being smaller than $f(\x^*) + 2 \bar{r}_{m} \ell + \frac{\lambda \bar{r}_m}{2}$ implies that $\min_{\x \in B} f(\x) \le f(\x^*) + 2 \bar{r}_{m} \ell + \lambda \bar{r}_m$. 
    We get the inclusion relationship 
    \begin{align*}
        &\; \cup \{ B \in  \mathcal{A}_m^{pre} : \wh{f}_m(B) \le f(\x^*) + 2 \bar{r}_{m} \ell + \frac{\lambda \bar{r}_m}{2}\} \\
        \subseteq &\; \{ \x \in \X : f(\x) \le f(\x^*) + 4 \bar{r}_{m} \ell + \lambda \bar{r}_m \}.
    \end{align*}
    By applying Lemma \ref{lem:l-s} for covering $L(4 \bar{r}_{m} \ell + \lambda \bar{r}_m)$ with $\bar{r}_m$-balls, we know there are at most $\[ \frac{9 \ell + 2\lambda }{\lambda} \]_2^d$ balls in $\mathcal{A}_m^{pre}$ satisfying $\wh{f}_m(B) \le f(\x^*) + 2 \bar{r}_{m} \ell + \frac{\lambda \bar{r}_m}{2}$.

    Specially, $\wh{f}_m(B_m^*) \le f(\x^*) + 2 \bar{r}_{m} \ell + \frac{\lambda \bar{r}_m}{2}$ where $B_m^*$ $\in$ $\mathcal{A}_m^{pre}$ is the ball which contains $\x^*$. That is to say, if we pick $\[ \frac{9 \ell + 2\lambda }{\lambda} \]_2^d$ balls with the smallest estimators, the ball constains $\x^*$ will be retained.

    Now, let us count the number of balls in $\mathcal{A}_m^{pre}$ with $\wh{f}_m(B) \le f(\x^*) + 9 \frac{\ell^2}{\lambda} \bar{r}_{m} + 4 \ell \bar{r}_m + \frac{\lambda \bar{r}_m}{2}$.
    Under event $\mathcal{E}$, there is the following inclusion relationship: 
    \begin{align*}
        &\{ B \!\in\!  \mathcal{A}_m^{pre}\! :\! \min_{\x \in B} f(\x) \le f(\x^*) +  \frac{9\ell^2}{\lambda} \bar{r}_{m} + 2 \ell \bar{r}_m \} \\
        \subseteq
        &\{ B \!\in\!  \mathcal{A}_m^{pre}\! :\! \wh{f}_m(B) \!\le\! f(\x^*) \!+\!  \frac{9\ell^2+\!4\ell\lambda\!+\!\lambda^2/2}{\lambda} \bar{r}_{m}  \} \\
        \subseteq
        &\{ B \!\in\!  \mathcal{A}_m^{pre}\! :\! \min_{\x \in B}\! f(\x) \!\le\! f(\x^*) \!+\!  \frac{9\ell^2\!+\!4\ell\lambda\!+\!\lambda^2}{\lambda} \bar{r}_{m} \},
    \end{align*}
    where the first term contains at least $\[ \frac{9 \ell + 2\lambda }{\lambda} \]_2^d$ balls in $\mathcal{A}_m^{pre}$, as determined by the left side inequality of (\ref{eq:l-s}).

    Therefore, for $B \in \mathcal{A}_m$, $\wh{f}_m(B)$ is smaller than $f(\x^*) + \frac{9\ell^2}{\lambda} \bar{r}_{m} + 4 \ell \bar{r}_m + \frac{\lambda \bar{r}_m}{2}$.
    At the same time, by the inclusion relationship, each arm $\x$ in $B$ ($B \in \mathcal{A}_m$) has an expected loss $f(\x) - f(\x^*)$, which is upper bounded by $\frac{9\ell^2}{\lambda} \bar{r}_{m} + 4 \ell \bar{r}_m+ \lambda \bar{r}_m = \left( \frac{9\ell^2+4\ell\lambda+\lambda^2}{\lambda}\right) \bar{r}_m$.
    
    The proposed selection rule for $\mathcal{A}_m$ satisfies
    \begin{align*}
        &\left| \mathcal{A}_m \right| = \[ \frac{9 \ell + 2\lambda }{\lambda} \]_2^d,
        \quad \text{and} \quad \\ 
        &| \mathcal{A}_m^{pre} | 
        \le 
        \( \frac{ \bar{r}_{m-1} }{ \bar{r}_m} \)^d \[ \frac{9 \ell + 2\lambda }{\lambda} \]_2^d .
    \end{align*}
    
    Similar to the proof frame in Theorem 1, for $m=1,2,\cdots ,2M$,
    \begin{align*} 
        R_{m} 
        \le& \;  
        |\mathcal{A}_{m}^{pre}| \cdot n_{m} \cdot\left( \frac{9\ell^2+4\ell\lambda+\lambda^2}{\lambda}\right) \bar{r}_{m-1} \\ 
        \le& \;  
        \( \frac{ \bar{r}_{m-1} }{ \bar{r}_{m}} \)^d \[ \frac{9 \ell + 2\lambda }{\lambda} \]_2^d \cdot \frac{16 \log T}{\lambda^2 \bar{r}_{m}^2} D_{\lambda,\ell} \bar{r}_{m-1} \\
        \le& \;  
        2^{d+2} D_{\lambda,\ell} \frac{16}{\lambda^2} \[ \frac{9 \ell + 2\lambda }{\lambda} \]_2^d \sqrt{T \log T},
    \end{align*}
    where $D_{\lambda,\ell}:=\frac{9\ell^2+4\ell\lambda+\lambda^2}{\lambda}$ is introduced for simplicity.
    
    For the cleanup phase, the regret (written $R_{2M+1}$) is bounded by  
    \begin{align*} 
        R_{2M+1}  
        \le 
        D_{\lambda,\ell} \Bar{r}_{2M} T  
        \le 
        D_{\lambda,\ell} \sqrt{T \log T} \(\frac{T}{\log T} \)^{\frac{1}{2}\hat{\eta}^M} . 
    \end{align*} 
    If choose $M= \frac{\log\log \frac{T}{\log T}}{\log \frac{1}{\hat{\eta}}}$ with $\hat{\eta}=\frac{d+1}{d+2}$, we have $\hat{\eta}^{\hat{M}} = \( \log \frac{T}{\log T} \) ^{-1}$, then 
     \begin{align*}
         R^{GN'} (T)
         \le&  
         2^{d+3} D_{\lambda,\ell}\sqrt{T \log T}  
         \\
        &\cdot\( \frac{16}{\lambda^2}\! \[ \frac{9 \ell + 2\lambda }{\lambda} \]_2^d \frac{\log\log \frac{T}{\log T}}{\log (d\!+\!2) \!-\! \log (d\!+\!1)} \!+\! e^\frac{1}{2} \) . 
     \end{align*}
\end{proof}

\color{black}


\bibliographystyle{ieeetr} 
\bibliography{references}

\begin{thebibliography}{10}

\bibitem{thompson1933likelihood}
W.~R. Thompson, ``On the likelihood that one unknown probability exceeds
  another in view of the evidence of two samples,'' {\em Biometrika}, vol.~25,
  no.~3/4, pp.~285--294, 1933.

\bibitem{robbins1952some}
H.~Robbins, ``Some aspects of the sequential design of experiments,'' {\em
  Bull. Amer. Math. Soc.}, vol.~58, pp.~527--535, 1952.

\bibitem{gittins1979bandit}
J.~C. Gittins, ``Bandit processes and dynamic allocation indices,'' {\em
  Journal of the Royal Statistical Society Series B: Statistical Methodology},
  vol.~41, no.~2, pp.~148--164, 1979.

\bibitem{lai1985asymptotically}
T.~L. Lai and H.~Robbins, ``Asymptotically efficient adaptive allocation
  rules,'' {\em Advances in Applied Mathematics}, vol.~6, no.~1, pp.~4--22,
  1985.

\bibitem{auer2002finite}
P.~Auer, N.~Cesa-Bianchi, and P.~Fischer, ``Finite-time analysis of the
  multiarmed bandit problem,'' {\em Machine learning}, vol.~47, no.~2,
  pp.~235--256, 2002.

\bibitem{auer2002nonstochastic}
P.~Auer, N.~Cesa-Bianchi, Y.~Freund, and R.~E. Schapire, ``The nonstochastic
  multiarmed bandit problem,'' {\em SIAM journal on computing}, vol.~32, no.~1,
  pp.~48--77, 2002.

\bibitem{perchet2016batched}
V.~Perchet, P.~Rigollet, S.~Chassang, and E.~Snowberg, ``Batched bandit
  problems,'' {\em The Annals of Statistics}, vol.~44, no.~2, pp.~660--681,
  2016.

\bibitem{gao2019batched}
Z.~Gao, Y.~Han, Z.~Ren, and Z.~Zhou, ``Batched multi-armed bandits problem,''
  {\em Advances in Neural Information Processing Systems}, vol.~32,
  pp.~503--513, 2019.

\bibitem{berry1985bandit}
D.~A. Berry and B.~Fristedt, ``Bandit problems: sequential allocation of
  experiments (monographs on statistics and applied probability),'' {\em
  London: Chapman and Hall}, vol.~5, no.~71-87, pp.~7--7, 1985.

\bibitem{valko2013stochastic}
M.~Valko, A.~Carpentier, and R.~Munos, ``Stochastic simultaneous optimistic
  optimization,'' in {\em International Conference on Machine Learning},
  pp.~19--27, PMLR, 2013.

\bibitem{10.5555/3294771.3294841}
L.~Zhang, T.~Yangt, J.~Yi, R.~Jin, and Z.-H. Zhou, ``Improved dynamic regret
  for non-degenerate functions,'' in {\em Proceedings of the 31st International
  Conference on Neural Information Processing Systems}, NIPS'17, (Red Hook, NY,
  USA), p.~732–741, Curran Associates Inc., 2017.

\bibitem{gemp2024approximating}
I.~Gemp, L.~Marris, and G.~Piliouras, ``Approximating nash equilibria in
  normal-form games via stochastic optimization,'' in {\em The Twelfth
  International Conference on Learning Representations}, 2024.

\bibitem{NEURIPS2019_0a3df703}
J.~W. Mueller, V.~Syrgkanis, and M.~Taddy, ``Low-rank bandit methods for
  high-dimensional dynamic pricing,'' in {\em Advances in Neural Information
  Processing Systems} (H.~Wallach, H.~Larochelle, A.~Beygelzimer,
  F.~d\textquotesingle Alch\'{e}-Buc, E.~Fox, and R.~Garnett, eds.), vol.~32,
  Curran Associates, Inc., 2019.

\bibitem{chen2023robust}
X.~Chen and Y.~Wang, ``Robust dynamic pricing with demand learning in the
  presence of outlier customers,'' {\em Operations Research}, vol.~71, no.~4,
  pp.~1362--1386, 2023.

\bibitem{Perakis2023dynamic}
G.~Perakis and D.~Singhvi, ``Dynamic pricing with unknown nonparametric demand
  and limited price changes,'' {\em Operations Research}, 2023.

\bibitem{mankiw1998principles}
N.~G. Mankiw, {\em Principles of microeconomics}, vol.~1.
\newblock Elsevier, 1998.

\bibitem{Yau1974}
S.-T. Yau, ``Non-existence of continuous convex functions on certain riemannian
  manifolds,'' {\em Mathematische Annalen}, vol.~207, pp.~269--270, Dec 1974.

\bibitem{petersen2006riemannian}
P.~Petersen, {\em Riemannian geometry}, vol.~171.
\newblock Springer, 2006.

\bibitem{shamir2013complexity}
O.~Shamir, ``On the complexity of bandit and derivative-free stochastic convex
  optimization,'' in {\em Conference on Learning Theory}, pp.~3--24, PMLR,
  2013.

\bibitem{kleinberg2005nearly}
R.~Kleinberg, ``Nearly tight bounds for the continuum-armed bandit problem,''
  {\em Advances in Neural Information Processing Systems}, vol.~18,
  pp.~697--704, 2005.

\bibitem{kleinberg2008multi}
R.~Kleinberg, A.~Slivkins, and E.~Upfal, ``Multi-armed bandits in metric
  spaces,'' in {\em Proceedings of the fortieth annual ACM symposium on Theory
  of computing}, pp.~681--690, 2008.

\bibitem{bubeck2008tree}
S.~Bubeck, R.~Munos, G.~Stoltz, and C.~Szepesv{\'a}ri, ``Online optimization in
  $\mathcal{X}$-armed bandits,'' {\em Advances in Neural Information Processing
  Systems}, vol.~22, pp.~201--208, 2009.

\bibitem{bubeck2011x}
S.~Bubeck, R.~Munos, G.~Stoltz, and C.~Szepesv{\'a}ri, ``$\mathcal{X}$-armed
  bandits,'' {\em Journal of Machine Learning Research}, vol.~12, no.~5,
  pp.~1655--1695, 2011.

\bibitem{agrawal2012analysis}
S.~Agrawal and N.~Goyal, ``Analysis of thompson sampling for the multi-armed
  bandit problem,'' in {\em Conference on learning theory}, pp.~39--1, JMLR
  Workshop and Conference Proceedings, 2012.

\bibitem{arora2012multiplicative}
S.~Arora, E.~Hazan, and S.~Kale, ``The multiplicative weights update method: a
  meta-algorithm and applications,'' {\em Theory of computing}, vol.~8, no.~1,
  pp.~121--164, 2012.

\bibitem{MAL-024}
S.~Bubeck and N.~Cesa-Bianchi, ``Regret analysis of stochastic and
  nonstochastic multi-armed bandit problems,'' {\em Foundations and Trends® in
  Machine Learning}, vol.~5, no.~1, pp.~1--122, 2012.

\bibitem{MAL-068}
A.~Slivkins, ``Introduction to multi-armed bandits,'' {\em Foundations and
  Trends® in Machine Learning}, vol.~12, no.~1-2, pp.~1--286, 2019.

\bibitem{lattimore2020bandit}
T.~Lattimore and C.~Szepesv{\'a}ri, {\em Bandit algorithms}.
\newblock Cambridge University Press, 2020.

\bibitem{auer2002using}
P.~Auer, ``Using confidence bounds for exploitation-exploration trade-offs,''
  {\em Journal of Machine Learning Research}, vol.~3, no.~Nov, pp.~397--422,
  2002.

\bibitem{dani2007price}
V.~Dani, S.~M. Kakade, and T.~Hayes, ``The price of bandit information for
  online optimization,'' {\em Advances in Neural Information Processing
  Systems}, vol.~20, 2007.

\bibitem{chu2011contextual}
W.~Chu, L.~Li, L.~Reyzin, and R.~Schapire, ``Contextual bandits with linear
  payoff functions,'' in {\em Proceedings of the Fourteenth International
  Conference on Artificial Intelligence and Statistics}, pp.~208--214, JMLR
  Workshop and Conference Proceedings, 2011.

\bibitem{abbasi2011improved}
Y.~Abbasi-Yadkori, D.~P{\'a}l, and C.~Szepesv{\'a}ri, ``Improved algorithms for
  linear stochastic bandits,'' {\em Advances in neural information processing
  systems}, vol.~24, 2011.

\bibitem{srinivas2012information}
N.~Srinivas, A.~Krause, S.~M. Kakade, and M.~W. Seeger, ``Information-theoretic
  regret bounds for gaussian process optimization in the bandit setting,'' {\em
  IEEE transactions on information theory}, vol.~58, no.~5, pp.~3250--3265,
  2012.

\bibitem{contal2014gaussian}
E.~Contal, V.~Perchet, and N.~Vayatis, ``Gaussian process optimization with
  mutual information,'' in {\em International Conference on Machine Learning},
  pp.~253--261, PMLR, 2014.

\bibitem{pmlr-v134-podimata21a}
C.~Podimata and A.~Slivkins, ``Adaptive discretization for adversarial
  lipschitz bandits,'' in {\em Proceedings of Thirty Fourth Conference on
  Learning Theory} (M.~Belkin and S.~Kpotufe, eds.), vol.~134 of {\em
  Proceedings of Machine Learning Research}, pp.~3788--3805, PMLR, 15--19 Aug
  2021.

\bibitem{agarwal2017learning}
A.~Agarwal, S.~Agarwal, S.~Assadi, and S.~Khanna, ``Learning with limited
  rounds of adaptivity: coin tossing, multi-armed bandits, and ranking from
  pairwise comparisons,'' in {\em Conference on Learning Theory}, pp.~39--75,
  PMLR, 2017.

\bibitem{jin2019efficient}
T.~Jin, J.~Shi, X.~Xiao, and E.~Chen, ``Efficient pure exploration in adaptive
  round model,'' {\em Advances in Neural Information Processing Systems},
  vol.~32, 2019.

\bibitem{karpov2024parallel}
N.~Karpov and Q.~Zhang, ``Parallel best arm identification in heterogeneous
  environments,'' in {\em Proceedings of the 36th ACM Symposium on Parallelism
  in Algorithms and Architectures}, pp.~53--64, 2024.

\bibitem{agrawal1995continuum}
R.~Agrawal, ``The continuum-armed bandit problem,'' {\em SIAM Journal on
  Control and Optimization}, vol.~33, no.~6, pp.~1926--1951, 1995.

\bibitem{auer2007improved}
P.~Auer, R.~Ortner, and C.~Szepesv{\'a}ri, ``Improved rates for the stochastic
  continuum-armed bandit problem,'' in {\em Conference on Computational
  Learning Theory}, pp.~454--468, Springer, 2007.

\bibitem{cope2009regret}
E.~W. Cope, ``Regret and convergence bounds for a class of continuum-armed
  bandit problems,'' {\em IEEE Transactions on Automatic Control}, vol.~54,
  no.~6, pp.~1243--1253, 2009.

\bibitem{bubeck2011lipschitz}
S.~Bubeck, G.~Stoltz, and J.~Y. Yu, ``Lipschitz bandits without the {L}ipschitz
  constant,'' in {\em International Conference on Algorithmic Learning Theory},
  pp.~144--158, Springer, 2011.

\bibitem{magureanu2014lipschitz}
S.~Magureanu, R.~Combes, and A.~Proutiere, ``Lipschitz bandits: Regret lower
  bound and optimal algorithms,'' in {\em Conference on Learning Theory},
  pp.~975--999, PMLR, 2014.

\bibitem{lu2019optimal}
S.~Lu, G.~Wang, Y.~Hu, and L.~Zhang, ``Optimal algorithms for {L}ipschitz
  bandits with heavy-tailed rewards,'' in {\em International Conference on
  Machine Learning}, pp.~4154--4163, 2019.

\bibitem{krishnamurthy2020contextual}
A.~Krishnamurthy, J.~Langford, A.~Slivkins, and C.~Zhang, ``Contextual bandits
  with continuous actions: Smoothing, zooming, and adapting,'' {\em The Journal
  of Machine Learning Research}, vol.~21, no.~1, pp.~5402--5446, 2020.

\bibitem{majzoubi2020efficient}
M.~Majzoubi, C.~Zhang, R.~Chari, A.~Krishnamurthy, J.~Langford, and
  A.~Slivkins, ``Efficient contextual bandits with continuous actions,'' {\em
  Advances in Neural Information Processing Systems}, vol.~33, pp.~349--360,
  2020.

\bibitem{10239433}
Y.~Feng, Z.~Huang, and T.~Wang, ``Lipschitz bandits with batched feedback,''
  {\em IEEE Trans. Inf. Theor.}, vol.~70, p.~2154–2176, Sept. 2023.

\bibitem{slivkins2011contextual}
A.~Slivkins, ``Contextual bandits with similarity information,'' {\em Journal
  of Machine Learning Research}, vol.~15, no.~1, pp.~2533--2568, 2014.

\bibitem{cesa2013online}
N.~Cesa-Bianchi, O.~Dekel, and O.~Shamir, ``Online learning with switching
  costs and other adaptive adversaries,'' {\em Advances in Neural Information
  Processing Systems}, vol.~26, pp.~1160--1168, 2013.

\bibitem{jun2016top}
K.-S. Jun, K.~Jamieson, R.~Nowak, and X.~Zhu, ``Top arm identification in
  multi-armed bandits with batch arm pulls,'' in {\em Artificial Intelligence
  and Statistics}, pp.~139--148, PMLR, 2016.

\bibitem{tao2019collaborative}
C.~Tao, Q.~Zhang, and Y.~Zhou, ``Collaborative learning with limited
  interaction: tight bounds for distributed exploration in multi-armed
  bandits,'' in {\em 2019 IEEE 60th Annual Symposium on Foundations of Computer
  Science (FOCS)}, pp.~126--146, IEEE, 2019.

\bibitem{han2020sequential}
Y.~Han, Z.~Zhou, Z.~Zhou, J.~Blanchet, P.~W. Glynn, and Y.~Ye, ``Sequential
  batch learning in finite-action linear contextual bandits,'' {\em arXiv
  preprint arXiv:2004.06321}, 2020.

\bibitem{karpov2020collaborative}
N.~Karpov, Q.~Zhang, and Y.~Zhou, ``Collaborative top distribution
  identifications with limited interaction,'' in {\em 2020 IEEE 61st Annual
  Symposium on Foundations of Computer Science (FOCS)}, pp.~160--171, IEEE,
  2020.

\bibitem{esfandiari2021regret}
H.~Esfandiari, A.~Karbasi, A.~Mehrabian, and V.~Mirrokni, ``Regret bounds for
  batched bandits,'' in {\em Proceedings of the AAAI Conference on Artificial
  Intelligence}, vol.~35, pp.~7340--7348, 2021.

\bibitem{ruan2021linear}
Y.~Ruan, J.~Yang, and Y.~Zhou, ``Linear bandits with limited adaptivity and
  learning distributional optimal design,'' in {\em Proceedings of the 53rd
  Annual ACM SIGACT Symposium on Theory of Computing}, pp.~74--87, 2021.

\bibitem{li2022gaussian}
Z.~Li and J.~Scarlett, ``Gaussian process bandit optimization with few
  batches,'' in {\em International Conference on Artificial Intelligence and
  Statistics}, pp.~92--107, PMLR, 2022.

\bibitem{agarwal2022batched}
A.~Agarwal, R.~Ghuge, and V.~Nagarajan, ``Batched dueling bandits,'' {\em arXiv
  preprint arXiv:2202.10660}, 2022.

\bibitem{jin2021almost}
T.~Jin, J.~Tang, P.~Xu, K.~Huang, X.~Xiao, and Q.~Gu, ``Almost optimal anytime
  algorithm for batched multi-armed bandits,'' in {\em International Conference
  on Machine Learning}, pp.~5065--5073, PMLR, 2021.

\bibitem{jin2021double}
T.~Jin, P.~Xu, X.~Xiao, and Q.~Gu, ``Double explore-then-commit: Asymptotic
  optimality and beyond,'' in {\em Conference on Learning Theory},
  pp.~2584--2633, PMLR, 2021.

\bibitem{absil2008optimization}
P.-A. Absil, R.~Mahony, and R.~Sepulchre, {\em Optimization algorithms on
  matrix manifolds}.
\newblock Princeton University Press, 2008.

\bibitem{boumal2023introduction}
N.~Boumal, {\em An introduction to optimization on smooth manifolds}.
\newblock Cambridge University Press, 2023.

\bibitem{li2023stochastic}
J.~Li, K.~Balasubramanian, and S.~Ma, ``Stochastic zeroth-order riemannian
  derivative estimation and optimization,'' {\em Mathematics of Operations
  Research}, vol.~48, no.~2, pp.~1183--1211, 2023.

\bibitem{li2023zeroth}
J.~Li, K.~Balasubramanian, and S.~Ma, ``Zeroth-order riemannian averaging
  stochastic approximation algorithms,'' {\em arXiv preprint arXiv:2309.14506},
  2023.

\bibitem{doi:10.1137/140955483}
W.~Huang, K.~A. Gallivan, and P.-A. Absil, ``A broyden class of quasi-newton
  methods for riemannian optimization,'' {\em SIAM Journal on Optimization},
  vol.~25, no.~3, pp.~1660--1685, 2015.

\bibitem{doi:10.1137/16M1098759}
B.~Gao, X.~Liu, X.~Chen, and Y.-x. Yuan, ``A new first-order algorithmic
  framework for optimization problems with orthogonality constraints,'' {\em
  SIAM Journal on Optimization}, vol.~28, no.~1, pp.~302--332, 2018.

\bibitem{doi:10.1137/17M1116787}
H.~Sato, H.~Kasai, and B.~Mishra, ``Riemannian stochastic variance reduced
  gradient algorithm with retraction and vector transport,'' {\em SIAM Journal
  on Optimization}, vol.~29, no.~2, pp.~1444--1472, 2019.

\bibitem{doi:10.1137/18M122457X}
S.~Chen, S.~Ma, A.~Man-Cho~So, and T.~Zhang, ``Proximal gradient method for
  nonsmooth optimization over the stiefel manifold,'' {\em SIAM Journal on
  Optimization}, vol.~30, no.~1, pp.~210--239, 2020.

\bibitem{gao2021riemannian}
B.~Gao, N.~T. Son, P.-A. Absil, and T.~Stykel, ``Riemannian optimization on the
  symplectic stiefel manifold,'' {\em SIAM Journal on Optimization}, vol.~31,
  no.~2, pp.~1546--1575, 2021.

\bibitem{doi:10.1137/20M1312952}
A.~Ruszczy\'{n}ski, ``A stochastic subgradient method for nonsmooth nonconvex
  multilevel composition optimization,'' {\em SIAM Journal on Control and
  Optimization}, vol.~59, no.~3, pp.~2301--2320, 2021.

\bibitem{wang2018optimization}
Y.~Wang, S.~Balakrishnan, and A.~Singh, ``Optimization of smooth functions with
  noisy observations: Local minimax rates,'' {\em Advances in Neural
  Information Processing Systems}, vol.~31, 2018.

\bibitem{10.1093/acprof:oso/9780199535255.001.0001}
S.~Boucheron, G.~Lugosi, and P.~Massart, {\em {Concentration Inequalities: A
  Nonasymptotic Theory of Independence}}.
\newblock Oxford University Press, 02 2013.

\bibitem{10.1007/BFb0064610}
J.~Bretagnolle and C.~Huber, ``Estimation des densit{\'e}s : Risque minimax,''
  in {\em S{\'e}minaire de Probabilit{\'e}s XII} (C.~Dellacherie, P.~A. Meyer,
  and M.~Weil, eds.), (Berlin, Heidelberg), pp.~342--363, Springer Berlin
  Heidelberg, 1978.

\end{thebibliography}

\end{document}